\documentclass{article}

\usepackage{arxiv}

\usepackage{amsmath,amssymb,amsthm}
\usepackage{bm}
\usepackage{graphicx,here,comment}
\usepackage{xcolor}
\usepackage{fullpage}
\usepackage{booktabs}
\usepackage{hyperref}

\usepackage[utf8]{inputenc}
\usepackage{amsmath,amsthm,amssymb}
\usepackage{bm}
\usepackage[square,sort,comma,numbers]{natbib}
\usepackage{mathtools}
\usepackage{url}
\usepackage{xcolor}
\usepackage{textcomp}

\newtheorem{lemma}{Lemma}
\newtheorem{remark}{Remark}
\newtheorem{theorem}{Theorem} 
\newtheorem{assumption}{Assumption}
\newtheorem{proposition}{Proposition}
\usepackage{booktabs}
\usepackage[inline]{enumitem}
\usepackage{appendix}
\usepackage{subcaption}
\usepackage{algorithm}
\usepackage{algpseudocode}
\usepackage{pgfplots}
\usepackage{multirow}
\usepackage{array}

\makeatletter
\def\plist@algorithm{Alg.\space}
\makeatother

\usetikzlibrary{shapes,decorations,arrows,calc,arrows.meta,fit,positioning}
  
\newcommand{\argmin}{\operatornamewithlimits{argmin}}
\newcommand{\argmax}{\operatornamewithlimits{argmax}}
  
\newcommand\ci{\perp\!\!\!\perp} 
\newcolumntype{P}[1]{>{\hspace{0pt}}p{#1}}
\allowdisplaybreaks

\newcommand{\bc}{\color{black}}
   
\newcommand{\sref}[2]{\hyperref[#2]{#1~\ref{#2}}}
\newcommand{\eqnref}[1]{\hyperref[#1]{Equation~\eqref{#1}}}

\renewcommand{\algorithmicrequire}{\textbf{Input:}}
\renewcommand{\algorithmicensure}{\textbf{Output:}}
\newcommand{\rtext}[1]{\hfill\textcolor{gray}{#1}}

\newcommand{\inx}{\ensuremath{\mathcal{X}}}
\newcommand{\iny}{\ensuremath{\mathcal{Y}}} 
\newcommand{\inz}{\ensuremath{\mathcal{Z}}}  
\newcommand{\inw}{\ensuremath{\mathcal{W}}}

\newcommand{\rr}{\mathbb{R}} 		         
\newcommand{\ep}{\mathbb{E}}                     

\newcommand{\dd}{\, \mathrm{d}}

\newcommand{\pto}{\overset{\mathrm{p}}{\to}}
\newcommand{\inF}{\mathcal{F}}
 
\begin{document}

\title{Instrumental Variable Regression \\via Kernel Maximum Moment Loss}

\author{
 Rui Zhang \\
  Australian National University \\
  \texttt{rui.zhang@anu.edu.au} \\
   \And
 Masaaki Imaizumi \\
  The University of Tokyo \\
  \texttt{imaizumi@g.ecc.u-tokyo.ac.jp} \\
  \And
 Bernhard Sch\"olkopf \\
  Max Planck Institute for Intelligent Systems \\
  \texttt{bs@tuebingen.mpg.de} \\
  \AND
  Krikamol Muandet \\
  CISPA--Helmholtz Center for Information Security \\
  \texttt{muandet@cispa.de} \\
}





\maketitle

 \begin{abstract}
    {We investigate a simple objective for nonlinear instrumental variable (IV) regression based on a kernelized conditional moment restriction (CMR) known as a maximum moment restriction (MMR). 
    The MMR objective is formulated by maximizing the interaction between the residual and the instruments belonging to a unit ball in a reproducing kernel Hilbert space (RKHS).
    First, it allows us to simplify the IV regression as an empirical risk minimization problem, where the risk functional depends on the reproducing kernel on the instrument and can be estimated by a U-statistic or V-statistic.
    Second, based on this simplification, we are able to provide the consistency and asymptotic normality results in both parametric and nonparametric settings. 
    Lastly, we provide easy-to-use IV regression algorithms with an efficient hyper-parameter selection procedure.
    We demonstrate the effectiveness of our algorithms using experiments on both synthetic and real-world data.}
\end{abstract}


\section{Introduction}

Instrumental variables (IV) have become standard tools for economists, epidemiologists, and social scientists to uncover causal relationships from observational data 
\citep{Angrist08:Harmless,klungel2015instrumental}.
Randomization of treatments or policies has been perceived as the gold standard for such tasks, but is generally prohibitive in many real-world scenarios due to time constraints or ethical concerns.
When treatment assignment is not randomized, it is generally impossible to discern between the causal effect of treatments and spurious correlations that are induced by unobserved factors.
Instead, IVs enable the investigators to incorporate \emph{natural} variation through an IV that is associated with the treatments, but not with the outcome variable, other than through its effect on the treatments. 
In economics, for instance, the season-of-birth was used as an IV to study the return from schooling, which measures causal effect of education on labor market earning \citep{Card99:Labor}.
In genetic epidemiology, the idea to use genetic variants as IVs, known as Mendelian randomization, has also gained increasing popularity \citep{Burgess17:Mendelian,Burgess20:GIV}.

{\bc
There are numerous variations of methods for dealing with IV.
Classical methods for IV regression often rely on a linearity assumption in which a two-stage least squares (2SLS) is the most popular technique \citep{Angrist96:IV}. 
The generalized method of moments (GMM) of \citep{Hansen82:GMM}, which imposes the orthogonality restrictions, can also be used for the linear IV regression. 
For nonlinear regression, numerous methodologies have emerged in the field of nonparametric IV  \citep{Newey03:NIV,Hall05:NIV,Blundell07:Semi-NIV, horowitz2011applied} and recent machine learning \citep{Hartford17:DIV,Lewis18:AGMM,Bennett19:DeepGMM,Muandet19:DualIV,Singh19:KIV,Liao20:NeuralSEM, Bennett20:variational}.
These estimators can be categorized into two general approaches. 
The former generalizes a two-stage optimization procedure \citep{Hartford17:DIV, Singh19:KIV}, whereas the latter requires solving a minimax optimization problem \citep{Lewis18:AGMM,Bennett19:DeepGMM,Liao20:NeuralSEM,daskalakis2018limit,bailey2020finite}.
While the equivalence between 2SLS and GMM is known in the linear setting, the connection between two-stage procedure and minimax optimization in the nonlinear setting was established recently, see, e.g., \citep{Muandet19:DualIV}, \citep[Appendix F]{Liao20:NeuralSEM}, and \citep[Sec 3.3]{Mastouri21:ProxyKernel}.
For statistical inference, numerous approaches have been developed using asymptotic distributions of each point of uniform bands of the target functions by the sieve method \citep{horowitz2012uniform,chen2015optimal,chen2018optimal,babii2020honest,chen2021adaptive}.

}

{\bc
In this paper, we propose a novel nonlinear IV regression framework named \textit{Maximum Moment Restriction-IV} ({MMR-IV}), which utilizes the reproducing kernel Hilbert spaces (RKHSes) and an associated positive definite kernel from the machine learning domain.
}
The MMR-IV possesses mainly two important features:
(i) It reformulates the original conditional moment restriction (CMR) into a single-step empirical risk minimization (ERM) problem, by using the maximum moment restriction framework of \citep{Muandet20:KCM} with the kernel and RKHSes.
(ii) Based on the ERM problem, it reveals a closed form of the solution by U/V-statistics techniques \citep{Serfling80:Approximation}. 
Based on this framework, we propose two practical algorithms based on kernel functions and neural networks, which we call  MMR-IV (RKHS) and MMR-IV (NN), respectively.

Our MMR-IV has the following advantages.
First, the derived closed-form expression is computationally simple and stable. 
Further, owing to the generic structure of ERM, we can substitute neural networks for kernel functions, which is suitable for highly nonlinear modeling.
{\bc 
Second, MMR-IV can identify the optimal solution, i.e., the minimizer of the risk function is guaranteed to be unique, by using the positive definite property of the kernel function.
}
Third, we can identify a distribution that the solution follows in the large sample limit, i.e., the asymptotic distribution. 
The identified asymptotic distribution allows us to analyze uncertainties of algorithms and estimators, such as statistical tests and confidence intervals. 
This result is applicable to the case where the model is both parametric and nonparametric, and especially in the nonparametric case, we utilize the functional differentiation technique \citep{hable2012asymptotic} to derive the results.
{\bc Moreover, based on the asymptotic distribution, we discuss how the choice of kernel functions affects the efficiency of MMR-IV.}
Fourth, we provide a hyperparameter selection scheme, by developing a novel interpretation of MMR-IV from a Gaussian process (GP) perspective. 
This interpretation is obtained by mapping our ERM form to a likelihood of GPs.
Since hyperparameter selection for IV regression is a non-trivial problem in both minimax and two-stage formulation, our framework sheds light on efficient methods to solve this problem.
In the experimental section, we demonstrate the superiority of the proposed framework by validating its accuracy and effectiveness, and show its usefulness in real data analysis.

Our contributions can be summarized as follows: 
\begin{enumerate}[noitemsep,nolistsep,label=(\roman*)]
    \item We study the simple ERM problem converted from the minimax optimization,
    then propose the MMR-IV algorithms to minimize the risk based on kernel functions and neural networks. 
    This algorithm avoids the difficulties of minimax problems, and is computationally simple and stable.
    \item {\bc Our method can identify unique solutions by selecting positive definite kernel functions. Our analysis also reveals how the asymptotic efficiency varies with the choice of kernel functions.}
    \item We derive the asymptotic normal distribution of the MMR-IV estimators in both parametric and non-parametric settings. 
    This result is derived from the ERM form, and allows for statistical inference.
    \item We develop the efficient framework to select hyperparameters of MMR-IV.
    This approach is based on the interpretation of MMR-IV with kernel functions from a GP perspective. 
    \item We compare our methods to a wide range of baselines on different settings, and show that MMR-IV has competitive performance. 
\end{enumerate}

The rest of this paper is organized as follows. Section \ref{sec:preliminaries} introduces the instrumental variable regression problem and the MMR framework. Our methods are presented in Section \ref{sec:mmr-iv}, followed by the hyperparameter selection in Section \ref{sec:model-selection}. Next, we provide the consistency and asymptotic normality results in Section \ref{sec:theory}. We then provide a thorough discussion on related works in Section \ref{sec:related-work}. Finally, the experimental results are presented in Section \ref{sec:experiments}. 
All proofs of the main results can be found in the appendix.

\section{Preliminaries}
\label{sec:preliminaries}

\subsection{Instrumental Variable Regression}
\label{sec:ivr}

Following standard setting in the literature \citep{Hartford17:DIV,Lewis18:AGMM,Bennett19:DeepGMM,Singh19:KIV,Muandet19:DualIV}, 
let $X$ be a treatment (endogenous) variable taking value in $\inx\subseteq\rr^d$ and $Y$ a real-valued outcome variable.
Our goal is to estimate a function $f:\inx\to\rr$ from a structural equation model of the form
\begin{equation}\label{eq:model}
    \textstyle Y = f(X) + \varepsilon, \quad
    X = t(Z) + g(\varepsilon) + \nu,
\end{equation}
\noindent where we assume that $\ep[\varepsilon] = 0$ and $\ep[\nu] = 0$.
Unfortunately, as we can see from \eqref{eq:model}, $\varepsilon$ is correlated with the treatment $X$, i.e., $\ep[\varepsilon|X] \neq 0$,  and hence standard regression methods cannot be used to estimate $f$.
This setting arises, for example, when there exist unobserved confounders (i.e., common causes) between $X$ and $Y$. To illustrate the problem, let us consider an example taken from \citep{Hartford17:DIV} which aims at predicting sales of airline ticket $Y$ under an intervention in price of the ticket $X$.
However, there exist unobserved variables that may affect both sales and ticket price, e.g., conferences, COVID-19 pandemic, etc.
This creates a correlation between $\varepsilon$ and $X$ in \eqref{eq:model} that prevents us from applying the standard regression toolboxes directly on observational data.

\textbf{Instrumental variable (IV).}
To address this problem, we assume access to an \emph{instrumental} variable $Z$ taking value in $\inz\subseteq\rr^{d'}$. 
As we can see in \eqref{eq:model}, the instrument $Z$ is associated with the treatments $X$, but not with the outcome $Y$, other than through its effect on the treatments.
Formally, $Z$ must satisfy 
\begin{enumerate*}[label=(\roman*),noitemsep]
    \item \emph{Relevance:} $Z$ has a causal influence on $X$,
    \item \emph{Exclusion restriction:} $Z$ affects $Y$ only through $X$, i.e., ${Y \ci Z} | X,\varepsilon$, and
    \item \emph{Unconfounded instrument(s):} $Z$ is independent of the error, i.e., ${\varepsilon \ci Z}$.
\end{enumerate*}
See Fig. \ref{fig:causal-graph} for a visualization of the three conditions.
For example, the instrument $Z$ may be the cost of fuel, which influences sales $Y$ only via price $X$.
For detailed exposition on IV, we refer the readers to \citep{Newey03:NIV} and \citep{Angrist08:Harmless}.

\textbf{Conditional moment restriction (CMR).}
Together with the structural assumption in \eqref{eq:model}, these conditions imply that $\ep[\varepsilon\,|\,Z] = 0$ for $P_Z$-almost all $z$.
This is known as a \emph{conditional moment restriction} (CMR) which we can use to estimate $f$ \citep{Newey93:CMR}. 
For every measurable function $h$, the CMR implies a continuum of unconditional moment restrictions \citep{Lewis18:AGMM,Bennett19:DeepGMM}\footnote{Also, IV assumptions imply $\varepsilon$ is independent of $Z$, i.e., $\mathbb{E}[(Y - f(X))h(Z)] = 0$ for all measurable $h$.}:
\begin{equation}\label{eq:umr}
    \textstyle \ep[(Y-f(X))h(Z)] = 0 .
\end{equation}
That is, there exists an infinite number of moment conditions, each of which is indexed by the function $h$.
One of the key questions in econometrics is which moment condition should be used as a basis for estimating the function $f$ \citep{Donald01:ChooseIV,Hall05:GMM}.
In this work, we show that, for the purpose of consistently estimating $f$, it is sufficient to restrict $h$ to be within a unit ball of a RKHS of real-valued functions on $Z$.

\tikzset{
    -Latex,auto,node distance =1 cm and 1 cm,semithick,
    state/.style ={ellipse, draw, minimum width = 0.7 cm},
    point/.style = {circle, draw, inner sep=0.04cm,fill,node contents={}},
    bidirected/.style={Latex-Latex,dashed},
    el/.style = {inner sep=2pt, align=left, sloped}
}

\begin{figure}[t!]
        \centering
        \begin{tikzpicture}[scale=1.5]
            \node (x) at (0,0) [label=125:X,point];
            \node (y) at (1.3,0) [label=right:Y,point];
            \node (z) at (-1.3,0) [label=left:Z,point];
    
            \path (x) edge [line width=1pt] (y);
            \path (z) edge [line width=1pt,color=blue] (x);
            \path[bidirected] (x) edge[line width=1pt,color=red,bend left=60] (y);
        \end{tikzpicture}
        \vspace{8pt}
        \caption{A causal graph depicting an instrumental variable $Z$ that satisfies an exclusion restriction and unconfoundedness (there may be a confounder $\varepsilon$ acting on $X$ and $Y$, but it is independent of $Z$).}
        \label{fig:causal-graph}
    \end{figure}
    
\subsection{Maximum Moment Restriction}

Throughout this paper, we assume that $h$ is a real-valued function on $\inz$ which belongs to a RKHS $\mathcal{H}_k$ endowed with a reproducing kernel $k:\inz\times\inz\to\rr$.
The RKHS $\mathcal{H}_k$ satisfies two important properties: for all $z\in\inz$ and $h\in\mathcal{H}_k$,
\begin{enumerate*}[label=(\roman*),noitemsep]
    \item $k(z,\cdot)\in\mathcal{H}_k$ and 
    \item (reproducing property) $h(z) = \langle h,k(z,\cdot)\rangle_{\mathcal{H}_k}$ where $k(z,\cdot)$ is a function of the second argument. 
\end{enumerate*}
Furthermore, we define  $\Phi_k(z)$ as a \emph{canonical} feature map of $z$ in $\mathcal{H}_k$.
It follows from the reproducing property that $k(z,z') = \langle \Phi_k(z),\Phi_k(z')\rangle_{\mathcal{H}_k}$ for every $z,z'\in\inz$, i.e., an inner product between the feature maps of $z$ and $z'$ can be evaluated through the kernel evaluation.
Every positive definite kernel $k$ uniquely determines the RKHS for which $k$ is a reproducing kernel \citep{aronszajn50reproducing}.
For detailed exposition on kernel methods, see e.g. \citep{Scholkopf01:LKS}, \citep{Berlinet04:RKHS}, and \citep{Muandet17:KME}.

Instead of considering all measurable functions $h$ in \eqref{eq:umr} as instruments, we only restrict to functions that lie within a unit ball of the RKHS $\mathcal{H}_k$. 
The risk functional is then defined as a maximum value of the moment restriction with respect to this function class:
\begin{equation}\label{eq:mmr}
    \textstyle R_k(f) := \sup_{h\in\mathcal{H}_k, \|h\|\leq 1}\; \left(\ep[(Y-f(X))h(Z)]\right)^2.
\end{equation}
The benefits of this formulation are two-fold. 
First, it is computationally intractable to learn $f$ from \eqref{eq:umr} using \emph{all} measurable functions as instruments.
By restricting the function class to a unit ball of the RKHS, the problem becomes computationally tractable, as will be shown in Lemma \ref{lem:analytic-form} below.
Second, this restriction still preserves the consistency of parameter estimated using $R_k(f)$. 
In other words, the RKHS is a sufficient class of instruments for the nonlinear IV problem (cf. Theorem \ref{thm:orig_cond_m_res}).
A crucial insight for our approach is that the population risk $R_k(f)$ has an analytic solution.
\begin{lemma}[\citep{Muandet20:KCM}, Theorem 3.3]\label{lem:analytic-form}
Assume that $\ep[(Y-f(X))^2k(Z,Z)] < \infty$. Then, we have the following closed-form expression, where $(X',Y',Z')$ is an independent copy of $(X,Y,Z)$:
\begin{equation}
    \label{eq:mmr-kernel} 
    \textstyle R_k(f) = \ep[(Y-f(X))(Y'-f(X'))k(Z,Z')].
\end{equation}
\end{lemma} We assume throughout that the reproducing kernel $k$ is integrally strictly positive definite (ISPD). 
\begin{assumption}\label{asmp:k_idp}
The kernel $k$ is continuous, bounded (i.e., $\sup_{z \in \inz}\sqrt{k(z,z)} < \infty$) and satisfies the condition of integrally strictly positive definite (ISPD) kernels, i.e., for every function $g$ that satisfies $0 < \|g\|_2^2 < \infty$, we have
$\iint_{\inz} g(z)k(z,z')g(z') \dd z \dd z'>0$.
\end{assumption}
Popular kernel functions that satisfy Assumption \ref{asmp:k_idp} are the Gaussian RBF kernel and Laplacian kernel
$$k(z,z') = \exp\left(-\frac{\|z-z'\|_2^2}{2\sigma^2}\right), \qquad k(z,z') = \exp\left(-\frac{\|z-z'\|_1}{\sigma}\right),$$ 
where $\sigma$ is a positive bandwidth parameter. 
Another important kernel satisfying this assumption is an inverse multiquadric (IMQ) kernel 
$$k(z,z') = (c^2 + \|z - z'\|_2^2)^{-\gamma}$$ 
where $c$ and $\gamma$ are positive parameters \citep[Ch. 4]{steinwart2008support}. 
This class of kernel functions is closely related to the notions of universal kernels \citep{Steinwart02:Universal} and characteristic kernels \citep{Fukumizu04:DR}. 
The former ensures that kernel-based classification/regression algorithms can achieve the Bayes risk, whereas the latter ensures that the kernel mean embeddings can distinguish different probability measures.
In principle, they guarantee that the corresponding RKHSs induced by these kernels are sufficiently rich for the tasks at hand.  
We refer the readers to \citep{Sriperumbudur11:Universal} and \citep{Simon-Gabriel18:KDE} for more details.

Next, we further assume the identification for the minimizer of $R_k(f)$.
\begin{assumption}\label{asmp:opt_ident}
Consider the function space $\mathcal{F}$ and $f^* \in \argmin_{f\in \mathcal{F}} R_k(f)$. Then for every $g \in \mathcal{F}$ with $|\ep[g(X)]| < \infty$, $\ep[g(X)-f^*(X)\,|\,Z]=0$ implies $g = f^*$ almost everywhere with respect to $P_X$.
\end{assumption}
A sufficient condition for identification follows from the \emph{completeness} property of $X$ for $Z$, e.g., the conditional distribution of $X$ given $Z$ belongs to the exponential family \citep{Newey03:NIV}. See \citep{Xavier11:Completeness} for more general sufficient conditions. Provided identification, it is straightforward to obtain consistency.

\section{Our Method}
\label{sec:mmr-iv}

We propose to learn $f$ by minimizing $R_k(f)$ in \eqref{eq:mmr-kernel}.
To this end, we define an optimal function $f^*$ as a minimizer of the above population risk w.r.t. a function class $\mathcal{F}$ of real-valued functions on $\inx$, i.e., 
$$\; f^* \in \argmin_{f\in\mathcal{F}}\; R_k(f).$$
It is instructive to note that population risk $R_k$ depends on the choice of the kernel $k$. Based on Assumption \ref{asmp:k_idp} and Lemma \ref{lem:analytic-form}, we obtain the following result by applying  \citep[Theorem 3.2]{Muandet20:KCM}, showing that $R_k(f) = 0$ if and only if $f$ satisfies the original CMR (see Appendix. \ref{sec:proof-sufficiency} for the proof).
\begin{proposition}\label{thm:orig_cond_m_res}
Assume that the kernel $k$ is ISPD. 
Then for every real-valued measurable function $f$, $R_k(f) = 0$ if and only if $\ep[Y - f(X)\,|\,Z=z] = 0$ for $P_Z$-almost every $z$.
\end{proposition}
Proposition \ref{thm:orig_cond_m_res} holds as long as the kernel $k$ belongs to a class of ISPD kernels. 
Hence, it allows for more flexibility in terms of the kernel choice.
Moreover, it is not difficult to show that $R_k(f)$ is strictly convex in $f$; see the proof in Appendix. \ref{pf:convex}.

\begin{proposition}\label{thm:convex}
If $\inF$ is a convex set and \sref{Assumptions}{asmp:k_idp}, \ref{asmp:opt_ident} hold, then the risk $R_k$ given in \eqref{eq:mmr-kernel} is strictly convex on $\inF$.
\end{proposition}

\subsection{Empirical Risk Minimization}
\label{sec:erm}

The previous results pave the way for an empirical risk minimization (ERM) framework \citep{Vapnik98:SLT} to be used in our work. 
That is, given an i.i.d. sample $\{(x_i,y_i,z_i)\}_{i=1}^n \sim P^n(X,Y,Z)$ of size $n$, empirical estimates of the risk $R_k(f)$ can be obtained as in the form of U-statistic or V-statistic \citep[Section 5]{Serfling80:Approximation}: 
\begin{align*}
    \widehat{R}_U(f) &\coloneqq \frac{1}{n(n-1)}\sum_{i=1}^n\sum_{j\neq i}(y_i-f(x_i))(y_j-f(x_j))k(z_i,z_j) \\
    \intertext{and}
    \widehat{R}_V(f) &\coloneqq \frac{1}{n^2}\sum_{i=1}^n\sum_{j=1}^n(y_i-f(x_i))(y_j-f(x_j))k(z_i,z_j).
\end{align*}
Both forms of empirical risk can be used as a basis for a consistent estimation of $f$.
The advantage of $\widehat{R}_U$ is that it is a minimum-variance unbiased estimator with appealing asymptotic properties, whereas $\widehat{R}_V$ is a biased estimator of the population risk \eqref{eq:mmr-kernel}, i.e., $\ep[\widehat{R}_V] \neq R_k$.
However, the estimator based on V-statistics employs full pairs of samples and hence may yield better estimate of the risk than the U-statistic counterpart. Let $\bm x := [x_1,\ldots,x_n]^\top$, $\bm y:=[y_1,\ldots,y_n]^\top$ and $\bm z:=[z_1,\ldots,z_n]^\top$ be column vectors. 
Let $K_{\mathbf{z}}$ be the kernel matrix $K(\bm z, \bm z)=[k(z_i,z_j)]_{ij}$ evaluated on the instruments $\mathbf{z}$ and $f(\bm x) := [f(x_1),\ldots,f(x_n)]^{\top}$.
Then, both $\widehat{R}_U$ and $\widehat{R}_V$ can be rewritten as
\begin{equation}\label{eq:quadratic-form}
    \widehat{R}_{V(U)}(f) = (\bm y-f(\bm x))^{\top} W_{V(U)} (\bm y-f(\bm x)),
\end{equation}
where $W_{V(U)}$ is a $n\times n$ symmetric weight matrix that depends on the kernel matrix $K_{\bm z}$. 
Specifically, $W_U = (K_{\bm z} - \mathrm{diag}(k(z_1,z_1),\ldots,k(z_n,z_n)))/(n(n-1))$ corresponds to $\widehat{R}_U$,
where $\mathrm{diag}(a_1,\ldots,a_n)$ denotes an $n\times n$ diagonal matrix whose diagonal elements are $a_1,\ldots,a_n$.
As shown in Appendix. \ref{pf:indefinite_wu}, $W_U$ is indefinite and may cause problematic inferences.
For $\widehat{R}_V$, $W_V:=K_{\bm z}/n^2$ is positive definite because the kernel $k$ is ISPD. 
Finally, our objective \eqref{eq:quadratic-form} also resembles the well-known generalized least regression with correlated noise \citep[Chapter 2]{kariya2004generalized} where the covariance matrix is the $Z$-dependent invertible matrix $W_{V(U)}^{-1}$.

Based on $\widehat{R}_U$ and $\widehat{R}_V$, we estimate $f^*$ by minimizing the  \emph{regularized} empirical risk over $\inF$:
\begin{align}
    \textstyle \hat{f}_{V(U)}\in \argmin_{f\in\mathcal{F}}\; \widehat{R}_{V(U)}(f) +\lambda \Omega(f) 
\end{align}
where $\lambda >0$ is a regularization constant satisfying $\lim_{n \rightarrow \infty} \lambda = 0$, and $\Omega(f)$ is the regularizer. 
Since $\hat{f}_U$ and $\hat{f}_V$ minimize objectives which are regularized U-statistic and V-statistic, they can be viewed as a specific form of M-estimators; see, e.g., \citep[Ch. 5]{Vaart00:Asymptotic}.
In this work, we focus on the V-statistic empirical risk and provide practical algorithms when $\inF$ is parametrized by deep neural networks (NNs) and an RHKS of real-valued functions. 

\textbf{Kernelized GMM.}
We may view the objective \eqref{eq:quadratic-form} from the GMM perspective \citep{Hall05:GMM}.
The assumption that the instruments $Z$ are exogenous implies that $\ep[\Phi_k(Z)\varepsilon] = 0$ where $\Phi_k$ denotes the canonical feature map associated with the kernel $k$.
This gives us an infinite number of moments, $g(f) = \ep[\Phi_k(Z)(Y - f(X))]$.
Hence, we can write the sample moments as $\hat{g}(f)=(1/n)\sum_{i=1}^n\Phi_k(z_i)(y_i - f(x_i))$.
The intuition behind GMM is to choose a function $f$ that sets these moment conditions as close to zero as possible, motivating the objective function 
\begin{eqnarray*} 
J(f) &:=& \|\hat{g}(f)\|_{\mathcal{H}_k}^2 
= \langle \hat{g}(f),\hat{g}(f) \rangle_{\mathcal{H}_k}\\
    &=& \frac{1}{n^2}\sum_{i,j}(y_i-f(x_i))\langle \Phi_k(z_i),\Phi_k(z_j)\rangle_{\mathcal{H}_k} (y_j - f(x_j))\\
    &=& \widehat{R}_V(f).
\end{eqnarray*}
Hence, our objective $\widehat{R}_V(f)$ \eqref{eq:quadratic-form} is a special case of the GMM objective when the weighting matrix is the identity operator. 
\citep[Ch. 6]{Carrasco07:LIP} showed the optimal weighting operator in terms of the inversed covariance operator.

{\bc
\subsubsection{Theoretical Perspective: Identification and Efficiency}
Before introducing practical algorithms, we discuss the theoretical aspects of our framework, specifically, the \emph{identifiability} and \emph{efficiency} issues by changing the norm constraint in the formulation.
The key difference between our risk $R_k(f)$ in \eqref{eq:mmr} and the ordinary framework of the unconstrained moment restriction is that, in the ordinary framework, the CMR \eqref{eq:umr} is rewritten to the following unconditional moment restriction
\begin{align}
    \sup_{h: \mathbb{E}[(f(X))^2] < \infty} (\mathbb{E}[(Y-f(X))h(Z)])^2. \label{eq:UMR_l2}
\end{align}
Note that the test function $h$ is maximized over a class of square integrable functions, i.e., $\mathbb{E}[(h(X))^2] < \infty$.
In contrast, our risk $R_k(f)$ utilizes a maximized test function $h$ over a unit ball of the RKHS, that is, $h \in \mathcal{H}_k, \|h\|\leq 1$ with the RKHS norm $\|\cdot\|$.
This arbitrary design with the RKHS plays a critical role in our method.

\begin{description}
\item[\textbf{Identification:}]
An advantage of our design \eqref{eq:mmr} with RKHSes is that it uniquely defines the function that minimizes the risk. 
Since RKHSes equips a positive definite kernel, and we further assume the kernel is ISPD in Assumption \ref{asmp:k_idp},  the risk $R_k(f)$ has a positive definite Hessian matrix (operator) at the minimum, hence it has the unique minimizer.
In short, the loss design with RKHSes has the advantage that the function class of $h$ can be chosen so that the solution is uniquely identified.

\item[\textbf{Efficiency of IV:}]
Another important issue is the efficiency of the estimator, i.e., the variance of the asymptotic distribution of estimators. 
For the linear regression problem with finite-dimensional parameters, we can define the \textit{optimal IV}, i.e., the estimator with the IV that minimizes the asymptotic variance \citep{Newey03:NIV}. 
In contrast, our regression model in this setup is nonlinear, infinite-dimensional, and also based on U-statistics, hence it is not trivial to discuss optimality.
\end{description}

To address the latter issue, we provide the variance of the asymptotic distribution of our estimator as a way to measure efficiency. 
In Proposition \ref{prop:efficiency} in Section \ref{sec:theory}, we derive asymptotic Gaussian approximations of the estimator, both pointwise and uniformly, that reveal their variance.
The result in Proposition \ref{prop:efficiency} is summarized as follows.
We firstly define coefficients $\Lambda_{\mathrm{max}}$ and $\Lambda_{\mathrm{min}}$ as
\begin{align*}
    \Lambda_{\mathrm{max}} := \sup_{g: \|g\|_{L^2} = 1} \iint_{\mathcal{Z}} g(z) k(z,z') g(z') \mathrm{d} z \mathrm{d} z', ~\mathrm{and}~\Lambda_{\mathrm{min}} := \inf_{g: \|g\|_{L^2} = 1} \iint_{\mathcal{Z}} g(z) k(z,z') g(z') \mathrm{d} z \mathrm{d} z',
\end{align*}
where $\|g\|_{L^2} = \int_{\mathcal{Z}}g(z)^2 \mathrm{d}z$ is the L2 norm.
The value $\Lambda_{\mathrm{max}}$ and $\Lambda_{\mathrm{min}}$ is the largest/smallest eigenvalue of an embedding operator to the RKHS.
Here, we define $\delta_\Lambda := \max\{|\Lambda_{\mathrm{max}} - 1|, |1 - \Lambda_{\mathrm{min}}|\}$.
Then, we show that an upper bound of the asymptotic variance with the proposed estimator is increased by 
\begin{align*}
    O\left(\Lambda_{\mathrm{max}} \delta_\Lambda + \{\Lambda_{\mathrm{max}} \delta_\Lambda\}^2 \right),
\end{align*}
compared to the case without kernel functions.
The case with $\delta_\Lambda = 1$ corresponds to the unconditional moment restriction with the L2 norm, as displayed in \eqref{eq:UMR_l2}.
This indicates that kernels that are concentrated in the direction of a particular eigenfunction may have worse efficiency.
More details will be provided in Section \ref{sec:efficiency}.

}

\subsection{Practical MMR-IV Algorithms}
\label{sec:practical-mmr}

A workflow of our algorithm based on $\widehat{R}_V$ is summarized in  \sref{Algorithm}{alg:mmriv}; we leave the $\widehat{R}_U$ based method to future work to solve the inference issues caused by indefinite $W_U$.
We provide examples of the class $\mathcal{F}$ in both parametric and non-parametric settings below.

\textbf{Deep neural networks.}
In the parametric setting, the function class $\mathcal{F}$ can often be expressed as $\mathcal{F}_{\Theta} = \{f_{\theta}\,:\, \theta\in\Theta\}$ where $\Theta\subseteq\rr^m$ denotes a parameter space.
We consider a very common nonlinear model in machine learning $f(x) = W_0\Phi(x) + b_0$ where $\Phi: x \longmapsto \sigma_{h}(W_{h}\sigma_{h-1}(\cdots \sigma_1(W_1 x)))$ denotes a nonlinear feature map of a depth-$h$ NN.
Here, $W_i$ for $i=1,\ldots,h$ are parameter matrices and each $\sigma_i$ denotes the entry-wise activation function of the $i$-th layer.
In this case, $\theta = (b_0,W_0,W_1,\ldots,W_h)$. As a result, we can rewrite $\hat{f}_V$ in terms of their parameters as $\hat{\theta}_{V} \in \arg\min_{\theta\in\Theta}\,\widehat{R}_{V}(f_\theta) + \lambda \| \theta\|_2^2$
where $f_{\theta}\in\mathcal{F}_{\Theta}$. 
We denote $\theta^* \in \arg \min_{\theta \in \Theta}\; R_k(f_{\theta})$.
In what follows, we refer to this algorithm as \textbf{MMR-IV (NN)}; see \sref{Algorithm}{alg:mmriv-nn}.

\begin{algorithm}[t!]
        \caption{MMR-IV}
        \label{alg:mmriv}
        \begin{algorithmic}[1]
        \Require Dataset $D = \{\bm x, \bm y, \bm z\}$, kernel $k$ with parameters $\theta_k$, a function class $\mathcal{F}$, regularizer $\Omega(\cdot)$ and $\lambda$
        \Ensure The estimate of $f^*$ in $\mathcal{F}$.
        \State Compute the kernel matrix $K=k(\bm z, \bm z; \theta_k)$.
        \State Define the residual $\bm{\varepsilon}(f) = \bm{y} - f(\bm{x})$.
        \State $\hat{f}_{\lambda}\leftarrow\argmin_{f \in \mathcal{F}} \;n^{-2}\bm{\varepsilon}(f)^\top K\bm{\varepsilon}(f) + \lambda \Omega(\|f\|_{\mathcal{F}})$ 
        \State \textbf{return} $\hat{f}_{\lambda}$
        \end{algorithmic}
\end{algorithm}

\begin{algorithm}
    \caption{MMR-IV (NN)}
    \label{alg:mmriv-nn}
    \begin{algorithmic}[1]
    \Require Dataset $D = \{\bm x, \bm y, \bm z\}$, kernel function $k$ with parameters $\theta_k$, NN $f_\mathrm{NN}$ with parameters $\theta_{\mathrm{NN}}$, regularization parameter $\lambda$
    \Ensure predictive value at $\bm x_{*}$
    \State Compute $K=k(\bm z, \bm z; \theta_k)$
    \State $\hat{f}_{NN}=\argmin_{f_{\mathrm{NN}}} (\bm y - f_{\mathrm{NN}}(\bm x))^{\top} K (\bm y - f_{\mathrm{NN}}(\bm x))/n^2 + \lambda \| \theta_{\mathrm{NN}}\|_{2}^2$ 
    \State \textbf{return} $\hat{f}_{NN}(\bm x_*)$
    \end{algorithmic}
\end{algorithm}

\textbf{Kernel machines.} In a non-parametric setting, the function class $\inF$ becomes an infinite dimensional space.
In this work, we consider $\inF$ to be an RKHS $\mathcal{H}_l$ of real-valued functions on $\inx$ with a reproducing kernel $l:\inx\times\inx\to\rr$.
Then, the regularized solution can be obtained by $ \arg\min_{f\in\mathcal{H}_l}\,\widehat{R}_{V}(f)+\lambda \|f\|^2_{\mathcal{H}_l}$.
As per the representer theorem, every optimal $\hat{f}$ admits a form $\hat{f}(x) = \sum_{i=1}^{n}\alpha_i l(x,x_i)$ for some $(\alpha_1,\ldots,\alpha_n)\in\rr^n$ \citep{scholkopf2001generalized}, and
based on this representation, we rewrite the objective as
\begin{equation}
\label{eq:obj_reg0}
    \hat{f}_V=\arg\min_{\bm \alpha\in\rr^n}\, (\bm y-L\bm{\alpha})^{\top} W_{V} (\bm y-L\bm{\alpha})+\lambda \bm{\alpha}^{\top}L\bm{\alpha},
\end{equation}where $L = [l(x_i, x_j)]_{ij}$ is the kernel matrix on $\bm x$.  
For U-statistic version, the quadratic program \eqref{eq:obj_reg0} substitutes indefinite $W_U$ for $W_V$, so it may not be positive definite.
The value of $\lambda$ needs to be sufficiently large to ensure that \eqref{eq:obj_reg0} is definite.
On the other hand, the V-statistic based estimate \eqref{eq:obj_reg0} is definite for all non-zero $\lambda$ since $W_V$ is positive semi-definite.
Thus, the optimal $\widehat{\bm{\alpha}}$ can be obtained by solving the first-order stationary condition and if $L$ is positive definite, the solution has a closed form expression, $\widehat{\bm{\alpha}} = (LW_VL+\lambda L)^{-1}L W_V \bm y$.
Thus, we will focus on the V-statistic version in our experiments. In the following, we refer to this algorithm as \textbf{MMR-IV (RKHS)}.

\textbf{Nystr\"om approximation.}
The MMR-IV (RKHS) algorithm is computationally costly for large datasets as it requires a matrix inversion.
To improve the scalability, we resort to Nystr\"om approximation \citep{NIPS2000_1866} to accelerate the matrix inversion in $$\widehat{\bm \alpha} = (LW_VL+\lambda L)^{-1}L W_V \bm y.$$ 
First, we randomly select a subset of $m (\ll n)$ samples from the original dataset and construct the corresponding sub-matrices of $W_V$, namely, $W_{mm}$  and $W_{nm}$ based on this subset. 
Second, let $V$ and $U$ be the eigenvalue vector and the eigenvector matrix of $W_{mm}$.
Then, the Nystr\"om approximation is obtained as $W_V \approx \widetilde{U}\widetilde{V}\widetilde{U}^{\top}$ where $\widetilde{U} := \sqrt{\frac{m}{n}} W_{nm} U V^{-1}$ and $\widetilde{V} := \frac{n}{m} V$. 
We finally apply the Woodbury formula \citep[p. 75]{flannery1992numerical} to obtain
\begin{eqnarray}
    (LW_VL+\lambda L)^{-1}LW_V 
    &=& L^{-1} (W_V + \lambda L^{-1})^{-1}W_V
    \\
    &\approx&  
    \lambda^{-1}[ I - \widetilde{U} (\lambda^{-1} \widetilde{U}^{\top} L\widetilde{U} + \widetilde{V}^{-1})^{-1}\widetilde{U}^{\top}\lambda^{-1}L ]\widetilde{U}\widetilde{V}\widetilde{U}^{\top}.
\end{eqnarray}
We will refer to this algorithm as \textbf{MMR-IV (Nystr\"om)};  \sref{Algorithm}{alg:mmriv-rkhs}. 
The runtime complexity of this algorithm is $O(nm^2+n^2)$.

\begin{algorithm}
    \caption{MMR-IV (Nystr\"om)}
    \label{alg:mmriv-rkhs}
    \begin{algorithmic}[1]
    \Require Dataset $D = \{(\bm x, \bm y,\bm z)\}_{i=1}^{n}$, kernel functions $k$ and $l$ with parameters $\theta_k$ and $\theta_l$, regularization parameter $\lambda$, leave out data size $M$, Nystr\"om approximation sample size $M'$, LMOCV times $m$
    \Ensure predictive value at $\bm x_{*}$
    \State $K=k(\bm z, \bm z; \theta_k)$
    \State {\fontsize{10}{15}$\hat{\delta},\hat{\theta}_{l} = \argmin_{\delta, \theta_l} \sum_{i=1}^m (\bm b^{(i)}-\bm y_{M}^{(i)})^{\top}K_{M}^{(i)}(\bm b^{(i)}-\bm y_{M}^{(i)})$}\rtext{Equation \eqref{eq:LMOCV}}
    \State $K/n^2 \approx \widetilde{U}\widetilde{V}\widetilde{U}^{\top}$ \rtext{Nystr\"om Approx with $M'$ samples}
    \State {\fontsize{9}{15}$\hat{\alpha} = \lambda^{-1}[ I - \widetilde{U} (\lambda^{-1} \widetilde{U}^{\top} L\widetilde{U} + \widetilde{V}^{-1})^{-1}\widetilde{U}^{\top}\lambda^{-1}L] \widetilde{U}\widetilde{V}\widetilde{U}^{\top} \bm y$}
    \State \textbf{return} $\hat{f}(\bm x_*) = l(\bm x_*, \bm x; \hat{\theta}_l)\hat{\alpha}$
    \end{algorithmic}
\end{algorithm}

\section{Hyperparameter Selection}
\label{sec:model-selection}

\begin{figure}[t!]
        \centering
        \begin{tikzpicture}[scale=1]
            \node (x) at (0,0.8) [label=125:NN ($f$),point];
            \node (y) at (0,-0.8) [label=-125:RKHS ($f$),point];
            \node (z) at (-1.6,0) [label=left:KMMR,point];
            \node (a) at (1.6,-0.8) [label=right:Analytical CV,point];
            \node (b) at (1.6,0.8) [label=right:Standard CV,point];
            \node (c) at (1.6,0) [label=right:Median Heuristic ($k$),point];
            
            \path (z) edge [line width=1pt, color=blue] (x);
            \path (z) edge [line width=1pt] (y);
            
            \path (x) edge [line width=1pt, color=blue] (b);
            \path (y) edge [line width=1pt] (a);
            \path (x) edge [line width=1pt, dotted] (c);
            \path (y) edge [line width=1pt, dotted] (c);
            
        \end{tikzpicture}
        \vspace{8pt}
        \caption{A summary of proposed algorithms. Following the KMMR framework, two methods -- MMR-IV (NN) and MMR-IV (RKHS) -- are proposed where $f$ is parameterized by a neural network and an RKHS function, respectively. An analytical cross validation (CV) error is proposed for MMR-IV (RKHS) and both methods select the kernel $k$ by the median heuristic.}
        \label{fig:causal-graph2}
\end{figure}

We discuss the selection of hyper-parameters such as the regularization coefficient $\lambda$.
The hyper-parameters in the IV regression problem significantly affect the finite sample performance, but the selection process is non-trivial. 
We firstly give an overview of our cross-validation approach in Section \ref{sec:LMOCV}, then develop some techniques which makes the approach efficient in Section \ref{sec:baye_interpretation} and \ref{sec:analytical-form}.
Figure \ref{fig:causal-graph2} gives an overview the contents of this section.

\subsection{Leave-\texorpdfstring{$M$}{}-out-cross-validation}
\label{sec:LMOCV}

We utilize the Leave-$M$-out-cross-validation (LMOCV) approach for the hyper-parameter selection, which is described as follows.
Let $M$ be an integer such that $M \leq n$, and $\Lambda$ be a set of hyper-parameter candidates.

\begin{algorithmic}[1]
\State $K \gets n/M$
\State Split the data into $K$ parts: $[D_1,D_2,\ldots,D_K]$ where $|D_i|\approx M$ 
\For{each $(\lambda, \mathcal{F})\in\Lambda$}
    \For{each $D_i\in [D_1,D_2,\ldots,D_K]$}
        \State Solve $\hat{f}_{\lambda} \gets \arg\min_{f \in \mathcal{F}} \,\widehat{R}_V(f) + \lambda\Omega(f)$ using $D_{-i} = [D_1,\ldots,D_{i-1},D_{i+1},\ldots,D_K]$
        \State Evaluate $S(\lambda,i) \gets \widehat{R}_V(\hat{f}_{\lambda})$ using $D_i$
    \EndFor
    \State $S(\lambda) \gets (1/K)\sum_{k=1}^K S(\lambda,k)$
\EndFor
\State Output $\lambda^*, \mathcal{F}^* \gets \arg\min_{\lambda, \mathcal{F}} \, S(\lambda)$
\end{algorithmic}

One disadvantage of this approach is that it is computationally expensive because the empirical risks are optimized many times.
To alleviate this problem, we study an analytical form of the score function $S(\lambda)$ for MMR-IV (RKHS) and MMR-IV (Nystr\"om). 
The analytical error enables any fold of cross validation and in comparison, related solutions rely on a few fold of cross validation \citep{Dikkala20:Minimax,Muandet19:DualIV} or Median Heuristic \citep{Singh19:KIV} for this task, which however have rather limited selection capability. 
We propose this method in the next subsections.

\subsection{Gaussian Process (GP) Interpretation}
\label{sec:baye_interpretation}

Our interest is to obtain an analytical empirical risk for MMR-IV (Nystr\"om) with the V-statistic objective \eqref{eq:obj_reg0}.
Toward the goal, we relate our ERM form to a stochastic model with a Gaussian process (GP), inspired by \citep{vehtari2016bayesian}.
Different from the counterpart in the ordinary kernel regression, we need to apply additional analysis to deal with the additional weight matrix $W_V$ in the V-statistic risk \eqref{eq:obj_reg0}.

We build the connection between the V-statistic risk and the posterior distribution of a GP model in this subsection, after which we will show that the MMR-IV (Nystr\"om) estimator is equivalent to the Maximum-A-Posteriori estimator based on the GP. The relationship is inspired by the similarity of our objective function to that of GLS.

\textbf{Prior and Likelihood.} Let us consider a GP prior over a space of functions $f(\bm x) \sim \mathrm{GP}(\bm 0, \delta l( \bm x,  \bm x))$, where $\delta>0$ is a real constant, and a set of i.i.d. data $D = \{(x_i,y_i,z_i)\}_{i=1}^{n}$.
To define the likelihood $p(D\,|\,f)$, we recall from \sref{Lemma}{lem:analytic-form} that the risk $R_k(f)$ can be expressed in terms of two independent copies of random variables $(X,Y,Z)$ and $(X',Y',Z')$.
To this end, let $D' := \{(x'_i,y'_i,z'_i)\}_{i=1}^n$ be an independent sample of size $n$ with an identical distribution to $D$.
Given a pair of samples $(x,y,z)$ and $(x',y',z')$ from $D$ and $D'$, respectively, we then define the likelihood as
\begin{align}
p(\{(x,y,z),(x',y',z')\} \,|\, f)
&= \frac{ \exp\left[-\frac{1}{2}(y -f(x))k(z,z')(y'-f(x'))\right]}{ \int \exp\left[-\frac{1}{2}(\bar{y} -f(\bar x))k(\bar z,\bar z')(\bar y'-f(\bar x'))\right] d P\otimes P((\bar x,\bar y,\bar z),(\bar x',\bar y',\bar z'))} .
\end{align}
Here, $P$ is a joint distribution of $(X,Y,Z)$, and $P\otimes P$ is a product of the distribution. 
Let ${t} = \bar{y} -f(\bar x)$ and ${t}' = \bar{y}' -f(\bar x')$ and assume that the kernel $k$ is bounded.
Since $\int \exp(-\frac{1}{2}t k(\bar z,\bar z')t' ) d P \otimes P \leq \max\{ \int \exp(-\frac{1}{2} t^2 C)d P , \exp(-\frac{1}{2} (t')^2 C)d P\}$ with some $C > 0$, the likelihood function is integrable, hence well-defined.
Since the integration in the denominator does not depend on $(x,y,z),(x',y',z')$, we regard the denominator as a normalizing constant and simply denote
\begin{align*}
    p(\{(x,y,z),(x',y',z')\} \,|\, f) &\propto  \exp\left[-\frac{1}{2}(y -f(x))k(z,z')(y'-f(x'))\right].
\end{align*}
Hence, the likelihood on both $D$ and $D'$ can be expressed as
\begin{align}
 p(\{D,D'\}\,|\,f)
    &= \prod_{i=1}^{n} \prod_{j=1}^{n} p(\{(x_i,y_i,z_i),(x'_j,y'_j,z'_j)\} \,|\, f) \nonumber
    \\
    &\propto \exp\left[-\frac{1}{2}\sum_{i=1}^{n}\sum_{j=1}^n(y_i -f(x_i))k(z_i,z'_j
    )(y'_j-f(x'_j))\right] \label{eq:likelihood}.
\end{align}
This likelihood function is regarded to a generalized version of a loss function of generalized least squares with heteroskedastic noise \citep{greene2003econometric}.
That is, if we set $k(z_i,z_j) = \sigma_i^2$ if $i=j$ with $\sigma_i^2 > 0$, and $k(z_i,z_j) = 0$ otherwise, the likelihood corresponds to that of a linear regression model with Gaussian noise having heteroskedastic variance.
In practice, however, we only have access to a single copy of sample, i.e., $D$, but not $D'$.
One way of constructing $D'$ is through data splitting: the original dataset is split into two halves of equal size where the former is used to construct $D$ and the latter is for $D'$.
Unfortunately, this approach reduces the effective sample size that can be used for learning.
Alternatively, we propose to estimate the original likelihood \eqref{eq:likelihood} by using $M$-estimators: given the full dataset $D = \{(x_i,y_i,z_i)\}_{i=1}^n$, our approximated likelihood is defined as
\begin{align*}
    p(D\,|\,f) 
    &=  \prod_{i=1}^{n} \prod_{j=1}^{n} p(\{(x_i,y_i,z_i),(x_j,y_j,z_j)\} \,|\, f(x_i), f(x_j)) \\
    &\propto \exp\left[-\frac{1}{2}\sum_{i=1}^{n}\sum_{j=1}^n(y_i -f(x_i))k(z_i,z_j
    )(y_j-f(x_j))\right] \\
    &= \exp\left[-\frac{1}{2}(\bm y - f(\bm x))^{\top}K_{\bm z}(\bm y - f(\bm x))\right]. 
\end{align*}
This likelihood assumes correlated errors with the normal distribution. It is closely related to the standard GP regression which has a similar Gaussian likelihood, but the covariance matrix is computed on $\bm x$ rather than $\bm z$, and the generalized ridge regression as well as the weighted least regression.
The matrix $K_{\bm z}$, as a kernel matrix defined above, plays a role of correlating the residuals $y_i-f(x_i)$.
Based on the likelihood $p(D\,|\,f)$, the maximum likelihood (ML) estimator hence coincides with the minimizer of the unregularized version of our objective \eqref{eq:quadratic-form}.

\textbf{Maximum A Posteriori (MAP).} Combining the above likelihood function with the prior on $f(\bm x)$, the posterior probability of $f(\bm x)$ is
\begin{align*}
    p(f(\bm x) | D ) &=& N(f(\bm x) | \bm c, C) \propto p(f(\bm x)) p(D | f),\\
    C &=& (K_{\bm z} + (\delta L)^{-1})^{-1} =L(LW_VL + (\delta n^2)^{-1}L)^{-1}Ln^{-2},\\
    \bm c &=& C K_{\bm z} \bm y = L(LW_VL + (\delta n^2)^{-1}L)^{-1}L W_V \bm y.
\end{align*}

The maximum a posteriori (MAP) estimate of $f(\bm x)$ is simply $\bm c$,
\begin{align}
    \argmax_{f(\bm x)} \; \log p(f(\bm x) \,|\, D) 
    &= L\underbrace{(L W_V L+(\delta n^2)^{-1} L)^{-1}L W_V \bm y}_{=\hat{\alpha}}.
 \label{eq:MAP}
\end{align}
We can see that the MAP estimator returns the same result as that of the MMR-IV (RKHS) estimator given that $(n^2\delta)^{-1} = \lambda$ (see the context around \eqref{eq:obj_reg0}), and $\delta$ plays the role of the regularization parameter. For MMR-IV (Nystr\"om), we consider the GP model with $W_V$ approximated by $\widehat{U}\widehat{V}\widehat{U}^{\top}$ in $C$ and $\bm c$, and the conclusion still holds,
\begin{align*}
    p(f(\bm x) | D ) &=& N(f(\bm x) | \bm c', C') \\
    C' &=& \delta L [ I - \widetilde{U} (\lambda^{-1} \widetilde{U}^{\top} L\widetilde{U} + \widetilde{V}^{-1})^{-1}\widetilde{U}^{\top}\lambda^{-1}L]
    ,\\
    \bm c' &=& \delta n^2 L [ I - \widetilde{U} (\lambda^{-1} \widetilde{U}^{\top} L\widetilde{U} + \widetilde{V}^{-1})^{-1}\widetilde{U}^{\top}\lambda^{-1}L] \widehat{U}\widehat{V}\widehat{U}^{\top} \bm y.
\end{align*}

We summarize the result in Proposition \ref{thm:gp_interp}. Such a GP interpretation is used for an elegant derivation of the analytical cross-validation error in the following subsection.
\begin{proposition}\label{thm:gp_interp}
    Given $\delta=(\lambda n^2)^{-1}$ and $\hat{f} = \argmin_{f \in \mathcal{H}_l} \mathrm{Eqn.}\, \eqref{eq:obj_reg0}$ (including Nystr\"om approximated),
    $$\argmax_{f(\bm x_*)}\; p(f(\bm x_*)|D) = \hat{f}(\bm x_*).$$
\end{proposition}

\subsection{Analytical Form of Error}
\label{sec:analytical-form}

Now, we derive the analytical error of LMOCV via Bayes' rules based on the posterior GP. We split the whole dataset $D$ into training and development datasets, denoted as $D_{tr}$ and $D_{de} \coloneqq \{\bm x_{de}, \bm y_{de}, \bm z_{de}\}$ respectively, where $D_{de}$ has $M$ triplets of data points. Given $D_{tr}$, the predictive probability of $f(D_{de})$ can be obtained by Bayes' rules (see also Eqn.(22) in \citep{vehtari2016bayesian}) 
\begin{align}
    p(f(\mathbf{x}_{de})|D) &= p(f(\bm x_{de}) | D_{tr}, D_{de}) \notag \\
    &= p(D_{de} | f(\bm x_{de}), D_{tr})p(f(\bm x_{de}) | D_{tr})/p(D_{de}| D_{tr})\notag \\
    &= p(D_{de} |f(\bm x_{de}))p(f(\bm x_{de})|D_{tr})/p(D_{de}| D_{tr}), \notag \\
    \intertext{which implies that}
    p(f(\bm x_{de})|D_{tr}) &= p(f(\bm x_{de})|D) p(D_{de}| D_{tr}) / p(D_{de} |f(\bm x_{de})), \label{eq:p_fde_Dtr}
\end{align}
where $p(f(\bm x_{de}) | D)=N (\bm c_{de}, C_{de})$ with $\bm{c}_{de}$, $C_{de}$ the mean and covariance of $f(\bm x_{de})$ in $p(f(\bm x)| D)$ and $p(D_{de} | f(\bm x_{de}))\propto N(\bm y_{de}, K_{de}^{-1})$, $K_{de}=k(\bm z_{de}, \bm z_{de})$. From this result, we can see that $p(D_{de} |f(\bm x_{de}))$ and $p(f(\bm x_{de})|D_{tr})$ are in Gaussian  forms and $p(D_{de}| D_{tr})$ is a constant term, so $p(f(\bm x_{de}) | D_{tr})$ must be  Gaussian:
$$p(f(\bm x_{de}) | D_{tr}) = N(\bm b, B)$$ where $ B^{-1}=C_{de}^{-1}-K_{de}$ and $
\bm b =B(C_{de}^{-1}\bm c_{de}-K_{de} \bm y_{de})$ can be obtain by solving Eqn.\eqref{eq:p_fde_Dtr}. By \sref{Prop.}{thm:gp_interp}, we know that the prediction on the validation data $f(x_{de})$ given the training data $D_{tr}$ is $\bm b$. Therefore, the error of $m$ repeated LMOCV is  computed as below \eqref{eq:LMOCV}, where we use $(i)$ to denote the $i$-th split of the whole dataset into $D_{de}$ and $D_{tr}$ and $C_{de}^{(i)}$ and as mentioned just above, $\bm{c}_{de}^{(i)}$ are straightforwardly obtained from $p(f(\bm x)|D)$. 
As a result, this enables an efficient parameter selection with the CV principle: the standard LMOCV procedure requires multiple experiments (e.g. $m$ here) on different $D_{tr}$ and $D_{de}$, while our result needs only a single training on $D$ and the validation error of $m$ experiments is simply gained by plugging the training result into the analytical error \eqref{eq:LMOCV}:
\begin{align}\label{eq:LMOCV}
    \text{LMOCV Error} & \coloneqq  \sum_{i=1}^{m} (r^{(i)})^{\top} K_{de}^{(i)}r^{(i)}, \quad 
     r^{(i)}  \coloneqq  \bm b^{(i)}-\bm y_{M}^{(i)} = (I- C_{de}^{(i)} K_{de}^{(i)})^{-1} (\bm c_{de}^{(i)} - \bm y_{de}^{(i)}).
\end{align}

\section{Theoretical Result: Consistency and Asymptotic Normality}
\label{sec:theory}

We provide the consistency and asymptotic normality of $\hat{f}_V$. 
Similar results for $\hat{f}_U$, together with all the proofs, are provided in the appendix. 
Readers interested in the finite sample rate should consult \citep[E.4]{Dikkala20:Minimax} for further details.

\subsection{Consistency.}
We first show the consistency of $\hat{f}_V$,
which depends on the uniform convergence of the risk functions.
The result holds for both parametric and non-parametric cases regardless of the shape of $\Omega(f)$, so we can utilize the regularization $\|\theta\|_2^2$ which is common for NN but non-convex in terms of $f$.
In the following result on the consistency, we consider that $\inF$ has topology induced by distance for parameters $\theta \in \Theta$ that characterize functions, that is, for functions $f_\theta, f_{\theta'} \in \inF$, their distance is measured by that of their parameters $\theta$ and $\theta'$: $d(f_\theta, f_{\theta'}) = \|\theta - \theta'\|_2$.
\begin{proposition}[Consistency of $\hat{f}_V$] \label{thm:cons_emp_fv_nonconvex}
Assume that $\ep[|Y|^2]<\infty$, $\ep[\sup_{f\in \inF}|f(X)|^2]<\infty$, $\inF$ is compact, \sref{Assumption}{asmp:k_idp}, \ref{asmp:opt_ident} hold, $\Omega(f)$ is a bounded function and $\lambda \pto 0$. Then  $\hat{f}_V \pto f^*$.
\end{proposition}
If $\Omega(f)$ is convex in $f$, the consistency can be obtained more easily by \citep[Thm. 2.7]{NEWEY19942111} which  relies on the convexity of the risk functions.
In this case, we can avoid several conditions.
We provide an additional result with this setting in Appendix. \ref{pf:cons_emp_fv_convex}.

\subsection{Asymptotic Normality with Finite-Dimension Case}
We analyze asymptotic normality of the estimator $\hat{f}_V$, which is important to advanced statistical analysis such as tests.
The results are obtained by applying delta method with finite parameters and functionals respectively.

We first consider the $\hat{f}_V$ that is characterized by a finite-dimensional parameter from a parameter space $\Theta$. 
We rewrite the regularized V-statistic risk as a compact form $\widehat{R}_{V,\lambda}(f_{\theta})\coloneqq \frac{1}{n^2}\sum_{i,j}h_{\theta}(u_i,u_j)+\lambda \Omega(\theta)$, $h_{\theta}(u_i,u_j)\coloneqq (y_i-f_{\theta}(x_i))(y_j-f_{\theta}(x_j))k(z_i,z_j)$, and consider $R_k(f_\theta)$ is uniquely minimized at $\theta^* \in \Theta$ and $\rightsquigarrow$ denotes a convergence in law.
\begin{theorem}[Asymptotic Normality of $\hat{\theta}_V$]\label{thm:asym_norm_theta_v}
Suppose that $f_\theta$ and $\Omega(\theta)$ are twice continuously differentiable about $\theta$, $\Theta$ is compact, $H= \ep[\nabla^2_{\theta} h_{\theta^*}(U,U')]$ is non-singular,  $\ep[|Y|^2]<\infty$, $\ep \left [\sup_{\theta\in \Theta}|f_\theta (X)|^2\right]<\infty$, $\ep \left [\sup_{\theta\in \Theta} \|\nabla_\theta f_\theta (X)\|_2^2\right]<\infty$, $\ep \left [\sup_{\theta\in \Theta} \|\nabla_\theta^2 f_\theta (X)\|_F^2\right]<\infty$, $\sqrt{n}\lambda \pto 0$, $R_k(f_\theta)$ is uniquely minimized at $\theta^*$ which is an interior point of $\Theta$,  and \sref{Assumption}{asmp:k_idp} holds. Then,
\begin{align*}
    \sqrt{n} (\hat{\theta}_V -\theta^*) \rightsquigarrow N(\bm 0, \Sigma_V)
\end{align*}
holds, where $\Sigma_V = 4H^{-1}\mathrm{diag}(\ep_U[\ep^{2}_{U'}[h_{\theta^*}(U,U')]])H^{-1}$.
\end{theorem}
{\bc 
\subsubsection{Connection to Efficiency} \label{sec:efficiency}

The variance of the asymptotic distribution varies with the choice of kernel.
To see this, we consider the constants $\Lambda_{\mathrm{min}} \leq \Lambda_{\mathrm{max}}$ such that
\begin{align*}
    \Lambda_{\mathrm{min}} = \inf_{g : \|g\|_{L^2} = 1} \iint_{\mathcal{Z}} g(z) k(z,z') g(z') \mathrm{d} z \mathrm{d} z'\leq \sup_{g : \|g\|_{L^2} = 1} \iint_{\mathcal{Z}} g(z) k(z,z') g(z') \mathrm{d} z \mathrm{d} z'= \Lambda_{\mathrm{max}}.
\end{align*}
Their positiveness and existence are guaranteed by Assumption \ref{asmp:k_idp}.
By Theorem 4.51 in \citep{steinwart2008support}, $\Lambda_{\mathrm{min}}$ and $\Lambda_{\mathrm{max}}$ correspond to the smallest and largest eigenvalue of the embedding operator $T:L^2 \to \mathcal{H}_k$ where $\mathcal{H}_k$ is the RKHS whose reproducing kernel is $k$, respectively.
Equivalently, by the Mercer's theorem (Theorem 4.49 in \citep{steinwart2008support}), we have the form of $k$ as
\begin{align*}
    k(z,z') = \sum_{j} \lambda_j \phi_j(x) \phi_j(z'),
\end{align*}
where $\{\lambda_j\}_j$ are coefficients and $\phi_j(\cdot)$ is an orthonormal basis function, and $\Lambda_{\mathrm{max}} = \max_j \lambda_j$ and $\Lambda_{\mathrm{min}} = \min_j \lambda_j$.
We also define an asymptotic variance \textit{without} the kernel $k$.
That is, we define a loss without the kernel as $\check{h}_{\theta}(u_i,u_j)\coloneqq (y_i-f_{\theta}(x_i))(y_j-f_{\theta}(x_j))$, and its associated matrices $\check{H}= \ep[\nabla^2_{\theta} \check{h}_{\theta^*}(U,U')]$ and $\check{\Sigma}_V = 4\check{H}^{-1}\mathrm{diag}(\ep_U[\ep^{2}_{U'}[\check{h}_{\theta^*}(U,U')]])\check{H}^{-1}$.
Let $\|\Sigma_V\|_F = \sqrt{\mathrm{tr}(\Sigma_V^2)}$ be the Frobenius norm for square matrices.

\begin{proposition} \label{prop:efficiency}
    Suppose that Assumption \ref{asmp:k_idp} holds.
    Define $\delta_\Lambda := \max\{|\Lambda_{\mathrm{max}}-1|, |1 - \Lambda_{\mathrm{min}}|\}$.
    Then, we have
    \begin{align*}
        \|\Sigma_V\|_F \leq \|\check{\Sigma}_V\|_F + O(\Lambda_{\mathrm{max}}\delta_\Lambda + (\Lambda_{\mathrm{max}}\delta_\Lambda)^2).
    \end{align*}
\end{proposition}

This result evaluates the increase in the asymptotic variance due to the choice of kernel by the upper bound. 
Strictly speaking, as the eigenvalues $\Lambda_{\mathrm{min}}$ and $\Lambda_{\mathrm{max}}$ of the kernel $k$ approach $1$, i.e., as the kernel function approaches a constant function, the asymptotic variance decreases at the speed of the first and second power of $\delta_\Lambda$. 
Therefore, this result demonstrates that utilizing the simplest possible kernel, one that is not focused on a specific eigenspace, is the method for preventing an increase in the asymptotic variance.

Intuitively, this result represents the difficulty of choosing the kernel function as the optimal IV in nonlinear IV regression. 
Unlike the case of linear regression \citep{Newey93:CMR}, the nonlinear IV regression does not provide information such as which eigenspaces of data have large variance. 
Thus, the anisotropy of eigenvalues of kernels may worsen the efficiency in the worst case. 
However, efficiency can be improved if the anisotropy of the eigenvalues fit the data. 
If such properties could be known, one can develop an optimal kernel selection procedure based on them.

}

\subsection{Asymptotic Normality of Infinite-Dimension Case}
We show the asymptotic normality of an infinite-dimensional estimator $\hat{f}_V$.
That is, we show that an error of $\hat{f}_V$ weakly converges to a Gaussian process that takes values in a function space $\mathcal{H}_l$.
We set $\Omega(f) = \|f\|_{\mathcal{H}_l}^2$ and consider a minimizer: $f^*_{\lambda_0} \in \argmin_{f\in\mathcal{F}}\; R_k(f) + \lambda_0 \|f\|_{\mathcal{H}_l}^2$ with arbitrary $\lambda_0 > 0$.
We also define $\mathcal{N}(\varepsilon, \mathcal{H}, \|\cdot\|)$ as an $\varepsilon$-covering number of $\mathcal{H}$ in terms of $\|\cdot\|$.
Then, we obtain the following result which allows statistical inference on functional estimators such as kernel machines.

\begin{assumption}[Low-Entropy Condition] \label{asmp:entropy}
    There exists $s \in (0,2)$ and a constant $C_H > 0$ such that
    \begin{align*}
        \log \mathcal{N}(\varepsilon, \mathcal{H}_l, \|\cdot\|_{L^\infty}) \leq C_H \varepsilon^{-s},
    \end{align*}
    for every $\varepsilon \in (0,1)$.
\end{assumption}
Although the covering number condition $\log \mathcal{N}(\varepsilon, \mathcal{H}_l, \|\cdot\|_{L^\infty}) \leq C_H \varepsilon^{-s}$ in Theorem \ref{thm:asympnorm} is strong, we can present some examples that satisfy the condition.
    A representative example is the set of Lipschitz functions on a bounded interval. The condition is then satisfied under $s=1$ (Example 5.10 in \citep{wainwright2019high}). Although the support is one-dimensional, it can cover non-differentiable functions.
    Another example is the set of $b$-times differentiable functions on a compact set in $\mathbb{R}^d$ with $b > 2d$ (Theorem 2.7.1 in \citep{vanweak}). 
    Although requires higher-order smoothness, it can handle functions with multidimensional inputs.
    The last example is the RKHS with kernel functions characterized by the spectrum: the condition is satisfied when $s$ is a decay rate of the eigenvalues of the embedding operator into the RKHS. 
    Several common kernels, such as the Gaussian kernel (see Section 4 in \citep{steinwart2008support}) and the Mercer kernel (explained in \citep{zhou2002covering}), satisfy it with a certain parameter configuration.
    This characterization is used in several machine learning studies, e.g., \citep{steinwart2009optimal,caponnetto2007optimal}.

\begin{theorem}[Asymptotic Normality of $\hat{f}_V$] \label{thm:asympnorm}
    Suppose \sref{Assumption}{asmp:k_idp} holds,  $l$ is a bounded kernel, $k$ is a uniformly bounded function, and  $\lambda - \lambda_0 = o(n^{-1/2} )$ holds. 
    Also, suppose that $\inx$, $\inz$, and $\iny$ are compact spaces, and Assumption \ref{asmp:entropy} holds.
    We define a Gaussian process $\mathbb{G}_P^*$ on $\mathcal{H}_l$ with zero mean and its covariance
    \begin{align*}
        \mathrm{Cov}(G_P^*(f), G_P^*(f')) =  \ep_{U}[P^1{h}_f(U) P^1{h}_{f'}(U)] - \ep_{U}[P^1{h}_f(U)]\ep_{U}[P^1{h}_{f'}(U)]
    \end{align*}
    with $h_f(u,u') = (y - f(x)) (y' - f(x')) k(z, z')$ and $P^1{h}_f(\cdot) = (\int h_f(u,\cdot) + h_f(\cdot,u) \text{d} P(u))/2$.
    Then, there exists a linear operator $\nabla S_{P^2, \lambda_0}: \mathcal{H}_l \to \mathcal{H}_l$ such that
    \begin{align*}
        \sqrt{n}(\hat{f}_{V} - f^*_{\lambda_0}) \rightsquigarrow \nabla S_{P^2, \lambda_0}(\mathbb{G}_P^*) \mbox{~in~} \mathcal{H}_l.
    \end{align*}
    Note that $\nabla S_{P^2, \lambda_0}(\mathbb{G}_P^*)$ is also a Gaussian process on $\mathcal{H}_l$ owing to the linearily of $\nabla S_{P^2, \lambda_0}$ (Section 3.9.2 in \citep{vanweak}).
\end{theorem}

We provide the detailed derivation in the appendix.
A specific form of the operator $\nabla S_{P^2, \lambda_0}$ is also given in the appendix.
All conditions for the theorem are valid for many well-known kernels. 
For the boundedness assumption, many common kernels, such as the Gaussian RBF kernel, the Laplacian kernel, and the Mercer kernel, satisfy it.

This nonparametric asymptotic normality is a generic result that leads to finite-dimensional asymptotic normality.
That is, for every finite set $\{x_j\}_{j=1}^N \subset \mathcal{X}$, Theorem \ref{thm:asympnorm} immediately implies the following convergence:
\begin{align*}
    \sqrt{n} 
    \begin{pmatrix}
        \hat{f}_{V}(x_1) - f^*_{\lambda_0}(x_1) \\
        \vdots \\
        \hat{f}_{V}(x_N) - f^*_{\lambda_0}(x_N)
    \end{pmatrix}
    \rightsquigarrow \mathcal{N}_N(0, \Sigma),
\end{align*}
where $\mathcal{N}_N(0,\Sigma)$ is the $N$-variate normal distribution with zero mean and a corresponding covariance matrix $\Sigma$.
This result is a convenient generalization of asymptotic normality in a parametric setting.

\begin{remark}[On asymptotic normality]
Theorem \ref{thm:asympnorm} and
$\lambda_0 >0$ results in a biased center $f_{\lambda_0}^*$ while it is essential to obtain the limit distribution in infinite dimensional space.
This is one of the reasons why $\sqrt{n}$-convergence can be achieved in Theorem \ref{thm:asympnorm}.
Although it is not very satisfactory, we note that deriving the limit distribution has a different nature than obtaining error bounds and it is more restrictive than to develop a learner with bounded generalization errors. 
One can find that similar difficulty appears in deriving a limit with the non-Donsker class \citep{vanweak}.
To remove this constraint, we need to introduce some strong assumption, such as the low noise condition in the classification problem, as discussed in \citep{hable2012asymptotic}.
However, we do not solve this problem in this study because such a response would go far beyond the scope of this work.
\end{remark}

\section{Related Work} 
\label{sec:related-work}


Several extensions of 2SLS and GMM exist for the nonlinear IV problem.
In the two-stage approach, the function $f(x)$ has often been obtained by solving a Fredholm integral equation of the first kind $\mathbb{E}[Y|Z] = \int f(x) \dd P(x|Z)$.
In \citep{Newey03:NIV,Blundell07:Semi-NIV,horowitz2011applied,Chen12:NIV}, linear regression is replaced by a linear projection onto a set of known basis functions.
Like the kernel choice problem in our work, this approach requires choosing the appropriate set of basis functions.
A uniform convergence rate of this approach is provided in \citep{Chen18:NIV-Uniform}.
In \citep{Hall05:NIV} and \citep{Darolles11:NIV}, the first-stage regression is replaced by a conditional density estimation of $\mathbb{P}(X|Z)$ using a kernel density estimator.
On the contrary, our approach does not rely on the conditional density estimation problem.

The IV regression has also recently received attention in the machine learning community. 
\citep{Hartford17:DIV} proposed to solve the integral equation by first estimating $P(X|Z)$ with a mixture of deep generative models on which the function $f(x)$ can be learned with another deep NNs.
Instead of NNs, \citep{Singh19:KIV} proposed to model the first-stage regression using the conditional mean embedding of $P(X|Z)$ \citep{Song10:KCOND,Song2013,Muandet17:KME} which is then used in the second-step kernel ridge regression.
In other words, the first-stage estimation in \citep{Singh19:KIV} becomes a vector-valued regression problem.
In an attempt to alleviate the two-stage-estimation spirit, \citep{Muandet19:DualIV} and \citep{Liao20:NeuralSEM} reformulate the two-stage procedure as a convex-concave saddle-point problem, which is in a minimax form, via studying the Fenchel and Lagrange dual reformulations of the population risk, respectively.
By using RKHSes for the inner maximization, DualIV \citep{Muandet19:DualIV} obtains a quadratic objective function similar to ours, but its RKHS is applied over $(Z,Y)$ which cannot be interpreted as a valid instrument.
In contrast, the approach of \citep[Appendix F]{Liao20:NeuralSEM} provides an exact dual reformulation of our method and obtains our objective if RKHSes are used for the inner maximization. 
Besides, the starting objectives of the above works differ from ours:
in \citep{Hartford17:DIV}, \citep{Singh19:KIV}, \citep{Muandet19:DualIV} and \citep[Appendix F]{Liao20:NeuralSEM}, they started from minimizing the $L_2$ norm of the CMR, whereas we start from minimizing the squared UMR. 
The fact that they arrive at the same objective hints a deeper connection which requires further investigation.

Our work follows in spirit many GMM-based approaches for IV regression, namely, \citep{Lewis18:AGMM,Bennett19:DeepGMM,Muandet20:KCM}. 
We adapt the MMR framework of \citep{Muandet20:KCM}, which  only considers a conditional moment testing problem, to the parameter estimation problem of IV regression, and hence our approaches and both of theoretical and experimental analyses are different from theirs.
In fact, this framework was initially inspired by \citep{Lewis18:AGMM} and \citep{Bennett19:DeepGMM} which instead parametrize the instruments by deep NNs.
By combining the GMM framework with RKHS functions, the objective function can be evaluated in closed-form as shown in our work. 
As a result, our IV estimate can be obtained by minimizing the empirical risk, as opposed to an adversarial optimization \citep{Lewis18:AGMM,Bennett19:DeepGMM}. It is important to note that recently these adversarial approaches have been improved similarly with RKHSs by \citep{Dikkala20:Minimax,Bennett20:variational}, while there is major difference. First, \citep{Dikkala20:Minimax} extends the work of \citep{Lewis18:AGMM} to a single-stage algorithm similar to our MMR-IV (RKHS) \citep[Section 4]{Dikkala20:Minimax}.
Although both works employ RKHSs in the minimax frameworks,  \citep{Dikkala20:Minimax} incorporate a Tikhonov regularization on $h$ in \eqref{eq:mmr} and resort to the representer theorem \citep{Scholkopf01:Representer} to develop the analytical objective function, whereas we impose a unit-ball constraint which is a form of Ivanov regularization \citep{Ivanov02:Ill-posed}.
the basis of our objective is the adversarial GMM objective of \citep{Lewis18:AGMM}. 
Furthermore, prior works do not study the analytical cross validation error. 

We also discuss connections to existing works in the areas outside of the IV setting in Section \ref{sec:related-work-2} in Appendix.


\section{Experimental Results} 
\label{sec:experiments}

We present the experimental results in a wide range of settings for IV estimation. 
Following \citep{Lewis18:AGMM} and \citep{Bennett19:DeepGMM}, we consider both low and high-dimensional scenarios.
In the experiments, we compare our algorithms to the following baseline algorithms: 
\begin{itemize}
\item \texttt{DirectNN}: A standard least square regression on $X$ and $Y$ using a neural network (NN).
\item \texttt{2SLS}: A vanilla 2SLS on raw $X$ and $Z$. 
\item \texttt{Poly2SLS}: The 2SLS that is performed on polynomial features of $X$ and $Z$ via ridge regressions.
{\bc 
\item \texttt{SieveIV} \citep{chen2021adaptive}: This algorithm solves the CMR via the sieve method, which uses an orthonormal basis functions. In the implementation, we utilize the trigonometric basis, and the number of basis functions is selected by the Lepski's method.
The details are found in \citep{chen2021adaptive} and its related studies.
}
\item \texttt{DeepIV} \citep{Hartford17:DIV}: A nonlinear extension of 2SLS using deep NNs. We use the implementation available at \url{https://github.com/microsoft/EconML}.
\item \texttt{KernelIV} \citep{Singh19:KIV}: A generalization of 2SLS by modeling relations among $X$, $Y$, and $Z$ as nonlinear functions in RKHSs. We use the publicly available implementation at \url{https://github.com/r4hu1-5in9h/KIV}.   
\item \texttt{GMM+NN}: An optimally-weighted GMM \citep{Hansen82:GMM} is combined with a NN $f(X)$. The details of this algorithm can be found in \citep[Section 5]{Bennett19:DeepGMM}.
\item \texttt{AGMM} \citep{Lewis18:AGMM}: This algorithm models $h(Z)$ by a deep NN and employs a minimax optimization to solve for $f(X)$. The implementation we use is available at \url{https://github.com/vsyrgkanis/adversarial_gmm}.
\item \texttt{DeepGMM} \citep{Bennett19:DeepGMM}: This algorithm is a variant of \texttt{AGMM} with optimal inverse-covariance weighting matrix. 
The publicly available implementation at \url{https://github.com/CausalML/DeepGMM} is used and all of the above baselines are provided in the package except KernelIV.
\item \texttt{AGMM-K} \citep{Dikkala20:Minimax}: This algorithm extends AGMM by modeling $h(Z)$ and $f(X)$ as RKHSs. Nystr\"om approximation is applied for fast computation.  The publicly available implementation at \url{https://github.com/microsoft/AdversarialGMM} is used.
\item \texttt{DualIV} \citep{Muandet19:DualIV}: This algorithm solves the CMR via a saddle-point reformulation and obtains a minimax problem. By using a RKHS for the inner maximization, it gets a similar objective to ours while the central kernel matrix is evaluated on $Z$ and $Y$, which thus can't be viewed as instruments and different from ours. The publicly available implementation at \url{https://github.com/krikamol/DualIV-NeurIPS2020} is used.
\end{itemize}

\begin{table}[t!]
    \centering
    \caption{{\bc The mean square error (MSE) $\pm$ one standard deviation in the large-sample regime ($n = 2000$). Italics denote the second best. \label{tab:result_low_dim_large}}}
    \begin{tabular}{l c c c c} 
    \toprule
    \multirow{2}{*}{\textbf{Algorithm}} & \multicolumn{4}{c}{\textbf{True Function $f^*$}} \\
    & abs & linear & sin & step \\ 
    \midrule
    \texttt{DirectNN}    & .116 $\pm$ .000 & .035 $\pm$ .000 & .189 $\pm$ .000 & .199 $\pm$ .000 \\
    \texttt{2SLS} & .522 $\pm$ .000 & \textbf{.000} $\pm$ \textbf{.000} & .254 $\pm$ .000 & .050 $\pm$ .000 \\
    \texttt{Poly2SLS}    & .083 $\pm$ .000 & \textbf{.000} $\pm$ \textbf{.000} & .133 $\pm$ .000 & .039 $\pm$ .000 \\
    \texttt{GMM+NN}         & .318 $\pm$ .000 & .044 $\pm$ .000 & .694 $\pm$ .000 & .500 $\pm$ .000 \\
    \texttt{AGMM}        & .600 $\pm$ .001 & .025 $\pm$ .000 & .274 $\pm$ .000 & .047 $\pm$ .000 \\
    \texttt{DeepIV}      & .247 $\pm$ .004 & .056 $\pm$ .003 & .165 $\pm$ .003 & .038 $\pm$ .001 \\
    \texttt{DeepGMM}  & .027 $\pm$ .009 & .005 $\pm$ .001 & .160 $\pm$ .025 & \textit{.025} $\pm$ \textit{.006} \\
    \texttt{KernelIV}  & \textit{.019}  $\pm $ \textit{.000} & .009  $\pm $ .000 & \textit{.046}  $\pm$ \textit{.000} & .026  $\pm $ .000\\
    \texttt{AGMM-K} & 181 $\pm$ .000 & 2.34 $\pm$ .000 & 19.4 $\pm$ .000 & 4.13 $\pm$ .000 \\
    \texttt{DualIV} & .344 $\pm$ .000 & .034 $\pm$ .000 & .379 $\pm$ .000 & .345 $\pm$ .000 \\
    \texttt{SieveIV} & .170 $\pm$ .009 & .279 $\pm$ .001 & .021 $\pm$ .001 & 6.89 $\pm$ .044 \\
    \texttt{MMR-IV}\textsubscript{NN} & \textbf{.011} $\pm $ \textbf{.002} &.005  $\pm $ .000  & .153 $\pm$ .019  & .040 $\pm$ .004 \\
    \texttt{MMR-IV}\textsubscript{Nys} & \textbf{.011} $\pm$ \textbf{.001} & \textit{.001} $\pm$ \textit{.000} & \textbf{.006} $\pm$ \textbf{.002}  & \textbf{.020} $\pm$ \textbf{.002} \\
    \bottomrule
    \end{tabular}
\end{table}

\begin{table}[t!]
    \centering    
    \caption{{\bc The mean square error (MSE) $\pm$ one standard deviation in the small-sample regime ($n=200$). Italics denote the second best. \label{tab:result_low_dim_small}}}
    \begin{tabular}{l c c c c} 
    \toprule
    \multirow{2}{*}{\textbf{Algorithm}} & \multicolumn{4}{c}{\textbf{True Function $f^*$}} \\
    & abs & linear & sin & step \\ 
    \midrule
    \texttt{DirectNN}    & .143 $\pm$ .000 & .046 $\pm$ .000 & .404 $\pm$ .006 & .253 $\pm$ .000 \\
    \texttt{2SLS} & .564 $\pm$ .000 & \textbf{.003} $\pm$ \textbf{.000} & .304 $\pm$ .000 & .076 $\pm$ .000 \\
    \texttt{Poly2SLS}    & .125 $\pm$ .000 & \textbf{.003} $\pm$ \textbf{.000} & .164 $\pm$ .000 & .077 $\pm$ .000 \\
    \texttt{GMM+NN}         & .792 $\pm$ .000 & .203 $\pm$ .000 & 1.56 $\pm$ .001 & .550 $\pm$ .000 \\
    \texttt{AGMM}        & .031 $\pm$ .000 & .011 $\pm$ .000 & .330 $\pm$ .000 & .080 $\pm$ .000 \\
    \texttt{DeepIV}      & .204 $\pm$ .008 & .047 $\pm$ .004 & .197 $\pm$ .004 & \textbf{.039} $\pm$ \textbf{.001} \\
    \texttt{DeepGMM}  & \textit{.022} $\pm$ \textit{.003} & .032 $\pm$ .016 & .143 $\pm$ .030 & \textit{.039} $\pm$ \textit{.002} \\
    \texttt{KernelIV}  & .063 $\pm$ .000 & .024 $\pm$ .000 & \textit{.086} $\pm$ \textit{.000} & .055 $\pm$ .000\\
    \texttt{AGMM-K} & 12.3 $\pm$ .000 & 1.32 $\pm$ .000 & 1.57 $\pm$ .000 & 1.71 $\pm$ .000 \\
    \texttt{DualIV} & .202 $\pm$ .000 & .103 $\pm$ .000 & .251 $\pm$ .000 & .362 $\pm$ .000 \\
    \texttt{SieveIV} & .466 $\pm$ .002 & .452 $\pm$ .001 & .363 $\pm$ .002 & 9.07 $\pm$ .038 \\
    \texttt{MMR-IV (NN)} & \textbf{.019} $\pm$ \textbf{.003} & \textit{.004} $\pm$ \textit{.001}  & .292 $\pm$ .024  & .075 $\pm$ .008 \\
    \texttt{MMR-IV (RKHS)} & .030 $\pm$ .000 &.011 $\pm$ .000  & \textbf{.075} $\pm$ \textbf{.000}  & .057 $\pm$ .000 \\
    \bottomrule
    \end{tabular}
\end{table}

\subsection{Low-dimensional Scenarios}
\label{sec:low_dim_sce}
Following \citep{Bennett19:DeepGMM}, we employ the following data generation process:
\begin{equation*}
    \textstyle Y = f^* (X) + e + \delta,\,\, X = Z_1+e+\gamma, 
    \\
    Z:=(Z_1,Z_2), 
\end{equation*}
where $e \sim \mathcal{N}(0,1)$, $\gamma, \delta \sim \mathcal{N}(0,0.1^2)$ and $Z\sim \mathrm{Uniform}([-3,3]^2)$ is a two-dimensional IV, but $Z_1$ has an effect on X. 
The variable $e$ is the confounding variable that creates the correlation between $X$ and the residual $Y-f^*(X)$. 
We vary the true function $f^*$ between the following cases to enrich the datasets:
\begin{enumerate*}[label=(\roman*),noitemsep]
   \item \texttt{sin}: $f^*(x)= \sin(x)$.
   \item \texttt{step}: $f^*(x)= \mathrm{1}_{\{x\geq 0\}}$.
   \item \texttt{abs}: $f^*(x)= |x|$.
   \item \texttt{linear}: $f^*(x)=x$.
\end{enumerate*}
We consider both small-sample ($n=200$) and large-sample ($n=2000$) regimes. 

\subsubsection{Experimental Settings} 
For the experiments on low-dimensional data, we consider both small-sample ($n=200$) and large-sample ($n=2000$) regimes, in which $n$ points are sampled for training, validation and test sets, respectively.
In both regimes, we standardize the values of $Y$ to have zero mean and unit variance for numerical stability. 

Hyper-parameters of NNs in Algorithm \ref{alg:mmriv-nn}, including the learning rate and the regularization parameter, are chosen by 2-fold CV for fair comparisons with the baselines. 
As per the dimensions of $X$, we parametrize $f$ either as a fully connected neural network with leaky ReLU activations and 2 hidden layers, each of which has 100 cells, for non-image data, or a deep convolutional neural network (CNN) architecture for MNIST data.  We denote the fully connected neural network as FCNN(100,100) and refer readers to our code release for exact details on our CNN construction. Learning rates and regularization parameters are summarized in \sref{Table}{tab:hyper_nn}.
Besides, we use the well-tuned hyper-parameter selections of baselines provided in their packages without changes. We fix the random seed to 527 for all data generation and model initialization.

\begin{table}[t!]
    \centering    
    \caption{Hyper-parameters of neural networks used in the experiments.}
    \label{tab:hyper_nn}
    \resizebox{\textwidth}{!}{
    \begin{tabular}{c c c c} 
    \toprule
    \textbf{Scenario}  & \textbf{Model $f$}  & \textbf{Learning Rates} & $\lambda$ \\
    \midrule
    Low Dimensional  & FCNN(100,100)  & ($10^{-12},10^{-11},10^{-10},10^{-9},10^{-8},10^{-7},10^{-6}$)& ($5\times10^{-5},10^{-4},2\times 10^{-4}$) \\
    MNIST\textsubscript{Z}  & FCNN(100,100)  & ($10^{-12},10^{-11},10^{-10},10^{-9},10^{-8},10^{-7},10^{-6}$)& ($5\times10^{-5},10^{-4},2\times 10^{-4}$) \\
    MNIST\textsubscript{X}  & CNN  &  ($10^{-12},10^{-11},10^{-10},10^{-9},10^{-8},10^{-7},10^{-6}$) & ($5\times10^{-5},10^{-4},2\times 10^{-4}$)\\
    MNIST\textsubscript{XZ}  & CNN  & ($10^{-12},10^{-11},10^{-10},10^{-9},10^{-8},10^{-7},10^{-6}$)& ($5\times10^{-5},10^{-4},2\times 10^{-4}$)  \\
    Mendelian & FCNN(100,100)  & ($10^{-12},10^{-11},10^{-10},10^{-9},10^{-8},10^{-7},10^{-6}$)& ($5\times10^{-5},10^{-4},2\times 10^{-4}$)  \\
    \bottomrule
    \end{tabular}}
\end{table}

In contrast to our NN-based method, the RKHS-based method in Algorithm \ref{alg:mmriv-rkhs} has the analytic form of CV error. 
For the kernel function on instruments, we employ the sum of Gaussian kernels 
$$k(z,z')= \frac{1}{3}\sum_{i=1}^{3} \exp\left(-\frac{\|z-z'\|_2^2}{2\sigma_{ki}^2}\right)$$ 
and the Gaussian kernel for $l(x,x')= \exp(-\|x-x'\|_2^2/(2\sigma_l^2))$, where $\sigma_{k1}$ is chosen as the median interpoint distance of $\bm z = \{z_i\}_{i=1}^n$ and $\sigma_{k2}=0.1\sigma_{k1}$, $\sigma_{k3}=10\sigma_{k1}$. The motivation of such a kernel $k$ is to optimize $f$ on multiple kernels, and we leave parameter selection of $k$ to the future work.  
We combine the training and validation sets to perform leave-2-out CV to select parameters of the kernel $l$ and the regularization parameter $\lambda$. We choose the Gaussian kernel for $l$.
For the Nystr\"om approximation, we subsample 300 points from the combined set. As a small subset of the Gram matrix is used as Nystr\"om samples, which misses much information of data, we avoid outliers in test results by averaging 10 test errors with different Nystr\"om approximations in each experiment.
All methods are repeated 10 times on each dataset with random initialization.

\subsubsection{Experiment Results}
Table \ref{tab:result_low_dim_large} and Table \ref{tab:result_low_dim_small} report the results for the large-sample and small-sample regimes respectively. 
First, under the influence of confounders, \texttt{DirectNN} performs worst as it does not use instruments.
Second, \texttt{MMR-IV}s perform reasonably well in both small-sample and large-sample regimes. 
For the linear function, \texttt{2SLS} and \texttt{Poly2SLS} tend to outperform other algorithms as the linearity assumption is satisfied in this case. 
For non-linear functions, some NN based methods show competitive performance in certain cases. Notably, \texttt{GMM+NN} has unstable performance as the function $h$ in \eqref{eq:umr} is designed manually. 
\texttt{KernelIV} performs well because of simple simulation models. Its kernel parameters are selected by median heuristic which is relatively naive to the cross validation, and in later more complicated experiments, the performance of \texttt{KernelIV} deteriorates. \texttt{AGMM-K} is similar in principle to our method while its errors are high. We suspect that this is due to its hyper-parameter selection, a few folds of cross validation over a small set of hyper-parameters, is not effective enough, because we observed that the \texttt{AGMM-K}'s performance becomes better as the numbers of CV folds and of hyper-parameter candidates increase. A similar observation was also obtained on DualIV. Thus, it shows that it is desirable to have the analytical CV error, which is an advantage of our method against other RKHS baselines. Besides, the selection of the weight matrix remains an open question, and although DualIV, AGMM-K and MMRIV (RKHS) rely on the median heuristic to select the bandwidth, we believe that the selection has different effects on different methods and it is difficult to get fair selection. We will leave this problem to future work.
Moreover, the performances of the complicated methods like \texttt{DeepIV} and \texttt{DeepGMM} deteriorate more in the large-sample regimes than the small-sample regimes.
We suspect that this is because these methods rely on two NNs and thus require proper hyper-parameters (e.g. model structures) on random samples.
{\bc 
For \texttt{SieveIV}, it performs well when the true function is the sin function, because we utilize the trigonometric basis for the sieve method. 
It suggests that the choice of basis functions is important for the performance of this method.
}
In contrast, \texttt{MMR-IV (Nystr\"om)} has the advantage of adaptive hyper-parameter selection.

We also record the runtimes of all methods on the large-sample regime and report them in \sref{Fig.}{fig:runtime}. 
Compared to the NN-based methods, i.e., \texttt{AGMM}, \texttt{DeepIV}, \texttt{DeepGMM}, our \texttt{MMR-IV (NN)} is the most computationally efficient method, which is clearly a result of a simpler objective. 
Using a minimax optimization between two NNs, \texttt{AGMM} is the least efficient method. 
\texttt{DeepGMM} and \texttt{DeepIV} are more efficient than \texttt{AGMM}, but are less efficient than \texttt{MMR-IV (NN)}. 
Lastly, all four RKHS-based methods, namely, \texttt{KernelIV}, \texttt{AGMM-K}, \texttt{DualIV} and \texttt{MMR-IV (RKHS)}, have similar computational time. All four methods are observed to scale poorly on large datasets.
{\bc In contrast, \texttt{SieveIV} is the most computationally efficient method in this experiment.}

\begin{figure}[t!]
\centering
\resizebox{0.85\textwidth}{!}{
\begin{tikzpicture}
\begin{axis}[
    ybar, axis on top,
    ybar=2pt,
    ymode=log,
    bar width=8pt,
    x=2.5cm,
    width=0.5\textwidth, height=3.8cm,
    enlargelimits=0.17,
    legend style={at={(1.13,1)},
    nodes near coords = \rotatebox{90}{{\pgfmathprintnumber[fixed zerofill, precision=2]{\pgfplotspointmeta}}},
      anchor=north,legend columns=1},
    ylabel={Second(s)},
    ymin=-10, ymax=5000,
    symbolic x coords={abs,linear,sin,step},
    xtick=data,
    nodes near coords,
    nodes near coords align={vertical},
    ylabel near ticks
    ]
\addplot+ coordinates {(abs,672) (linear,584) (sin,540) (step,725)};
\addplot+ coordinates {(abs,30) (linear,29) (sin,43) (step,44)};
\addplot+ coordinates {(abs,26) (linear,21) (sin,26) (step,19)};
\addplot+ coordinates {(abs,8) (linear,7) (sin,9) (step,7)};
\addplot+ coordinates {(abs,31) (linear,34) (sin,36) (step,32)};
\addplot+ coordinates {(abs,31) (linear,31) (sin,36) (step,37)};
\addplot+ coordinates {(abs,1.8) (linear,1.8) (sin,1.9) (step,1.7)};
\legend{\texttt{AGMM},\texttt{DeepIV},\texttt{DeepGMM}, \texttt{MMR-IV\textsubscript{\texttt{NN}}}, \texttt{KernelIV},\texttt{MMR-IV\textsubscript{\texttt{RKHS}}}, \texttt{SieveIV}}
\end{axis}
\end{tikzpicture}}
\caption{{\bc Runtime in the large-sample regime ($n=2000$). 
Time for parameter selection is excluded. 
\texttt{AGMM-K (No Nystr\"om)}, \texttt{DualIV} and \texttt{MMR-IV (RKHS)} overlap due to the same runtime. Python and library versions are included in the ''README'' and ''requirement'' files in the software. We use 2.9 GHz Intel Core i5 for the experiment. \label{fig:runtime}}}
\end{figure}

\subsection{High-dimensional Structured Scenarios}
\label{sec:high_dim_sce}

In high-dimensional setting, we employ the same data generating process as in \sref{Sec.}{sec:low_dim_sce}. 
We consider only the absolute function for $f^*(x)=|x|$, but map $Z$, $X$, or both $X$ and $Z$ to MNIST images (784-dim) \citep{lecun1998gradient}. Let us denote the original outputs in \sref{Sec.}{sec:low_dim_sce} by $X^{\mathrm{low}}$, $Z^{\mathrm{low}}$ and let $\pi(u
) := \mathrm{round}(\min(\max(1.5u+5, 0), 9))$ be a transformation function mapping inputs to an integer between 0 and 9, and let $\mathrm{RI}(d)$ be a function that selects a random MNIST image from the digit class $d$. 
Then, the scenarios we consider are
\begin{enumerate*}[label=(\roman*),noitemsep]
    \item \textbf{MNIST\textsubscript{Z}}: $Z  \leftarrow \text{RI}(\pi(Z_1^{\text{low}}))$, 
    \item \textbf{MNIST\textsubscript{X}}: $X \leftarrow \text{RI}(\pi(X^{\text{low}}))$, and  
    \item \textbf{MNIST\textsubscript{XZ}}: $X \leftarrow \text{RI}(\pi(X^{\text{low}})),\,Z  \leftarrow \text{RI}(\pi(Z_1^{\text{low}}))$. 
\end{enumerate*}

\subsubsection{Experimental Settings}
We sample $n=10,000$ points for the training, validation, and test sets, and run each method 10 times. 
For \texttt{MMR-IV (Nystr\"om)}, we run it only on the training set to reduce computational workload, and use 
principal component analysis (PCA) to reduce dimensions of $X$ from 728 to 8. 
We still use the sum of Gaussian kernels for $k(z,z')$, but use the automatic relevance determination (ARD) kernel for $l(x,x')$.
We omit \texttt{KernelIV} and \texttt{DualIV} in this experiment since their kernels are not suitably designed for structured data and as we will see, so is \texttt{MMR-IV (Nystr\"om)}.
{\bc 
We also omit \texttt{SieveIV} in this setting because the method requires a huge number of basis functions in the high-dimensional setting. 
}

\subsubsection{Experiment Results}
We report the results in \sref{Table}{tab:result_high_dim}. Like \citep{Bennett19:DeepGMM}, we observe that the \texttt{DeepIV} code returns NaN when $X$ is high-dimensional, so the implementation has to be improved before being used for similar scenarios. 
Besides, we run \texttt{Ridge2SLS}, which is \texttt{Poly2SLS} with fixed linear degree, in place of \texttt{Poly2SLS} to reduce computational workload, and as a result, its performance is unsatisfactory.
\texttt{2SLS} has large errors for high-dimensional $X$ because the first-stage regression from $Z$ to $X$ is ill-posed. The 
average performance of \texttt{GMM+NN} suggests that manually designed functions for instruments is insufficient to extract useful information.
Furthermore, \texttt{MMR-IV (NN)} performs competitively across scenarios. On MNIST\textsubscript{Z}, \texttt{MMR-IV}s perform better than other methods, which implies using the sum of Gaussian kernels for the kernel $k$ is proper.  \texttt{DeepGMM} has competitive performance as well. On MNIST\textsubscript{X} and MNIST\textsubscript{XZ}, \texttt{MMR-IV (NN)} outperforms all other methods.  \texttt{AGMM-K+NN}, i.e., the KLayerTrained method by \citep{Dikkala20:Minimax}, is a variant of \texttt{AGMM-K}, which uses neural networks to model $f(X)$ and to extract features of $Z$ as inputs of the kernel $k$. To reinforce the flexibility, \texttt{AGMM-K+NN} also trains the kernel hyper-parameter, i.e., the lengthscale of the RBF kernel $k$. Compared with \texttt{MMR-IV (NN)}, \texttt{AGMM-K+NN} has a more flexible kernel $k$ but lacks good kernel hyper-parameter selection. So, it is vulnerable to local minimum and shows unstable performance. The ARD kernel with PCA
of \texttt{MMR-IV (Nystr\"om)} fails because the features from PCA are not representative enough. 
In addition, we observe that \texttt{DeepGMM} often produces results that are unreliable across all settings. 
We suspect that the hard-to-optimize objective and the complicated optimization procedure are the causes. 
Compared with \texttt{DirectNN}, most of baselines can hardly deal with high-dimensional structured instrumental regressions.


%
\begin{table}[t!]
 \centering
 \caption{The mean square error (MSE) $\pm$ one standard deviation on high-dimensional structured data.
We run each method $10$ times. \label{tab:result_high_dim}}
 \begin{tabular}{l c c c} 
 \toprule
 \multirow{2}{*}{\textbf{Algorithm}} & \multicolumn{3}{c}{\textbf{Setting}} \\
 & $\mathrm{MNIST}_\mathrm{z}$ & $\mathrm{MNIST}_\mathrm{x}$ & $\mathrm{MNIST}_\mathrm{xz}$ \\ 
 \midrule
   \texttt{DirectNN}    & .134 $\pm$ .000 & \textit{.229} $\pm$ \textit{.000}                           & \textit{.196} $\pm$ \textit{.011}     \\
 \texttt{2SLS} & .563 $\pm$ .001 & $>$1000  & $>$1000 \\
 \texttt{Ridge2SLS}   & .567 $\pm$ .000 & .431 $\pm$ .000                           & .705 $\pm$ .000     \\
 \texttt{GMM+NN}      & .121 $\pm$ .004 & .235 $\pm$ .002                           & .240 $\pm$ .016     \\
 \texttt{AGMM} & \textit{.017} $\pm$ \textit{.007}& .732 $\pm$ .107 & .529 $\pm$.163 \\
  \texttt{DeepIV}      & .114 $\pm$ .005 & n/a                               & n/a         \\
 \texttt{DeepGMM}  & .038 $\pm$ .004 & .315 $\pm$ .130 & .333 $\pm$ .168\\
 \texttt{AGMM-K+NN} & .021 $\pm$ .007 & 1.05 $\pm$ .366 & .327 $\pm$ .192 \\
 \texttt{MMR-IV (NN)} & .024 $\pm$ .006 & \textbf{.124} $\pm$ \textbf{.021} & \textbf{.130}  $\pm $ \textbf{.009}  \\
 \texttt{MMR-IV (Nys)} & \textbf{.015} $\pm$ \textbf{.002} & .442 $\pm$ .000 & .425 $\pm$ .002 \\
 \bottomrule
\end{tabular}

\end{table}

\begin{figure*}[b!]
\begin{subfigure}[t]{0.26\textwidth}
\centering
\resizebox{!}{1.2in}{
\begin{tikzpicture} 
\begin{axis}[width=1.9in, height=1.5in, xmode=log,log basis x={2},ymode=log,xmode=log,xtick=data,log basis x=2,xlabel=\textsc{Number of IVs},ylabel=MSE,ymax=2, ymin = 0.0005, tick label style = {font=\small}]
\axispath\draw
            (7.49165,-10.02171)
        |-  (38.31801,-11.32467)
        node[near start,left] {$\frac{dy}{dx} = -1.58$};
\addplot+[dotted, mark=o,every mark/.append style={solid},error bars/.cd] plot coordinates  {
(8,0.235) +- (0.000,0.000)
(16,0.235) +- (0.000,0.000)
(32,0.254) +- (0.000,0.000)};
\addplot+[dashed, mark=o,every mark/.append style={solid},error bars/.cd] plot coordinates{
(8,0.0001) +- (0.000,0.000)
(16,0.001) +- (0.000,0.000)
(32,0.017) +- (0.000,0.000)};
\addplot+[dashdotted, mark=o,every mark/.append style={solid},error bars/.cd] plot coordinates{
(8,1.067) +- (0.000,0.000)
(16,0.002) +- (0.000,0.000)
(32,142.201) +- (0.000,0.000)};
\addplot+[dashed, mark=triangle,every mark/.append style={solid},error bars/.cd] plot coordinates{
(8,0.094) +- (0.000,0.000)
(16,0.190) +- (0.000,0.000)
(32,0.234) +- (0.000,0.000)};
\addplot+[dotted, mark=triangle,every mark/.append style={solid},error bars/.cd] plot coordinates{
(8,0.089) +- (0.000,0.000)
(16,0.129) +- (0.000,0.000)
(32,8.800) +- (7.611,7.611)};
\addplot+[dashdotted, mark=triangle,every mark/.append style={solid},error bars/.cd] plot coordinates{
(8,0.211) +- (0.002,0.002)
(16,0.118) +- (0.001,0.001)
(32,0.050) +- (0.003,0.003)};
\addplot+[dashdotted, mark=square,every mark/.append style={solid},error bars/.cd] plot coordinates{
(8,0.031) +- (0.000,0.000)
(16,0.059) +- (0.004,0.004)
(32,0.291) +- (0.000,0.000)};
\addplot+[dashed, mark=square,every mark/.append style={solid},error bars/.cd] plot coordinates{
(8,0.063) +- (0.000,0.000)
(16,0.063) +- (0.000,0.000)
(32,0.045) +- (0.000,0.000)};
\addplot+[dotted, mark=square,every mark/.append style={solid},error bars/.cd] plot coordinates{
(8,1.62) +- (0.000,0.000)
(16,0.905) +- (0.000,0.000)
(32,1.20) +- (0.000,0.000)};
\addplot+[solid, mark=square,every mark/.append style={solid},error bars/.cd] plot coordinates{
(8,0.226) +- (0.000,0.000)
(16,0.217) +- (0.000,0.000)
(32,0.265) +- (0.000,0.000)};
\addplot+[solid,mark=otimes,every mark/.append style={solid},error bars/.cd] plot coordinates{
(8,0.032) +- (0.011,0.011)
(16,0.037) +- (0.016,0.016)
(32,0.039) +- (0.027,0.027)};
\addplot+[solid, mark=otimes,every mark/.append style={solid},error bars/.cd] plot coordinates{
(8,0.009) +- (0.007,0.007)
(16,0.020) +- (0.002,0.002)
(32,0.024) +- (0.002,0.002)};
\end{axis} 
\end{tikzpicture}}
\end{subfigure}
\hspace{-5pt}
\begin{subfigure}[t]{0.24\textwidth}
\centering
\resizebox{!}{1.18in}{
\begin{tikzpicture} 
\begin{axis}[width=1.9in, height=1.5in, ymajorticks=false, ymode=log,xmode=log,xtick=data,log basis x=2,legend pos=outer north east,xlabel=\textsc{Value of $c_1$},ymax=2,ymin=0.0005,tick label style = {font=\small}]
\axispath\draw
            (7.49165,-10.02171)
        |-  (38.31801,-11.32467)
        node[near start,left] {$\frac{dy}{dx} = -1.58$};
\addplot+[dotted, mark=o,every mark/.append style={solid},error bars/.cd] plot coordinates  {
(0.5,0.104) +- (0.000,0.000)
(1,0.235) +- (0.000,0.000)
(2,0.420) +- (0.000,0.000)};
\addplot+[dashed, mark=o,every mark/.append style={solid},error bars/.cd] plot coordinates{
(0.5,0.001) +- (0.000,0.000)
(1,0.001) +- (0.000,0.000)
(2,0.003) +- (0.000,0.000)};
\addplot+[dashdotted, mark=o,every mark/.append style={solid},error bars/.cd] plot coordinates{
(0.5,0.002) +- (0.000,0.000)
(1,0.002) +- (0.000,0.000)
(2,0.003) +- (0.000,0.000)};
\addplot+[dashed, mark=triangle,every mark/.append style={solid},error bars/.cd] plot coordinates{
(0.5,0.079) +- (0.000,0.000)
(1,0.190) +- (0.000,0.000)
(2,0.347) +- (0.000,0.000)};
\addplot+[dotted, mark=triangle,every mark/.append style={solid},error bars/.cd] plot coordinates{
(0.5,0.188) +- (0.000,0.000)
(1,0.129) +- (0.000,0.000)
(2,0.227) +- (0.005,0.005)};
\addplot+[dashdotted, mark=triangle,every mark/.append style={solid},error bars/.cd] plot coordinates{
(0.5,0.254) +- (0.009,0.009)
(1,0.118) +- (0.001,0.001)
(2,0.033) +- (0.002,0.002)};
\addplot+[dashdotted, mark=square,every mark/.append style={solid},error bars/.cd] plot coordinates{
(0.5,0.038) +- (0.000,0.000)
(1,0.059) +- (0.000,0.000)
(2,0.112) +- (0.000,0.000)};
\addplot+[dashed, mark=square,every mark/.append style={solid},error bars/.cd] plot coordinates{
(0.5,0.084) +- (0.000,0.000)
(1,0.063) +- (0.000,0.000)
(2,0.046) +- (0.000,0.000)};
\addplot+[dotted, mark=square,every mark/.append style={solid},error bars/.cd] plot coordinates{
(0.5,1.04) +- (0.000,0.000)
(1,0.905) +- (0.000,0.000)
(2,0.818) +- (0.000,0.000)};
\addplot+[solid, mark=square,every mark/.append style={solid},error bars/.cd] plot coordinates{
(0.5,0.091) +- (0.000,0.000)
(1,0.217) +- (0.000,0.000)
(2,0.394) +- (0.000,0.000)};
\addplot+[solid,mark=otimes,every mark/.append style={solid},error bars/.cd] plot coordinates{
(0.5,0.096) +- (0.115,0.115)
(1,0.037) +- (0.016,0.016)
(2,0.055) +- (0.016,0.016)};
\addplot+[solid, mark=otimes,every mark/.append style={solid},error bars/.cd] plot coordinates{
(0.5,0.010) +- (0.008,0.008)
(1,0.020) +- (0.002,0.002)
(2,0.023) +- (0.002,0.002)};
\end{axis} 
\end{tikzpicture}}
\end{subfigure}
\hspace{-12pt}
\begin{subfigure}[t]{0.46\textwidth}
\resizebox{!}{1.18in}{
\begin{tikzpicture} 
\begin{axis}[width=1.9in, height=1.5in, ymajorticks=false, ymode=log,xmode=log,xtick=data,log basis x=2, legend pos=outer north east, legend columns=2, legend cell align={left}, xlabel=\textsc{Value of $c_2$},ymax=2,ymin=0.0005,ylabel near ticks,legend style={nodes={scale=0.7, transform shape}},tick label style = {font=\small}]
\axispath\draw
            (7.49165,-10.02171)
        |-  (38.31801,-11.32467)
        node[near start,left] {$\frac{dy}{dx} = -1.58$};
\addplot+[dotted, mark=o,every mark/.append style={solid},error bars/.cd] plot coordinates  {(0.5,0.400) +- (0.000,0.000)
(1,0.235) +- (0.000,0.000)
(2,0.106) +- (0.000,0.000)};
\addplot+[dashed, mark=o,every mark/.append style={solid},error bars/.cd] plot coordinates{
(0.5,0.001) +- (0.000,0.000)
(1,0.001) +- (0.000,0.000)
(2,0.004) +- (0.000,0.000)};
\addplot+[dashdotted, mark=o,every mark/.append style={solid},error bars/.cd] plot coordinates{
(0.5,0.001) +- (0.000,0.000)
(1,0.002) +- (0.000,0.000)
(2,0.055) +- (0.000,0.000)};
\addplot+[dashed, mark=triangle,every mark/.append style={solid},error bars/.cd] plot coordinates{
(0.5,0.198) +- (0.001,0.001)
(1,0.190) +- (0.000,0.000)
(2,0.096) +- (0.000,0.000)};
\addplot+[dotted, mark=triangle,every mark/.append style={solid},error bars/.cd] plot coordinates{
(0.5,0.065) +- (0.000,0.000)
(1,0.129) +- (0.000,0.000)
(2,0.080) +- (0.000,0.000)};
\addplot+[dashdotted, mark=triangle,every mark/.append style={solid},error bars/.cd] plot coordinates{
(0.5,0.037) +- (0.002,0.002)
(1,0.118) +- (0.001,0.001)
(2,0.286) +- (0.010,0.010)};
\addplot+[dashdotted, mark=square,every mark/.append style={solid},error bars/.cd] plot coordinates{
(0.5,0.057) +- (0.000,0.000)
(1,0.059) +- (0.000,0.000)
(2,0.538) +- (0.000,0.000)};
\addplot+[dashed, mark=square,every mark/.append style={solid},error bars/.cd] plot coordinates{
(0.5,0.028) +- (0.000,0.000)
(1,0.063) +- (0.000,0.000)
(2,0.084) +- (0.000,0.000)};
\addplot+[dotted, mark=square,every mark/.append style={solid},error bars/.cd] plot coordinates{
(0.5,0.787) +- (0.000,0.000)
(1,0.905) +- (0.000,0.000)
(2,2.27) +- (0.000,0.000)};
\addplot+[solid,mark=square,every mark/.append style={solid},error bars/.cd] plot coordinates{
(0.5,0.384) +- (0.017,0.000)
(1,0.217) +- (0.016,0.000)
(2,0.094) +- (0.015,0.000)};
\addplot+[solid,mark=otimes,every mark/.append style={solid},error bars/.cd] plot coordinates{
(0.5,0.016) +- (0.017,0.017)
(1,0.037) +- (0.016,0.016)
(2,0.046) +- (0.015,0.015)};
\addplot+[solid, mark=otimes,every mark/.append style={solid},error bars/.cd] plot coordinates{
(0.5,0.001) +- (0.000,0.000)
(1,0.020) +- (0.002,0.002)
(2,0.035) +- (0.02,0.02)};
\legend{\texttt{DirectNN},\texttt{2SLS},\texttt{Poly2SLS},\texttt{GMM+NN},\texttt{AGMM},\texttt{DeepIV},\texttt{DeepGMM}, \texttt{KernelIV}, \texttt{AGMM-K},\texttt{DualIV}, \texttt{MMR-IV}\textsubscript{\texttt{NN}},\texttt{MMR-IV}\textsubscript{\texttt{Nys}}}
\end{axis} 
\end{tikzpicture}}
\end{subfigure}
\caption{The MSE of different methods on Mendelian randomization experiments as we vary the numbers of instruments (left),  the strength of confounders to exposures $c_1$ (middle), and the strength of confounders to instruments $c_2$ (right). 
The MSE is obtained from 10 repetitions of the experiment.}
\label{fig:mendelians}
\end{figure*}

\subsection{Mendelian Randomization}
\label{sec:mendelian-main}

We demonstrate our method in the setting of Mendelian randomization which relies on genetic variants that satisfy the IV assumptions.
The ``exposure'' $X$ and outcome $Y$ are univariate and generated from the  simulation process by \citep{hartwig2017robust}: $$Y = \beta X + c_1 e + \delta,\quad X = {\textstyle \sum_{i=1}^{m}} \alpha_iZ_i +c_2 e+\gamma,$$
where $Z\in \rr^{d'}$ with each entry $Z_i \sim B(2,p_i)$, $p_i\sim \mathrm{unif}(0.1,0.9)$, $e \sim \mathcal{N}(0,1)$, $\alpha_i \sim \mathrm{unif}( [0.8/d',1.2/d'])$, and $\gamma, \delta \sim \mathcal{N}(0,0.1^2)$.
$Z_i \sim B(2, p_i)$ mimics the frequency of an individual getting one or more genetic variants. 
The parameters $\beta,\,c_1$ control the strength of exposures and confounders to outcomes, while $c_2,\,\alpha_i$ control the strength of instruments and confounders to exposures. 
We set $\alpha_i\sim \mathrm{unif}( [0.8/d',1.2/d'])$ so that as the number of IVs increases, each IV becomes weaker while the overall strength of instruments ($\sum_i^{d'} \alpha_i$) remains constant.

In Mendelian randomization, it is known that genetic variants may act as \emph{weak} IVs \citep{Kuang20:Ivy, Hartford20:WeakIV, Burgess20:GIV}, so this experiment aims to evaluate the sensitivity of different methods to the number of instruments ($d'$) and confounder strengths $(c_1,c_2)$. 
We consider three scenarios: \begin{enumerate*}[label=(\roman*),noitemsep] \item $d' = 8,16,32$; \item $c_1 =0.5,1,2$; \item $c_2 =0.5,1,2$; \end{enumerate*} unmentioned parameters use default values: $\beta=1$, $d'=16$, $c_1=1$, $c_2=1$.
We draw $10,000$ samples for the training, validation and test sets, respectively, and train \texttt{MMR-IV (Nystr\"om)} only on the training set.
Other settings are the same as those of the low-dim scenario.

\sref{Fig.}{fig:mendelians} depicts the experimental results.
Overall, \texttt{2SLS} performs well on all settings due to the linearity assumption, except particular sensitivity to the number of (weak) instruments, which is a well-known property of \texttt{2SLS} \citep{Angrist08:Harmless}. 
Although imposing no such assumption, \texttt{MMR-IV}s perform competitively and even more stably, since the information of instruments is effectively captured by the kernel $k$ and we only need to deal with a simple objective, and also the analytical CV error plays an essential role.
\texttt{KernelIV} also achieves competitive and stable performance across all settings, but not as good as ours. \texttt{DirectNN} is always among the worst approaches on all settings as no instrument is used. 
\texttt{Poly2SLS} performs accurately on the last two experiments, while presents significant instability with the number of instruments in \sref{Fig.}{fig:mendelians}(left) because of failure of the hyper-parameter selection.
In \sref{Fig.}{fig:mendelians}(middle) and \sref{Fig.}{fig:mendelians}(right), we can observe that the performance of most approaches deteriorates as the effect of confounders becomes stronger. 
\texttt{MMR-IV (Nystr\"om)} has promising performance and shows a bit more sensitivity to $c_2$ than $c_1$, and the good performance takes
the advantage of the hyper-parameter selection compared with \texttt{AGMM-K} and \texttt{DualIV}.

\subsection{Vitamin D data}

Lastly, we apply our algorithm to the Vitamin D data \citep[Sec. 5.1]{Sjolander19:ivtools}.  
The data were collected from a 10-year study on 2571 individuals aged 40–71 and 4 variables are employed: \emph{age} (at baseline), \emph{filaggrin} (binary indicator of filaggrin mutations), \emph{VitD} (Vitamin D level at baseline) and \emph{death} (binary indicator of death during study). 
The goal is to evaluate the potential effect of VitD on death. 
We follow \citep{Sjolander19:ivtools} by controlling age in the analyses, using filaggrin as instrument, and then applying the MMR-IV (Nystr\"om) algorithm.
\citep{Sjolander19:ivtools} modeled the effect of VitD on death by a generalized linear model and found the effect is insignificant by 2SLS ($p$-{value} on the estimated coefficient is $0.13$ with the threshold of $0.05$).
More details are shown latter. 
The estimated effect is illustrated in \sref{Fig.}{fig:real_life} in the appendix. 
We observe that: (i) by using instruments, both our method and \citep{Sjolander19:ivtools} output more reasonable results compared with those without instruments: a low VitD level at a young age has a slight effect on death, but a more adverse effect at an old age \citep{meehan2014role}; (ii) Unlike \citep{Sjolander19:ivtools}, our method allows more flexible non-linearity for causal effect. 

\subsubsection{Experimental Settings on Vitamin D Data}
\label{sec:additional_exp_settings_vitd}
We normalize each variable to have a zero mean and unit variance to reduce the influence of different scales. 
We consider two cases: (i) without instrument and (ii) with instruments. 
By without instrument, we mean that the $W_V$ matrix of \texttt{MMR-IV (Nystr\"om)} becomes an identity matrix. 
Following \citep{Sjolander19:ivtools}, we assess the effect of Vitamin D (exposure) on mortality rate (outcome), control the age in the analyses, and use filaggrin as the instrument. 
We illustrate original Vitamin D, age and death in \sref{Fig.}{fig:vitd_data}.  We randomly pick (random seed is 527) 300 Nystr\"om samples and use leave-2-out cross validation to select hyper-parameters. The generalized linear models in \citep{Sjolander19:ivtools} are a linear function in the first step and a logistic regression model in the second step.

In this experiment, we use \emph{age} as a control variable by considering a structural equation model, which is similar to the model \eqref{eq:model} except the presence of the controlled (exogenous) variable $C$,
\begin{align}
    Y = f(X,C)+\varepsilon, \quad X = t(Z,C)+g(\varepsilon) + \nu
\end{align}
where $\ep[\varepsilon] = 0$ and $\ep[\nu] = 0$. We further assume that the instrument $Z$ satisfies the following three conditions:
\begin{enumerate}[label=(\roman*)]
    \item \emph{Relevance:} $Z$ has a causal influence on $X$;
    \item \emph{Exclusion restriction:} $Z$ affects $Y$ only through $X$, i.e., ${Y \ci Z} | X,\varepsilon, C$;
    \item \emph{Unconfounded instrument(s):} $Z$ is conditionally independent of the error, i.e., ${\varepsilon \ci Z}\,|\,C$.
\end{enumerate}
Unlike the conditions specified in the main text, (ii) and (iii) also include the controlled variable $C$. 
A similar model is employed in \citep{Hartford17:DIV}. 
From Assumption (iii), we can see that $\ep[\varepsilon \,|\, C, Z] = \ep[\varepsilon \,|\, C]$, and based on this, we further obtain
 \begin{align}
     \ep[(Y-f(X,C)-\ep[\varepsilon | C])h(Z,C)]=0
 \end{align}
for every measurable function $h$. 
Note that $\ep[\varepsilon \,|\, C]$ is only conditioned on $C$, remains constant on arbitrary values of $C$, and is typically non-zero. 
To adapt our method to this model, we only need to use the kernel $k$ with $(Z,C)$ as inputs and be aware that the output of the method is an estimate of $ f'(X,C) \coloneqq f(X,C)+\ep[\varepsilon | C]$  instead of just $f(X,C)$. For simplicity, we directly fit $f'$ to binary $Y$ without using such as the logistic transform $g(x)=e^x/(1+e^x)$. This is because the transform requires non-trivial modification to the analytical cross validation error and we would like to test MMR-IV (Nystr\"om) with the exact analytical error proposed in the paper.

\begin{figure}[t!]
\captionsetup[subfigure]{justification=centering}
\begin{subfigure}{.5\textwidth}
  \centering
  \includegraphics[width=\textwidth]{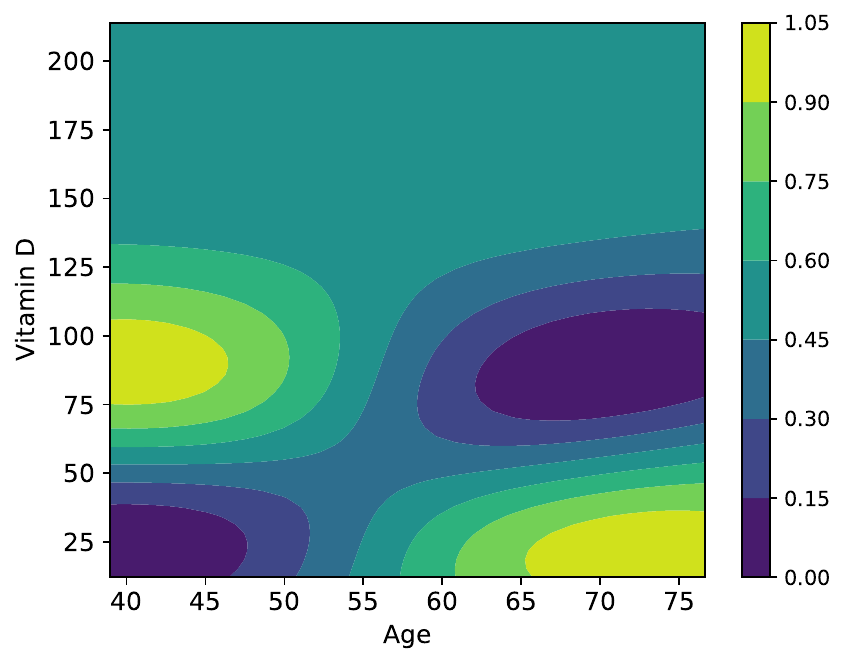}
  \caption{Kernel Ridge Regression (IV: None)}
  \label{fig:real_life_no_iv_krr} 
\end{subfigure}%
\begin{subfigure}{.5\textwidth}
  \centering
  \includegraphics[width=\textwidth]{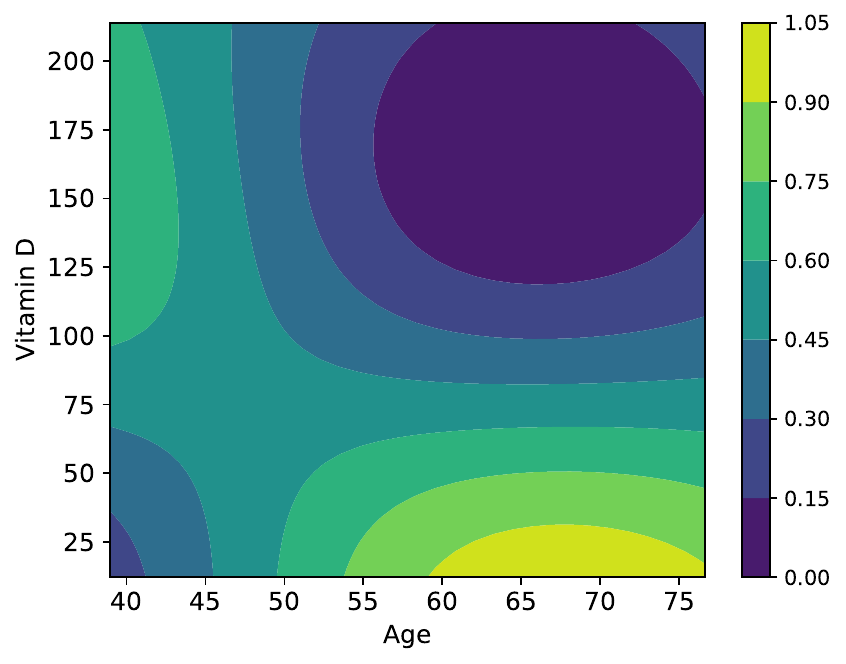}
  \caption{MMRIV (IV: Filaggrin Mutation)}
\label{fig:real_life_use_iv_mmriv}
\end{subfigure}
\vspace{0.5cm}
\\
\begin{subfigure}{.5\textwidth}
  \centering
  \includegraphics[width=\textwidth]{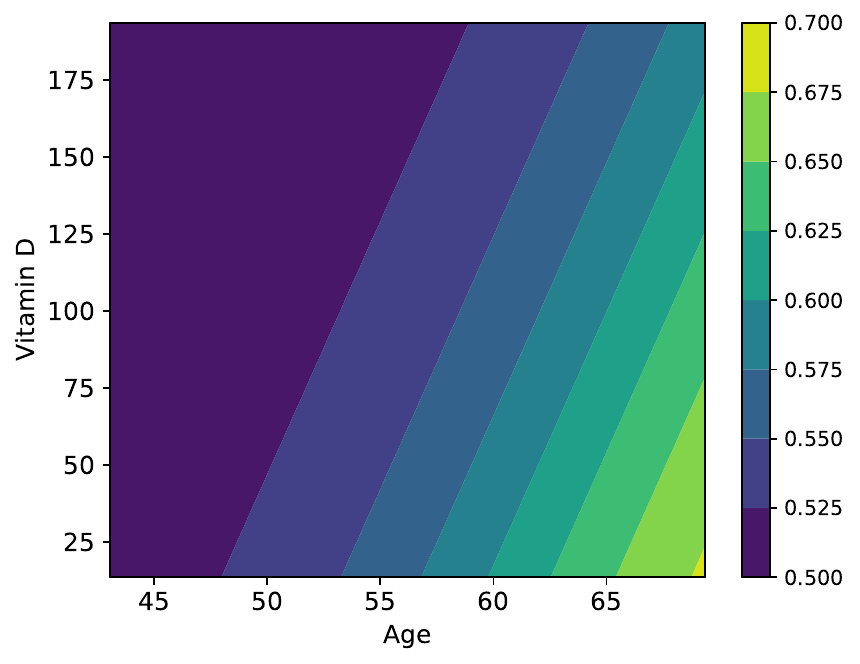}
  \caption{Generalized Linear Model (GLM, IV: None)}
  \label{fig:real_life_no_iv_glm} 
\end{subfigure}%
\begin{subfigure}{.5\textwidth}
  \centering
  \includegraphics[width=\textwidth]{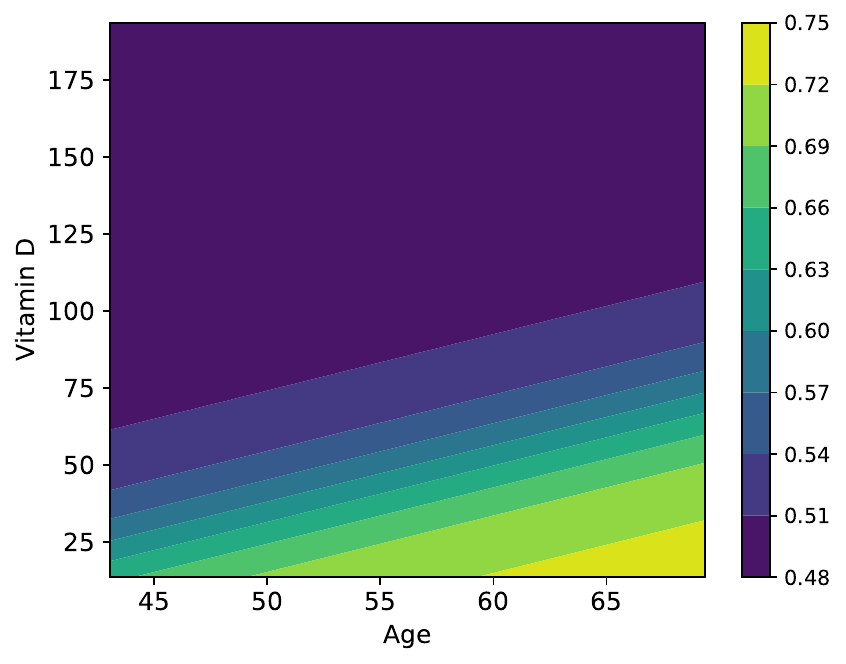}
  \caption{2SLS + GLM (IV: Filaggrin Mutation)}
\label{fig:real_life_use_iv_2sls}
\end{subfigure}
\caption{Estimated effect of vitamin D level on mortality rate, controlled by age. 
All plots depict normalized contours of  $f'(X,C)$ defined in Sec. \ref{sec:additional_exp_settings_vitd} where blue represents low mortality rate and yellow the opposite. We can divide each plot roughly into left (young age) and right (old age) parts. While the right parts reflect similar information (i.e., lower vitamin D level at an old age leads to higher mortality rate), the left parts are different. In (\subref{fig:real_life_no_iv_krr}), a high level of vitamin D at a young age can result in a high mortality rate, which is counter-intuitive.
A plausible explanation is that it is caused by some unobserved confounders between vitamin D level and mortality rate.
In (\subref{fig:real_life_use_iv_mmriv}), on the other hand, this spurious effect disappears when the filaggrin mutation is used as instrument, i.e., a low vitamin D level at a young age has only a slight effect on death, but a more adverse effect at an old age \citep{meehan2014role}. 
This comparison demonstrates the benefit of an instrument variable. (\subref{fig:real_life_no_iv_glm}) and (\subref{fig:real_life_use_iv_2sls}) correspond to the results obtained by using \citep{Sjolander19:ivtools}' generalized linear model (GLM), from which we can draw similar conclusions. It is noteworthy that MMR-IV allows more flexible non-linearity for causal effect.}
\label{fig:real_life}
\end{figure}

\begin{figure}[t!]
\centering
 \includegraphics[width=0.9\textwidth]{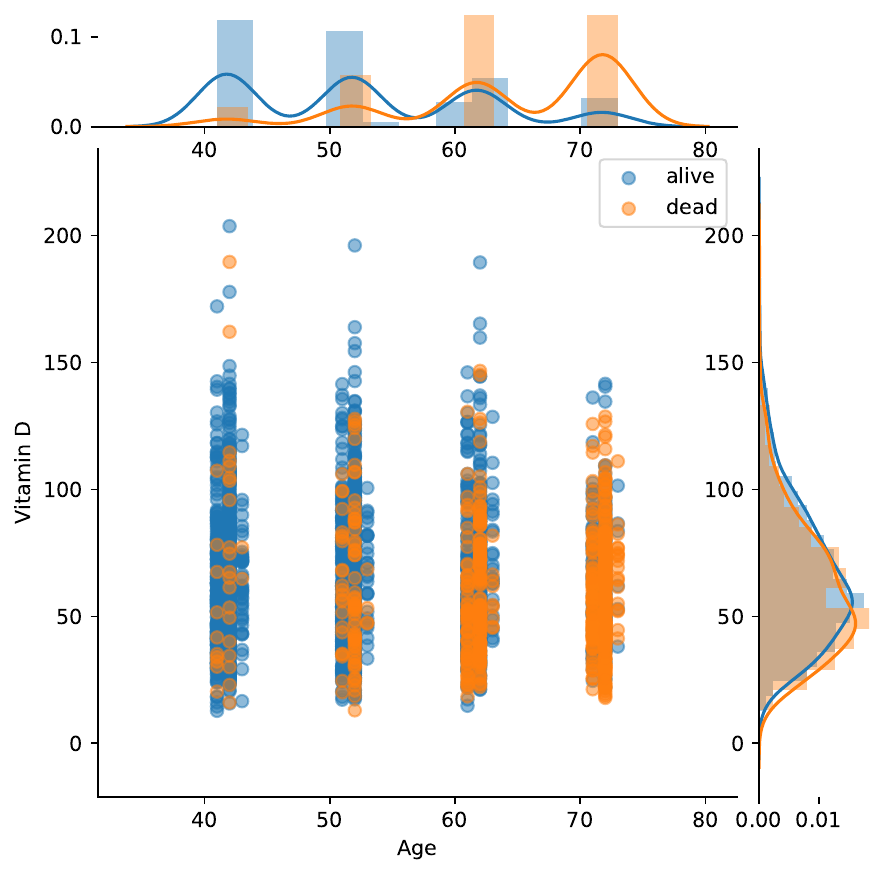}
\caption{Distribution of Vitamin D data. Data points are plotted in the middle, the solid curve and histogram on the right describe the kernel density estimation and histogram of Vitamin D, and those on the top are for Age.}
\label{fig:vitd_data}
\end{figure}


\section{Conclusion and Discussion}
\label{sec:conclusion}

Learning causal relations when hidden confounders are present is a cornerstone of reliable decision making.
 IV regression is a standard tool to tackle this task, but currently faces challenges in nonlinear settings.
This work presents a simple kernel-based framework that overcomes some of these challenges.
We employ RKHS theory in the reformulation of conditional moment restriction (CMR) as a maximum moment restriction (MMR) based on which we can approach the problem from the empirical risk minimization (ERM) perspective.
As we demonstrate, this framework not only facilitates theoretical analysis, but also results in easy-to-use algorithms that perform well in practice compared to existing methods.
The paper also shows a way of combining the elegant theoretical approach of kernel methods with practical merits of deep neural networks.
Despite these advantages, the optimal choice of the kernel $k$ in the MMR objective remains an open question which we aim to address in future work.

The efficiency issue, i.e., a volume of the error by our estimator, is one of the major unsolved problems. 
Importantly, since the efficiency heavily depends on the choice of kernel function $k$, we have to discuss an appropriate method of the kernel selection. 
This dependence has been clarified by the analysis of errors in the minimax problem by \citep{Dikkala20:Minimax}, which shows that the error is upper bounded by a metric entropy of RKHS associated with the kernel $k$.
Hence, the selection of optimal kernel functions is one of the important future directions. 

In the classical IV problem with parametric models, the optimal choice of IV is based on variance of an asymptotic distribution of estimators \citep{Newey93:CMR}. 
However, our setup corresponds to a nonparametric regression problem, which requires a more complicated  approach. 
A recent work \citep{zhang2021instrument} tackles this problem, but some arbitrariness remains.
Another challenge is a recent work of variational method of moments (VMM) \citep{bennett2020variational}.  
This work provides the optimal weighted objective and the closed-form expression of the estimator, which can achieve the semi-parametric efficiency bound for solving the conditional moment restriction. 
This approach would be highly relevant to the discussion of efficiency in our setting.

\section*{Acknowledgments}
We thank Yuchen Zhu, Vasilis Syrgkanis, and Heiner Kremer for a fruitful discussion. We are also indebted to anonymous reviewers for their constructive feedback on the initial draft of this manuscript.

\section*{Funding Information}
Rui Zhang was sponsored by PhD scholarship of Data61 and by the Empirical Inference department of Max Planck Institute.
Masaaki Imaizumi was supported by JSPS KAKENHI Grant Number 18K18114 and JST Presto Grant Number JPMJPR1852.
Bernhard Sch\"olkopf is a member of the Machine Learning Cluster of Excellence, EXC number 2064/1 – Project number 390727645.
This work was partly supported by the German Federal Ministry of Education and Research (BMBF): T\"ubingen AI Center, FKZ: 01IS18039B.

\section*{Author Contributions}
All authors have accepted responsibility for the entire content of this manuscript and approved its submission.
The study and manuscript have received significant contributions from all authors whose contributions are reflected in the author list.

\section*{Conflict of Interest}
The authors have a conflict of interest with the Australian National University, Data61 of Commonwealth Scientific and Industrial Research Organisation (CSIRO), Australia, Max Planck Institute (MPI), CISPA--Helmholtz Center for Information Security, and the University of Tokyo.

\section*{Ethical Approval}
The conducted research is not related to either human or animals use.

\section*{Data Availability Statement}
The datasets generated during and analyzed during the current study are available at \url{https://github.com/RuiZhang2016/MMRIV}, where the implementation of the approaches is also provided.


\appendix

\section{Detailed Proofs}

This section contains detailed proofs of the results that are missing in the main paper. Most of the proofs on consistency and asymptotic normality take advantages of the useful resource by \citep{NEWEY19942111}. Readers are referred to it for more detailed discussions on e.g. assumptions. Note that our proofs are based on the normed $\inF$ space.

\subsection{Proof of \texorpdfstring{\sref{Lemma}{lem:analytic-form}}{}}

\begin{proof} 
Since $\mathcal{H}_k$ is the RKHS, we can rewrite \eqref{eq:mmr} as
\begin{align}\label{eq:rkhs-norm-proof}
    R_k(f) &= \sup_{h\in\mathcal{H}_k,\|h\|\leq 1}\; \left(\ep[(Y-f(X))\langle h, k(Z,\cdot)\rangle_{\mathcal{H}}]\right)^2 \nonumber \\
    &= \sup_{h\in\mathcal{H}_k, \|h\|\leq 1}\; \left(\langle h, \ep[(Y-f(X))k(Z,\cdot)]\rangle_{\mathcal{H}}\right)^2  \nonumber \\
    &= \left\|\ep[(Y-f(X))k(Z,\cdot)] \right\|_{\mathcal{H}_k}^2  ,
\end{align}
where we used the reproducing property of $\mathcal{H}_k$ in the first equality and the fact that $\mathcal{H}_k$ is a vector space in the last equality.
By assumption, $\ep[(Y-f(X))k(Z,\cdot)]$ is Bochner integrable \citep[Def. A.5.20]{steinwart2008support}.
Hence, we can write \eqref{eq:rkhs-norm-proof} as
\begin{eqnarray*}
    \|\ep[(Y-f(X))k(Z,\cdot)]\|_{\mathcal{H}_k}^2 
    &=& \left\langle \ep[(Y-f(X))k(Z,\cdot)],\ep[(Y-f(X))k(Z,\cdot)]\right\rangle_{\mathcal{H}_k} \\
    &=& \ep\left[\langle (Y-f(X))k(Z,\cdot),\ep[(Y-f(X))k(Z,\cdot)]\rangle_{\mathcal{H}_k}\right] \\
    &=& \ep\left[\langle (Y-f(X))k(Z,\cdot),(Y'-f(X'))k(Z',\cdot)\rangle_{\mathcal{H}_k}\right] \\
    &=& \ep\left[(Y-f(X))(Y'-f(X'))k(Z,Z')\right],
\end{eqnarray*}
as required.
\end{proof}

\subsection{Proof of \texorpdfstring{\sref{Proposition}{thm:orig_cond_m_res}}{}}
\label{sec:proof-sufficiency}

\begin{proof}
    First, the law of iterated expectation implies that 
    \begin{equation}
        \ep[(Y-f(X))k(Z,\cdot)] = \ep_Z[\ep_{\mathit{XY}}[(Y-f(X))k(Z,\cdot)|Z]] 
        = \ep_Z[\ep_{\mathit{XY}}[Y-f(X)|Z]k(Z,\cdot)].
    \end{equation}
    By Lemma \ref{lem:analytic-form}, we know that $R_k(f) = \|\ep[(Y-f(X))k(Z,\cdot)]\|_{\mathcal{H}_k}^2$.
    As a result, $R_k(f) = 0$ if $\ep[Y - f(X)\,|\,z] = 0$ for $P_Z$-almost all $z$.
    To show the converse, we assume that $R_k(f) = 0$ and rewrite it as 
    \begin{equation*}
        R_k(f) = \iint_{\inz} g(z)k(z,z')g(z') \dd z \dd z' = 0,
    \end{equation*}
    where we define $g(z) := \ep_{\mathit{XY}}[Y-f(X)\,|\,z]p(z)$.
    Since $k$ is ISPD by assumption, this implies that $g$ is a zero function with respect to $P_Z$, i.e., $\ep[Y-f(X)\,|\,z] = 0$ for $P_Z$-almost all $z$.
\end{proof}

\subsection{Proof of Proposition \ref{thm:convex}}
\label{pf:convex}

\begin{proof}
Given $\alpha \in (0,1)$ and any functions $f,g: \mathcal X \rightarrow \mathbb R$, we will show that
\begin{equation*}
     R_k(\alpha f+(1-\alpha)g) - \alpha R_k(f) - (1-\alpha) R_k (g) < 0 .
\end{equation*}
By Lemma \ref{lem:analytic-form}, we know that $R_k(f) = \|\ep[(Y-f(X))k(Z,\cdot)]\|_{\mathcal{H}_k}^2$.
Hence, we can rewrite the above function as
\begin{align}
     &R_k(\alpha f+(1-\alpha)g) - \alpha R_k(f) - (1-\alpha) R_k (g) \\
    &= \| \ep[(Y-\alpha f(X) -(1-\alpha)g(X))k(Z,\cdot)]\|_{\mathcal{H}_k}^2 
 - \alpha\| \ep[(Y- f(X))k(Z,\cdot)]\|_{\mathcal{H}_k}^2 
    \\
    &\qquad -(1-\alpha)\| \ep[(Y- g(X))k(Z,\cdot)]\|_{\mathcal{H}_k}^2 
    \\
    &\overset{\text{(a)}}{=}  \alpha(\alpha-1) \| \ep[(Y-f(X))k(Z,\cdot)]\|_{\mathcal{H}_k}^2 + \alpha(\alpha-1)\|\ep[(Y-g(X))k(Z,\cdot)]\|_{\mathcal{H}_k}^2 
    \\
    &\qquad -2 \alpha(\alpha-1) \langle \ep[(Y-f(X))k(Z,\cdot)],\ep[(Y-g(X))k(Z,\cdot)] \rangle_{\mathcal{H}_k} 
    \\
    & \overset{\text{(b)}}{=}\alpha(\alpha-1) \underbrace{\| \ep[(f(X)-g(X))k(Z,\cdot)]\|_{\mathcal{H}_k}^2}_{>0} <0
\end{align}
The equality (a) is obtained by considering $Y = \alpha Y + (1-\alpha)Y$ in $\| \ep[(Y-\alpha f(X) -(1-\alpha)g(X))k(Z,\cdot)]\|_{\mathcal{H}_k}^2$ on the left hand side of (a). We note that the right hand side of (a) is quadratic in $\|\ep[(Y-f(X))k(Z,\cdot)]\|_{\mathcal{H}_k}$ and $\|\ep[(Y-g(X))k(Z,\cdot)]\|_{\mathcal{H}_k}$, and can be further expressed as a square binomial as the right hand side of $(b)$. Therefore, the convexity follows from the fact that $k$ is the ISPD kernel, $ \| \ep[(f(X)-g(X))|Z]\|_2 \neq 0$ and $\alpha(\alpha-1) < 0$.
\end{proof}

\subsection{Uniform Convergence of Risk Functionals}
\label{pf:unif_cons_emp_rv}

The results presented in this section are used to prove the consistency of $\hat{f}_V$ and $\hat{f}_U$.

\begin{lemma}[Uniform consistency of $\widehat{R}_V(f)$] \label{thm:unif_cons_emp_rv}
Assume that $\ep[| Y |^2]<\infty$,  $\inF$ is compact, $\ep[\sup_{f \in \mathcal F} | f(X) |^2]<\infty$, and \sref{Assumption}{asmp:k_idp} holds. Then, the risk $R_k(f)$ is continuous about $f\in \inF$ and $\sup_{f \in \inF} |\widehat{R}_V(f)-R_k(f)| \overset{\mathrm p}{\to}0$.
\end{lemma}

\begin{proof}
First, let $u := (x,y,z)$, $u' := (x',y',z')$, and $h_f(u,u'):=(y-f(x))(y'-f(x'))k(z,z')$ for some $(x,y,z),(x',y',z')\in\mathcal{X}\times\mathcal{Y}\times\mathcal{Z}$.
To prove that $\widehat{R}_V$ converges uniformly to $R_k$, we need to show that (i) $h_f(u,u')$ is continuous at each $f$ with probability one; (ii) $\ep_{U,U'}\left[\sup_{f \in \inF} |h_f(U,U')|\right]<\infty$, and $\ep_{U,U}\left[\sup_{f \in \inF} |h_f(U,U)|\right]<\infty$ \citep[Lemma 8.5]{NEWEY19942111}.
To this end, it is easy to see that
\begin{eqnarray}
    |h_f(u,u')|
    &=& |(y-f(x))(y'-f(x'))k(z,z')| 
    \\
    & \leq & |y-f(x)||y'-f(x')||k(z,z')|
    \\
    & \leq & |y-f(x)||y'-f(x')| \sqrt{k(z,z)k(z',z')}
    \\
    &\leq & (|y|+|f(x)|)(|y'|+|f(x')|)\sqrt{k(z,z)k(z',z')}.
\end{eqnarray}
The third inequality follows from the Cauchy-Schwarz inequality. 
Since $\inF$ is compact, every $f \in \inF$ has $f(x)$ bounded for $\|x\|<\infty$.
In term of $k(\cdot,\cdot)$ is bounded as per \sref{Assumption}{asmp:k_idp}, we have $h_f(u,u')<\infty$ and thus $h_f(u,u')$ continuous at each $f$ with probability one. Furthermore, we obtain the following inequalities
\begin{align}
    \ep_{U,U'}\left[\sup_{f \in \inF} |h_f(U,U')|\right]&\leq \ep\left [\sup_{f\in \mathcal{F}}(|Y|+|f(X)|)(|Y'|+|f(X')|)\sqrt{k(Z,Z)k(Z',Z')} \right]
    \\ 
    & \leq \ep\left [\sup_{f\in \mathcal{F}}(|Y|+|f(X)|)\sup_{f\in \mathcal{F}}(|Y'|+|f(X')|)\right] \sup_z k(z,z) 
    \\
    & \quad = \ep\left [\sup_{f\in \mathcal{F}}(|Y|+|f(X)|)\right]^2 \sup_z k(z,z)
    \\
    & \quad = \ep\left [|Y|+\sup_{f\in \mathcal{F}}|f(X)|\right]^2 \sup_z k(z,z) <\infty 
    \\
    \ep_{U,U}\left[\sup_{f \in \inF} |h_f(U,U)|\right]&\leq \ep_u \left [\sup_{f\in \mathcal{F}}(|Y|+|f(X)|)^2\right] \sup_z k(z,z) 
    \\
    &\quad = \ep \left [(|Y|+\sup_{f\in \mathcal{F}}|f(X)|)^2\right] \sup_z k(z,z)
    \\
    &\quad \leq 2\left(\ep \left [|Y|^2\right] + \ep \left [\sup_{f\in \mathcal{F}}|f(X)|^2\right]\right) \sup_z k(z,z) <\infty
\end{align}
Hence, our assertion follows from \citep[Lemma 8.5]{NEWEY19942111}.
\end{proof}

\begin{lemma}[Uniform consistency of $\widehat{R}_U(f)$] \label{thm:unif_cons_emp_ru}
Assume that $\ep[| Y |]<\infty$, $\inF$ is compact, $\ep[|f(X)|]<\infty$ and \sref{Assumption}{asmp:k_idp} holds. Then, $R_k(f)$ is continuous about $f$ and $\sup_{f \in \inF} |\widehat{R}_U(f)-R_k(f)| \overset{\mathrm p}{\to}0$.
\end{lemma}
\begin{proof}
First, let $u := (x,y,z)$, $u' := (x',y',z')$, and $h_f(u,u'):=(y-f(x))(y'-f(x'))k(z,z')$ for some $(x,y,z),(x',y',z')\in\mathcal{X}\times\mathcal{Y}\times\mathcal{Z}$.
To prove the uniform consistency of $\widehat{R}_U$, we need to show that (i) $h_f(u,u')$ is continuous at each $f$ with probability one; (ii) there is $d(u,u')$ with $|h_f(u,u')| \leq d(u,u')$ for all $f\in\inF$ and $\ep_{U,U'}[d(U,U')] < \infty$ \citep[Lemma 2.4]{NEWEY19942111}; (iii) $(u_i,u_j)_{i\neq j}^{n,n}$ has strict stationarity and ergodicity in the sense of \citep[Footnote 18 in P.2129]{NEWEY19942111}.
To this end, it is easy to see that
\begin{eqnarray*}
    |h_f(u,u')|
    &=& |(y-f(x))(y'-f(x'))k(z,z')| 
    \\
    & \leq & |y-f(x)||y'-f(x')||k(z,z')|
    \\
    & \leq & |y-f(x)||y'-f(x')| \sqrt{k(z,z)k(z',z')}
    \\
    & \leq & (|y|+|f(x)|)(|y'|+|f(x')|)\sqrt{k(z,z)k(z',z')}\equiv d(u,u').
\end{eqnarray*} 
The third inequality follows from the Cauchy-Schwarz inequality. 
Since $\inF$ is compact, every $f \in \inF$ has $f(x)$ bounded for $\|x\|<\infty$.
In terms of $k(\cdot,\cdot)$ is bounded as per \sref{Assumption}{asmp:k_idp}, we have $h_f(u,u')<\infty$ and thus it proves that (i)
$h_f(u,u')$ is continuous at each $f$ with probability one. To prove (ii)  $\ep_{U,U'}[d(U,U')] < \infty$, we show that 
\begin{align}
    \ep_{U,U'}[d(U,U')] &\leq \ep \left [|Y|+|f(X)|\right]^2 \sup_z k(z,z)<\infty
\end{align}

Furthermore, we show that $(u_i,u_j)_{i\neq j}^{n,n}$ has strict stationarity and ergodicity. 
Strict stationarity means that the distribution of a set of data $(u_i,u_{j\neq i})_{i=1,j=1}^{i+m,j+m'}$ does not depend on the starting indices $i,j$ for any $m$ and $m'$, which is easy to check. 
Ergodicity means that $\widehat{R}_U(f) \pto R_k(f)$ for all $f \in \inF$ and $\ep[|h_f(U,U')|]<\infty$. 
We have already shown that $h_f(u,u')$ is bounded and $\ep[|h_f(U,U')|]<\infty$, so $\widehat{R}_
U(f) \pto R_k(f)$ follows by \citep[P.25]{10.2307/2282952}.
Therefore, ergodicity holds, and we have shown all conditions required by extended results of \citep[Lemma 2.4]{NEWEY19942111}. Then, it follows that $\sup_{f \in \inF} |\widehat{R}_U(f)-R_k(f)| \overset{\mathrm p}{\to}0$ and $R_k(f)$ is continuous.
\end{proof}

\subsection{Indefiniteness of Weight Matrix \texorpdfstring{$W_U$}{}}
\label{pf:indefinite_wu}

\begin{theorem}\label{thm:indefinite_wu} 
If \sref{Assumption}{asmp:k_idp} holds, $W_U$ is indefinite.
\end{theorem}

\begin{proof}
By definition, we have
$$W_U = \frac{1}{n(n-1)}[K(\bm z, \bm z) -\mathrm{diag}(k(z_1,z_1),\ldots,k(z_n,z_n))] = \frac{1}{n(n-1)}K_U,$$ 
where $\mathrm{diag}(a_1,\ldots,a_n)$ denotes an $n\times n$ diagonal matrix whose diagonal elements are $a_1,\ldots,a_n$.
We can see that the diagonal elements of $K_U$ are zeros and therefore $\mathrm{trace}(W_U)=0$. 
Let us denote the eigenvalues of $W_U$ by $\{\lambda_i\}_{i=1}^n$. 
Since $\sum_{i=1}^n \lambda _i = \mathrm{trace}(W_U)$, we conclude that there exist both positive and negative eigenvalues (all eigenvalues being zeros yields trivial $W_U = \bm 0$).
As a result, $W_U$ is indefinite.
\end{proof}

\subsection{Consistency of \texorpdfstring{$\hat{f}_V$}{} with Convex \texorpdfstring{$\Omega(f)$}{}}
\label{pf:cons_emp_fv_convex}

\begin{theorem}[Consistency of $\hat{f}_V$ with convex $\Omega(f)$] \label{thm:cons_emp_fv_convex}
Assume that $\inF$ is a convex set, $f^*$ is an interior point of $\inF$, $\Omega(f)$ is convex about $f$, $\lambda \pto 0$ and \sref{Assumptions}{asmp:k_idp}, \ref{asmp:opt_ident} holds. Then, $\hat{f}_V$ exists with probability approaching one and $\hat{f}_V \pto f^*$.
\end{theorem}

\begin{proof}
Given $\Omega(f)$ is convex about $f$, we prove the consistency based on \citep[Theorem 2.7]{NEWEY19942111} which requires (i) $R_k(f)$ is uniquely maximized at $f^*$; (ii) $\widehat{R}_V(f)+\lambda\Omega(f)$ is convex; (iii) $\widehat{R}_V(f)+\lambda\Omega(f) \pto R_k(f)$ for all $f \in \inF$.

Recall that $\widehat{R}_V(f) = \| \frac{1}{n}\sum_{i=1}^n(y_i-f(x_i))k(z_i,\cdot)\|_{\mathcal H_k}^2$, and 
by the law of large number, we have that $\frac{1}{n}\sum_{i=1}^n(y_i-f(x_i))k(z_i,\cdot) \pto \ep[(Y-f(X))k(Z,\cdot)]$. Then 
$\widehat{R}_V(f) \pto R_k(f)$ follows from the Continuous Mapping Theorem \citep{mann1943} based on the fact that the function $g(\cdot)=\|\cdot\|_{\mathcal H_k}^2$ is continuous. 
As $\lambda \pto 0$, we obtain (iii) $\widehat{R}_V(f) + \lambda \Omega(f) \pto R_k(f)$ by Slutsky's theorem \citep[Lemma 2.8]{Vaart00:Asymptotic}.
Besides, it is easy to see that $\widehat{R}_V(f)$ is convex because the weight matrix $W_V$ is positive definite, and (ii) $\widehat{R}_V(f)+\lambda \Omega(f)$ is convex due to convex $\Omega(f)$. 
Further,
     the condition (i) directly follows from Proposition \ref{thm:convex}, and given that $f^*$ is an interior point of the convex set $\inF$, our assertion follows from \citep[Theorem 2.7]{NEWEY19942111}. 
\end{proof}

\subsection{Proof of \texorpdfstring{\sref{Proposition}{thm:cons_emp_fv_nonconvex}}{} }
\label{pf:cons_emp_fv_nonconvex}
\begin{proof}
From the conditions of \sref{Lemma}{thm:unif_cons_emp_rv}, we know that $\inF$ is compact, $R_k(f)$ is continuous about $f$ and $\sup_{f \in \inF} |\widehat{R}_V(f)-R_k(f)| \overset{\mathrm p}{\to}0$. 
As \sref{Assumptions}{asmp:k_idp}, \ref{asmp:opt_ident} hold, $R_k(f)$ is uniquely minimized at $f^*$. Based on the conditions that $\Omega(f)$ is bounded and $\lambda \pto 0$, we obtain by Slutsky's theorem that
\begin{align}
    \sup_{f \in \inF} \left|\widehat{R}_V(f)+\lambda \Omega(f)-R_k(f)\right| \leq \sup_{f \in \inF} \left|\widehat{R}_V(f)-R_k(f)\right|+\lambda \sup_{f \in \inF} \Omega(f) \pto 0.
\end{align}
Consequently, we assert the conclusion by \citep[Theorem 2.1]{NEWEY19942111}.
\end{proof}

\subsection{Consistency of \texorpdfstring{$\hat{f}_U$}{}}

\begin{theorem}[Consistency of $\hat{f}_U$] \label{thm:cons_emp_fu}
Assume that conditions of \sref{Lemma}{thm:unif_cons_emp_ru} and \sref{Assumption}{asmp:opt_ident} hold, $\Omega(f)$ is a bounded function and $\lambda \pto 0$. 
Then  $\hat{f}_U \pto f^*$.
\end{theorem}
\begin{proof}
By the conditions of \sref{Lemma}{thm:unif_cons_emp_ru}, we know that $\inF$ is compact, $R_k(f)$ is continuous about $f$ and $\sup_{f \in \inF} |\widehat{R}_U(f)-R_k(f)| \overset{\mathrm p}{\to}0$. 
As \sref{Assumptions}{asmp:k_idp}, \ref{asmp:opt_ident} hold, $R_k(f)$ is uniquely minimized at $f^*$. Based on the conditions that $\Omega(f)$ is bounded and $\lambda \pto 0$, we obtain by Slutsky's theorem that
\begin{align}
    \sup_{f \in \inF} \left|\widehat{R}_U(f)+\lambda \Omega(f)-R_k(f) \right| \leq \sup_{f \in \inF} \left|\widehat{R}_U(f)-R_k(f) \right|+\lambda \sup_{f \in \inF} \Omega(f) \pto 0.
\end{align} Consequently, we assert the conclusion by \citep[Theorem 2.1]{NEWEY19942111}.
\end{proof}

\subsection{Asymptotic Normality of \texorpdfstring{$\hat{\theta}_U$}{}}
In this section, we consider the regularized U-statistic risk $\widehat{R}_{U,\lambda}(f_{\theta})$.
For $u_i := (x_i,y_i,z_i)$ and $u_j := (x_j,y_j,z_j)$,
we express it in a compact form
\begin{eqnarray}
     \widehat{R}_{U,\lambda}(f_{\theta}) &\coloneqq& \underbrace{\dfrac{1}{n(n-1)}\sum_{i=1}^{n}\sum_{j\neq i}^n h_{\theta}(u_i,u_j)}_{\widehat{R}_{U}(f_\theta)} +\lambda \Omega(\theta)
    \\
     h_{\theta}(u_i,u_j) &\coloneqq& (y_i-f_{\theta}(x_i))k(z_i,z_j)(y_j-f_{\theta}(x_j)).
\end{eqnarray}
We will assume that $f_\theta$ and $\Omega(\theta)$ are twice continuously differentiable about $\theta$. 
The first-order derivative $\nabla_{\theta} \widehat{R}_U(f_{\theta})$ can also be written as
\begin{eqnarray}
    \nabla_{\theta} \widehat{R}_{U,\lambda}(f_{\theta}) &=& \underbrace{\dfrac{1}{n(n-1)}\sum_{i=1}^{n}\sum_{j\neq i}^n \nabla_{\theta} h_{\theta}(u_i,u_j)}_{\nabla_{\theta} \widehat{R}_U(f_{\theta})}+\lambda \nabla_{\theta} \Omega(\theta)
    \\
    \nabla_{\theta} h_{\theta}(u_i,u_j) &=& -\left[(y_i-f_{\theta}(x_i))\nabla_\theta f_{\theta}(x_j)+(y_j-f_{\theta}(x_j))\nabla_\theta f_{\theta}(x_i)\right]k(z_i,z_j).
\end{eqnarray}

\subsection{Asymptotic normality of \texorpdfstring{$\nabla_{\theta} \widehat{R}_{U,\lambda}(f_{\theta^*})$}{}}

We first show the asymptotic normality of $\nabla_{\theta} \widehat{R}_{U,\lambda}(f_{\theta^*})$. 
We assume that there exists $z \in \inz$ such that $\ep_{X}[\nabla_{\theta}f_{\theta^*}(X)\,|\,z]p(z)\neq 0$ or $\ep_{\mathit{XY}}[Y-f_{\theta^*}(X)\,|\,z]p(z) \neq 0$. 
Both terms being equal to zeros for all $z \in \inz$ leads to a singular $\nabla_{\theta}^2\widehat{R}_U(f_{\theta^*})$ and the asymptotic distribution therefore becomes much more complicated to analyze.

\begin{lemma}\label{thm:asym_norm_1st_deri_u}
Suppose that $f_\theta$ and $\Omega(\theta)$ are first continuously differentiable about $\theta$, $\ep[ \|\nabla_{\theta}h_{\theta^*}(U,U')\|_2^2] <\infty$, there exists $z \in \inz$ such that
$\ep_{X}[\nabla_{\theta}f_{\theta^*}(X)\,|\,z]p(z)\neq 0$ or $\ep_{XY}[Y-f_{\theta^*}(X)\,|\,z]p(z) \neq 0$, and $\sqrt{n}\lambda \pto 0$. Then, 
$$\sqrt{n}\nabla_{\theta} \widehat{R}_{U,\lambda}(f_{\theta^*}) \pto N(\bm 0, 4\mathrm{diag}(\ep_U[\ep^{2}_{U'}[\nabla_{\theta}h_{\theta^*}(U,U')]])).$$
\end{lemma}
\begin{proof}
The proof follows from \citep[Section 5.5.1 and Section 5.5.2]{Serfling80:Approximation} and we need to show that (i) $\nabla_{\theta} \widehat{R}_{U}(f_{\theta^*})\pto \bm 0$ and (ii) whether $\mathrm{Var}_U[\ep_{U'}[\nabla_{\theta}h_{\theta^*}(U,U')]] >0$ or not.
(i) can be obtained by the law of large numbers because $\nabla_{\theta} \widehat{R}_{U}(f_{\theta^*})$ is a sample average of $\nabla_{\theta} R_k(f_{\theta^*})=\bm 0$. 

To prove (ii), we first note that $\mathrm{Var}_U[\ep_{U'}[\nabla_{\theta}h_{\theta^*}(U,U')]] = \ep_U[\ep^{2}_{U'}[\nabla_{\theta}h_{\theta^*}(U,U')]]-\underbrace{\ep_{UU'}^2[\nabla_{\theta}h_{\theta^*}(U,U')]}_{=\bm 0}
    \geq \bm 0$,
where equality holds if for any $U$, there is $\ep_{U'}[\nabla_{\theta}  h_{\theta^*}(U,U')]=0$, i.e.,  
\begin{align}
    &\ep_{U'}[\nabla_{\theta} h_{\theta^*}(U,U')]\\
    &=- \ep_{X'Z'}[\nabla_{\theta}f_{\theta^*}(X')k(Z',Z)] (Y-f_{\theta^*}(X)) 
    -  \ep_{X'Y'Z'}[(Y'-f_{\theta^*}(X'))k(Z',Z)]\nabla_{\theta}f_{\theta^*}(X) \\
    &=\bm 0.
\end{align}
As the above equation holds for any $Y$, the coefficient of $Y$ must be $0$:
\begin{align*}
    &\ep_{X'Z'}[\nabla_{\theta}f_{\theta^*}(X')k(Z',Z)]=\ep_{Z'}[\ep_{X'}[\nabla_{\theta}f_{\theta^*}(X')|Z']k(Z',Z)]=0,
\end{align*}
where we note that $\ep[\nabla_{\theta}f_{\theta^*}(X')\,|\,Z']p(Z')=0$ for any $Z'$ implied by the second function above. 
Similarly, the coefficient of $\nabla_{\theta}f_{\theta^*}(X)$ must be zero, which implies that $\ep_{X'Y'}[(Y'-f_{\theta^*}(X'))\,|\,Z']p(Z')=0$ for any $Z'$. 
The two coefficients cannot be zero at the same time (otherwise against the given conditions), so $\mathrm{Var}_U[\ep_{U'}[\nabla_{\theta}h_{\theta^*}(U,U')]]>0$. 
Further due to the given condition $\ep[ \|\nabla_{\theta^*}h(U,U')\|_2^2] <\infty$, we obtain $\sqrt{n}\nabla_{\theta} \widehat{R}_U(f_{\theta^*}) \pto N(\bm 0, 4\ep_U[\ep^{2}_{U'}[\nabla_{\theta}h_{\theta^*}(U,U')]])$ as per \citep[Section 5.5.1]{Serfling80:Approximation}. Finally, as $\sqrt{n}\lambda \pto 0$ and $\nabla_{\theta}\Omega(\theta^*) < \infty$ by the condition that $\Omega(\theta)$ is first continuously differentiable, we assert the conclusion by Slutsky's theorem,
\begin{align*}
    \sqrt{n}\nabla_{\theta} \widehat{R}_{U,\lambda}(f_{\theta^*}) = \sqrt{n}\nabla_{\theta} \widehat{R}_{U}(f_{\theta^*}) + \sqrt{n}\lambda \nabla_{\theta}\Omega(\theta^*)   \pto N(\bm 0, 4\mathrm{diag}(\ep_U[\ep^{2}_{U'}[\nabla_{\theta}h_{\theta^*}(U,U')]])).
\end{align*}
This concludes the proof.
\end{proof}

\subsection{Uniform consistency of \texorpdfstring{$\nabla_{\theta}^2 \widehat{R}_{U,\lambda}(f_{\theta})$}{}}
Next, we consider the second derivative $\nabla_{\theta}^2 \widehat{R}_{U,\lambda}(f_{\theta})$ and show its uniform consistency. 
In what follows, we denote by $\|\cdot\|_{\mathrm{F}}$ the Frobenius norm.
We can express $\nabla_{\theta}^2 \widehat{R}_{U,\lambda}(f_{\theta})$ as
\begin{eqnarray}
    \nabla_{\theta}^2 \widehat{R}_{U,\lambda}(f_{\theta}) &=& \underbrace{\dfrac{1}{n(n-1)}\sum_{i=1}^{n}\sum_{j\neq i}^n \nabla_{\theta}^2 h_{\theta}(u_i,u_j)}_{\nabla_{\theta}^2 \widehat{R}_U(f_{\theta})}
    +\lambda \nabla_{\theta}^2 \Omega(\theta)
    \\
   \nabla_{\theta}^2 h_{\theta}(u_i,u_j) &=& [\nabla_\theta f_{\theta}(x_i)\nabla_\theta f_{\theta}^{\top}(x_j)-(y_i-f_{\theta}(x_i))\nabla_{\theta}^2 f_{\theta}(x_j)
    \\
    && \quad +\nabla_\theta f_{\theta}(x_j)\nabla_\theta f_{\theta}^{\top}(x_i)-(y_j-f_{\theta}(x_j))\nabla_{\theta}^2 f_{\theta}(x_i)]k(z_i,z_j).
\end{eqnarray}
\begin{lemma}\label{thm:unif_cons_2nd_deri_u}
Suppose that $f_{\theta}$ and $\Omega(\theta)$ are twice continuously differentiable about $\theta$, $\Theta$ is compact, $\ep \left [|f_\theta (X)|\right]<\infty$, $\ep \left [\|\nabla_\theta f_\theta (X)\|_2\right]<\infty$, $\ep \left [\|\nabla_\theta^2 f_\theta (X)\|_F\right]<\infty$, $\ep[|Y|]<\infty$, $\lambda \pto 0$ and \sref{Assumption}{asmp:k_idp} holds. Then, 
$\ep[\nabla_{\theta}^2 h_{\theta}(U,U')]$ is continuous about $\theta$ and
$$\sup_{\theta \in \Theta} \left\| \nabla_{\theta}^2 \widehat{R}_{U,\lambda}(f_{\theta}) - \ep[\nabla_{\theta}^2 h_{\theta}(U,U')]\right\|_{\mathrm{F}} \pto 0 .$$
\end{lemma}
\begin{proof}
The proof is similar to that of \sref{Lemma}{thm:unif_cons_emp_ru} and both applies extended results of \citep[Lemma 2.4]{NEWEY19942111}. As $(u_i,u_j)_{i\neq j}$ being strictly stationary in the sense of \citep[Footnote 18 in P.2129]{NEWEY19942111} has been shown in \sref{Lemma}{thm:unif_cons_emp_ru}, we only need to show that (i) $\nabla_{\theta}^2 h_{\theta}(u,u')$ is continuous at each $\theta \in \Theta$ with probability one and (ii) there exists $d(u,u') \geq \|\nabla_{\theta}^2 h_{\theta}(u,u')\|_F$ for all $\theta \in \Theta$ and $\ep[d(U,U')] < \infty$. 
We exploit the triangle inequality of the Frobenius norm and obtain
\begin{align}
    &\left\|\nabla_{\theta}^2 h_{\theta}(u,u')\right\|_{\mathrm{F}}\\
    &\leq \left [2\| \nabla_\theta f_{\theta}(x)\nabla_\theta f_{\theta}^{\top}(x')\|_F+(|y|+|f_{\theta}(x)|)
    \|\nabla_{\theta}^2 f_{\theta}(x')\|_{\mathrm{F}}
    +(|y'|+|f_{\theta}(x')|)
    \|\nabla_{\theta}^2 f_{\theta}(x)\|_{\mathrm{F}} \right ]k(z,z')
    \\
    & \quad \equiv d(u,u'),
\end{align}
We first show $d(u,u')$ is bounded for bounded $u,u'$. As $f_\theta$ is twice continuously differentiable about $\theta$ and $\Theta$ is compact, we have $f_{\theta}(x)$ bounded as well as each entry of $\nabla_\theta f_{\theta}(x)$ and $\nabla_{\theta}^2 f_{\theta}(x)$ for $\|x\|<\infty$. Further taking into account that $k(\cdot,\cdot)$ is bounded as per \sref{Assumption}{asmp:k_idp}, we know that $d(u,u')<\infty$ if $u,u'$ are bounded, and it follows that (i) $\nabla_{\theta}^2 h_{\theta}(u,u')$ is continuous at each $\theta \in \Theta$ with probability one as $f_\theta$ is twice continuously differentiable.

We then show that (ii) $\ep_{U,U'}[d(U,U')]<\infty$ by the following inequalities:
\begin{align}
    &\ep_{U,U'}[d(U,U')] \\
    &\leq 2\ep \left [\| \nabla_\theta f_{\theta}(X)\nabla_\theta f_{\theta}^{\top}(X')\|_F+(|Y|+|f_{\theta}(X)|)
    \|\nabla_{\theta}^2 f_{\theta}(X')\|_{\mathrm{F}}
    \right ] \sup_z k(z,z)
    \\
    & = 2\left(\underbrace{\ep \left [\| \nabla_\theta f_{\theta}(X)\|_F\right]}_{<\infty}\ep \left[ \|\nabla_\theta f_{\theta}^{\top}(X')\|_F\right]+\underbrace{\ep \left[ |Y|+|f_{\theta}(X)|
    \right]}_{<\infty}\underbrace{\ep \left[\|\nabla_{\theta}^2 f_{\theta}(X')\|_{\mathrm{F}}
    \right ]}_{<\infty}\right) \sup_z k(z,z)
    \\
    & < \infty.
\end{align}
Therefore, we obtain $\sup_{\theta \in \Theta}\| \nabla_{\theta}^2 \widehat{R}_{U}(f_{\theta}) - \ep[\nabla_{\theta}^2 h_{\theta}(U,U')]\|_{\mathrm{F}} \pto 0$ following from the extended results in the remarks of \citep[Lemma 2.4]{NEWEY19942111}. 
Furthermore, from the conditions that  $\Omega(\theta)$ is twice continuously differentiable and the parameter space $\Theta$ is compact, we obtain that $\|\nabla_\theta^2 \Omega(\theta)\|_{\mathrm{F}} < \infty$ for any $\theta \in \Theta$. 
Finally, it follows from the Slutsky's theorem that
\begin{align*}
    &\sup_{\theta \in \Theta} \left\| \nabla_{\theta}^2 \widehat{R}_{U,\lambda}(f_{\theta}) - \ep[\nabla_{\theta}^2 h_{\theta}(U,U')]\right\|_{\mathrm{F}} \\
    &\leq \sup_{\theta \in \Theta} \left\| \nabla_{\theta}^2 \widehat{R}_{U}(f_{\theta}) - \ep[\nabla_{\theta}^2 h_{\theta}(U,U')]\right\|_{\mathrm{F}}+\lambda\sup_{\theta \in \Theta} \left\|  \nabla_{\theta}^2\Omega(\theta)\right\|_{\mathrm{F}} \pto 0.
\end{align*}
This concludes the proof.
\end{proof}

\begin{theorem}[Asymptotic normality of $\hat{\theta}_U$]\label{thm:asym_norm_theta_u}
Suppose that $H= \ep[\nabla^2_{\theta} h_{\theta^*}(U,U')]$ is non-singular, $\Theta$ compact, $\ep \left [|f_\theta (X)|\right]<\infty$, $\ep[|Y|]<\infty$, $f_\theta$ and $\Omega(\theta)$ are twice continuously differentiable about $\theta$, $\ep \left [\|\nabla_\theta f_\theta (X)\|_2\right]<\infty$, $\ep \left [\|\nabla_\theta^2 f_\theta (X)\|_F\right]<\infty$, $\sqrt{n}\lambda \pto 0$, $R_k(f_\theta)$ is uniquely minimized at $\theta^*$ which is an interior point of $\Theta$, $\ep[\|\nabla_{\theta} h_{\theta^*}(U,U')\|_2^2] <\infty$  and \sref{Assumptions}{asmp:k_idp} hold. Then \[\sqrt{n} (\hat{\theta}_U -\theta^*) \pto N(\bm 0,  4H^{-1}\mathrm{diag}(\ep_U[\ep^{2}_{U'}[h_{\theta^*}(U,U')]])H^{-1}).\]
\end{theorem}
\begin{proof}
The proof follows by \citep[Theorem 3.1]{NEWEY19942111} and we need to show that (i) $\widehat{\theta}_U \pto \theta^{*}$; (ii) $\widehat{R}_{U,\lambda}(\theta)$ is twice continuously differentiable; (iii) $\sqrt{n}\nabla_{\theta} \widehat{R}_{U,\lambda}(f_{\theta^*}) \pto N(\bm 0, 4\ep_U[\ep^{2}_{U'}[h_{\theta^*}(U,U')]])$; (iv) there is $H(\theta)$ that is continuous at $\theta^*$ and $\sup_{\theta \in \Theta}\| \nabla_{\theta}^2\widehat{R}_{U,\lambda}(f_{\theta}) - H(\theta)\|_F \pto 0$; (v) $H(\theta^*)$ is nonsingular.

The proof of (i) is very similar to \sref{Theorem}{thm:cons_emp_fu} except that we consider finite dimensional parameter space instead of functional space. For a neat proof, we would like to omit the detailed proof here. We can first show the uniform consistency $\sup_{\theta \in \Theta}|  \widehat{R}_{U,\lambda}(f_{\theta}) - R_k(f_\theta)| \pto 0$ and $R_k(f_\theta)$ is continuous about $\theta$ similarly to \sref{Lemma}{thm:unif_cons_emp_ru}. Here, the proof is based on the conditions $\ep[| Y |]<\infty$, $\Theta$ is compact, $\ep[|f_\theta(X)|]<\infty$ and $f_\theta$ is twice continuously differentiable about $\theta$, and \sref{Assumption}{asmp:k_idp} holds. Then, $\widehat{\theta}_U \pto \theta^{*}$  similarly to \sref{Theorem}{thm:cons_emp_fu}, because of the extra condition $R_k(f_\theta)$ is uniquely minimized at $\theta^*$.

Furthermore, from the conditions that $\Theta$ is compact, $f_\theta$ is twice continuously differentiable about $\theta$,  $\ep \left [|f_\theta (X)|\right]<\infty$, $\ep \left [\|\nabla_\theta f_\theta (X)\|_2\right]<\infty$, $\ep \left [\|\nabla_\theta^2 f_\theta (X)\|_F\right]<\infty$, $\ep[|Y|]<\infty$ and $k(z,z')$ is bounded as implied by \sref{Assumption}{asmp:k_idp}, we can obtain (ii) $\widehat{R}_{V,\lambda}(\theta)$ is twice continuously differentiable about $\theta$. Given $H=\ep[\nabla^2_{\theta} h_{\theta^*}(U,U')] = \nabla^2_{\theta} R_k(\theta)$ is non-singular and $R_k(f_\theta)$ is uniquely minimized at $\theta^*$, we can obtain that the Hessian matrix $H$ is positive definite,
\begin{align}
    H &= 2\ep_{(XYZ),(X'Y'Z')} \left[\left(\nabla_\theta f_{\theta^*}(X)\nabla_\theta f_{\theta^*}^{\top}(X')-(Y-f_{\theta^*}(X))\nabla_{\theta}^2 f_{\theta^*}(X')\right)k(Z,Z')\right]\succ \bm 0.
\end{align}
If for all $z \in \inz$, there is
$\ep_{X}[\nabla_{\theta}f_{\theta^*}(X)\,|\,z]p(z) = \bm 0$ and $\ep_{XY}[Y-f_{\theta^*}(X)\,|\,z]p(z) = 0$, then we can see that the above function $H = \bm 0$ which contradicts $H\succ \bm 0$. Therefore, there must exist $z$ s.t. $\ep_{X}[\nabla_{\theta}f_{\theta^*}(X)\,|\,z]p(z) \neq \bm 0$ or $\ep_{XY}[Y-f_{\theta^*}(X)\,|\,z]p(z) \neq 0$. 
Then, it follows by \sref{Lemma}{thm:asym_norm_1st_deri_u} that (iii) $\sqrt{n}\nabla_{\theta} \widehat{R}_{U,\lambda}(f_{\theta^*}) \pto N(\bm 0, 4\ep_U[\ep^{2}_{U'}[h_{\theta^*}(U,U')]])$.

Finally by \sref{Lemma}{thm:unif_cons_2nd_deri_u}, we know that $H(\theta) = \ep[\nabla^2_{\theta} h_{\theta}(U,U')]$ and $H(\theta^*)=H$, so (iv) and (v) are satisfied. Now, conditions of \citep[Theorem 3.1]{NEWEY19942111} are all satisfied, so we assert the conclusion.
\end{proof}

\subsection{Proof of  \texorpdfstring{\sref{Theorem}{thm:asym_norm_theta_v}}{}}

We restate the notations
\begin{eqnarray}
     \widehat{R}_{V,\lambda}(f_{\theta}) &\coloneqq& \underbrace{\dfrac{1}{n^2}\sum_{i=1}^{n}\sum_{j=1}^n h_{\theta}(u_i,u_j)}_{\widehat{R}_{V}(f_{\theta}) }+\lambda \Omega(\theta)
    \\
     h_{\theta}(u_i,u_j) &\coloneqq& (y_i-f_{\theta}(x_i))k(z_i,z_j)(y_j-f_{\theta}(x_j)),
\end{eqnarray}

\begin{lemma}
Suppose that conditions of \sref{Lemma}{thm:asym_norm_1st_deri_u} hold. Then $\sqrt{n}\nabla_{\theta} \widehat{R}_{V,\lambda}(f_{\theta^*}) \pto N(\bm 0, 4\ep_U[\ep^{2}_{U'}[\nabla_\theta h_{\theta^*}(U,U')]])$.
\end{lemma}
\begin{proof}
As $\ep[\|\nabla_{\theta} h_{\theta^*}(U,U')\|_2^2] <\infty$, $\sqrt{n}\nabla_{\theta} \widehat{R}_V(f_{\theta^*})$ has the same limit distribution as that of $\sqrt{n}\nabla_{\theta} \widehat{R}_U(f_{\theta^*})$ by \citep[Section 5.7.3]{Serfling80:Approximation}. 
Furthermore, by $\sqrt{n}\lambda \pto 0$ and $\nabla_\theta \Omega(\theta^*)<\infty$ from that $\Omega(\theta)$ is first continuously differentiable, we assert the conclusion by Slutsky's theorem
$$\sqrt{n}\nabla_{\theta} \widehat{R}_{V,\lambda}(f_{\theta^*})= \sqrt{n}\nabla_{\theta} \widehat{R}_{V}+\sqrt{n}\lambda \nabla_{\theta}\Omega(\theta^*)   \pto N(\bm 0, 4\ep_U[\ep^{2}_{U'}[\nabla_{\theta}h_{\theta^*}(U,U')]]).$$
\end{proof}

\begin{lemma}\label{thm:unif_cons_2nd_deri_v}
Suppose that  $f_{\theta}$ and $\Omega(\theta)$ are twice continuously differentiable about $\theta$, $\Theta$ is compact, $\ep[\sup_{\theta \in \Theta} |f_\theta(X)|^2 ]<\infty$, $\ep[\sup_{\theta \in \Theta} \|\nabla_\theta f_\theta(X)\|_2^2 ]<\infty$, $\ep[\sup_{\theta \in \Theta} \|\nabla_\theta^2 f_\theta(X)\|_F^2 ]<\infty$, $\ep[|Y|^2]<\infty$, $\lambda \pto 0$ and \sref{Assumption}{asmp:k_idp} holds. Then, 
$\ep[\nabla_{\theta}^2 h_{\theta}(U,U')]$ is continuous about $\theta$ and
$\sup_{\theta \in \Theta}\| \nabla_{\theta}^2 \widehat{R}_V(f_{\theta}) - \ep[\nabla_{\theta}^2 h_{\theta}(U,U')]\|_{\mathrm{F}} \pto 0$.
\end{lemma}
\begin{proof}
We apply \citep[Lemma 8.5]{NEWEY19942111} for this proof and need to show (i) $\nabla_\theta^2 h_\theta(u,u')$ is continuous about each $\theta \in \Theta$ with probability one, and (ii) $\ep[ \sup_{\theta \in \Theta}\| \nabla_\theta^2 h_\theta(U,U')\|_{\mathrm{F}}]<\infty$ and $\ep[ \sup_{\theta \in \Theta}\| \nabla_\theta^2 h_\theta(U,U)\|_{\mathrm{F}}]<\infty$.

We first see that $\|\nabla_\theta^2 h_\theta(u,u')\|_{\mathrm{F}}$ and $\|\nabla_\theta^2 h_\theta(u,u)\|_{\mathrm{F}}$ are bounded for finite $u,u'$ because $f_\theta$ is twice continuously differentiable about $\theta$. It follows that (i) $\nabla_\theta^2 h_\theta(u,u')$ is continuous about $\theta$ with probability one. We then derive upper bounds for $\ep[ \sup_{\theta \in \Theta}\| \nabla_\theta^2 h_\theta(U,U')\|_{\mathrm{F}}]$ and $\ep[ \sup_{\theta \in \Theta}\| \nabla_\theta^2 h_\theta(U,U)\|_{\mathrm{F}}]$ so as to show their boundedness,
\begin{align}
    &\ep[ \sup_{\theta \in \Theta}\| \nabla_\theta^2 h_\theta(U,U')\|_{\mathrm{F}}] \\
    & \leq 2\ep \left [\sup_{\theta \in \Theta} \| \nabla_\theta f_{\theta}(X)\nabla_\theta f_{\theta}^{\top}(X')\|_F+(|Y|+|f_{\theta}(X)|)
    \|\nabla_{\theta}^2 f_{\theta}(X')\|_{\mathrm{F}}
    \right ] \sup_z k(z,z)
    \\
    & \leq 2\ep \left [\sup_{\theta \in \Theta} \| \nabla_\theta f_{\theta}(X)\|_2\right]^2+2\ep\left[|Y|+\sup_{\theta \in \Theta} |f_{\theta}(X)|\right]\ep\left[ \sup_{\theta \in \Theta}
    \|\nabla_{\theta}^2 f_{\theta}(X')\|_{\mathrm{F}}
    \right ] \sup_z k(z,z) <\infty,
\end{align}
and
\begin{align}
    &\ep[ \sup_{\theta \in \Theta}\| \nabla_\theta^2 h_\theta(U,U)\|_{\mathrm{F}}] \\
    & \leq 2\ep \left [\sup_{\theta \in \Theta} \| \nabla_\theta f_{\theta}(X)\nabla_\theta f_{\theta}^{\top}(X)\|_F+(|Y|+|f_{\theta}(X)|)
    \|\nabla_{\theta}^2 f_{\theta}(X)\|_{\mathrm{F}}
    \right ] \sup_z k(z,z)
    \\ 
    &\leq 2\ep \left [\sup_{\theta \in \Theta} \| \nabla_\theta f_{\theta}(X)\|_2^2\right]+2\ep\left[\left(|Y|+\sup_{\theta \in \Theta} |f_{\theta}(X)|\right)^2\right]\ep\left[ \sup_{\theta \in \Theta}
    \|\nabla_{\theta}^2 f_{\theta}(X')\|_{\mathrm{F}}^2
    \right ] \sup_z k(z,z)
    \\
    &\leq 2\ep \left [\sup_{\theta \in \Theta} \| \nabla_\theta f_{\theta}(X)\|_2^2\right]+\ep\left[(|Y|+\sup_{\theta \in \Theta} |f_{\theta}(X)|)^2\right]+\ep\left[ \sup_{\theta \in \Theta}
    \|\nabla_{\theta}^2 f_{\theta}(X')\|_{\mathrm{F}}^2
    \right ] \sup_z k(z,z)
    \\
    &\leq 2\ep \left [\sup_{\theta \in \Theta} \| \nabla_\theta f_{\theta}(X)\|_2^2\right]+2\ep\left[|Y|^2+\sup_{\theta \in \Theta} |f_{\theta}(X)|^2\right]+\ep\left[ \sup_{\theta \in \Theta}
    \|\nabla_{\theta}^2 f_{\theta}(X')\|_{\mathrm{F}}^2
    \right ] \sup_z k(z,z)
    \\
    & <\infty.
\end{align}

Thus, we assert the conclusion by \citep[Lemma 8.5]{NEWEY19942111}.
\end{proof}

\begin{proof}[Proof of \sref{Theorem}{thm:asym_norm_theta_v}]
The proof is the same as that of \sref{Theorem}{thm:asym_norm_theta_u} except that $\widehat{R}_{U}$ is replaced by $\widehat{R}_{V}$.
\end{proof}

\subsection{Asymptotic Normality in the Infinite-dimension Case} 
\label{sec:asym_inf_proof}

We firstly state the asymptotic normality theorem for $\hat{f}_U$ and its proof.
Afterwards, we provide the proof of \sref{Theorem}{thm:asympnorm} whose proof is a slightly modified version of that of $\hat{f}_U$.

\begin{theorem} \label{thm:asympnorm_U}
    Suppose \sref{Assumption}{asmp:k_idp} holds,  $l$ is a bounded kernel, $k$ is a uniformly bounded function, and $\lambda \geq \lambda_0$ holds. 
    Also, suppose that $\inx$, $\inz$, and $\iny$ are compact spaces, and there exists $s \in (0,2)$ and a constant $C_H > 0$ such that $\log \mathcal{N}(\varepsilon, \mathcal{H}_l, \|\cdot\|_{L^\infty}) \leq C_H \varepsilon^{-s}$ for any $\varepsilon \in (0,1)$.
    If $\lambda - \lambda_0 = o(n^{-1/2} )$ holds, then there exists a Gaussian process $\mathbb{G}_P^*$ such that
    \begin{align*}
        \sqrt{n}(\hat{f}_{U} - f^*_{\lambda_0}) \rightsquigarrow \mathbb{G}_P^* \mbox{~in~} \mathcal{H}_l.
    \end{align*}
\end{theorem}

An exact covariance of $\mathbb{G}_P^*$ is described in the proof.
The proof is based on the uniform convergence of U-processes on the function space \citep{arcones1993limit} and the functional delta method using the asymptotic expansion of the loss function \citep{hable2012asymptotic}.
This asymptotic normality allows us to perform statistical inference, such as tests, even in the non-parametric case.

To prove the theorem, we provide some notation.
Let $P$ be a probability measure which generates $u = (x,y,z)$ and $\inw = \inx \times \iny \times \inz$.
Also, we define a function $h_f(u,u') = (y - f(x)) (y' - f(x')) k(z, z')$.
Let $\mathbb{H}:= \{h_f: \inw \times \inw \to \rr \mid f \in \mathcal{H}_l\}$.
For preparation, we define $P^1{h}_f: \inw \to \rr$ as $P^1{h}_f(\cdot) = (\int h_f(u,\cdot) + h_f(\cdot,u) dP(u))/2$ for $h_f \in \mathbb{H}$.
For a signed measure $Q$ on $\inw$, we define a measure $Q^2 := Q \otimes Q$ on $\inw \times \inw$.
Then, we can rewrite the U-statistic risk as
\begin{align*}
    \widehat{R}_U(f) &= \frac{(n-2)!}{n!} \sum_{i=1}^n \sum_{j \neq i}^n h_f(u_i, u_j) =: U_n^2 h_f,
\end{align*}
where $U_n$ is an empirical measure for the $U$-statistics.
Similarly, we can rewrite the V-statistic risk as
\begin{align*}
    \widehat{R}_V(f) &= \frac{1}{n^2} \sum_{i=1}^n \sum_{j =1}^n h_f(u_i, u_j) =: P_n^2 h_f,
\end{align*}
where $P_n$ is an empirical measure of $u$. 

Further, we define a functional associated with measure.
We consider functional spaces $\mathcal{G}_1 := \{g: \inw \times \inw \to [0,1]\mid \mbox{a~convex~set~}\omega \mbox{~s.t.~} g=\textbf{1}\{\cdot \leq \omega\} \}$ and $\mathcal{G}_2 := \{g: \inw \times \inw \to [0,1]\mid \exists f,f' \in \mathcal{H}_l, g(u,u') = h_f(u,u')(f'(x) + f(x')) \}$.
Note that $\mathcal{G}_1$ contains a functional which corresponds to the $U_n^2$.
Then, we consider a set of functionals
\begin{align*} 
    B_S := \left\{F : \mathcal{G}_1 \cup \mathcal{G}_2 \to \rr \mid \exists \mbox{non-zero finite } Q^2 \mbox{~s.t.~} F(g) = \int g \,d Q^2, \quad \forall g \in \mathcal{G}_1 \cup \mathcal{G}_2 \right\},
\end{align*}
and also let $B_0$ be the closed linear span of $B_S$.
For a functional $F \in B_S$, let $\iota(F)$ be a measure which satisfies
\begin{align*}
    F(g) = \int g \,d\iota(F).
\end{align*}

\textbf{Uniform Central Limit Theorem}:
We firstly achieve the uniform convergence.
Note that a measure $P^2$ satisfies $P^2 h_f := \ep_{(U,U')}[h_f(U,U')]$.
The convergence theorem for U-processes is as follows:
\begin{theorem}[Theorem 4.4 in \citep{arcones1993limit}]\label{thm:u-process}
    Suppose $\mathbb{H}$ is a set of uniformly bounded class and symmetric functions, such that $\{P^1h_f \mid  h_f \in \mathbb{H}\}$ is a Donsker class and  
    \begin{align*} 
        \lim_{n \to \infty} \ep[n^{-1/2} \log \mathcal{N}(n^{-1/2} \varepsilon, \mathbb{H}, \|\cdot\|_{L^1(U_n^2)})] = 0
    \end{align*}
    holds, for all $\varepsilon > 0$.
    Then, we obtain
\begin{align*}
    \sqrt{n}(U_n^2 - P^2) \rightsquigarrow 2 \mathbb{G}_{P^{ 1}},\mbox{~in~}\ell^\infty(\mathbb{H}).
\end{align*}
    Here, $\mathbb{G}_{P^{1}}$ denotes a Brownian bridge, which is a Gaussian process on $\mathbb{H}$ with zero mean and a covariance 
\begin{align*}
    \ep_{U}[P^1{h}_f(U) P^1{h}_{f'}(U)] - \ep_{U}[P^1{h}_f(U)]\ep_{U}[P^1{h}_{f'}(U)],
\end{align*}
    with $ h_f,h_{f'} \in \mathbb{H}$.
\end{theorem}

To apply the theorem, we have to show that $\mathbb{H}$ satisfies the condition in  \sref{Theorem}{thm:u-process}.
We firstly provide the following bound:
\begin{lemma}\label{lem:bound}
    For any $f,f' \in \mathcal{H}_l$ {\color{black}such that $\|f\|_{L^\infty} \vee \|f'\|_{L^\infty} \leq B$} holds, $y ,y' \in [-B,B]$ and $u,u' \in \inw$, we have
    \begin{align*}
        |h_f(u,u') - h_{f'}(u,u')| \leq 4B|k(z,z')| \|f - f'\|_{L^\infty}.
    \end{align*}
\end{lemma}
\begin{proof}[Proof of \sref{Lemma}{lem:bound}]
    We simply obtain the following:
    \begin{align*}
        &|h_f(u,u') - h_{f'}(u,u')| \\
        &= |(y-f(x))k(z,z')(y'-f(x')) - (y-f'(x))k(z,z')(y'-f'(x'))| \\
        &= |k(z,z')| |(y-f(x))(y'-f(x')) - (y-f'(x))(y'-f'(x'))| \\
        &= |k(z,z')| |y'(f'(x) - f(x)) + y(f'(x') - f(x')) + f(x) f(x') - f'(x) f'(x')|  \\
        &=|k(z,z')| |y'(f'(x) - f(x)) + y(f'(x') - f(x')) +  f(x')(f(x) - f'(x)) - f'(x) ( f(x') - f'(x'))| \\
        & \leq  {\color{black}|k(z,z')| \{|y'||f'(x) - f(x)| + |y||f'(x') - f(x')| +  |f(x')||f(x) - f'(x)| - |f'(x)| | f(x') - f'(x')|\}} \\
        & \leq  4B|k(z,z')| \|f - f'\|_{L^\infty},
    \end{align*}
    as required.
\end{proof}

\begin{lemma} \label{lem:bracket_H}
    Suppose the assumptions of \sref{Theorem}{thm:asympnorm} hold.
    Then, the followings hold:
    \begin{enumerate}
        \item $\{P^1h_f \mid  h_f \in \mathbb{H}\}$ is a Donsker class.
        \item For any $\varepsilon > 0$, the following holds:
    \begin{align*}
        \lim_{n \to \infty} E[n^{-1/2} \log \mathcal{N}(n^{-1/2} \varepsilon, \mathbb{H}, \|\cdot\|_{L^1(U_n^2)})] = 0.
    \end{align*}
    \end{enumerate} 
\end{lemma}
\begin{proof}[Proof of \sref{Lemma}{lem:bracket_H}]
For preparation, fix $\varepsilon > 0$ and set $N= \mathcal{N}(\varepsilon, \mathcal{H}_l, \|\cdot\|_{L^\infty})$.
Also, let $Q$ be an arbitrary finite discrete measure.
Then, by the definition of a bracketing number, there exist $N$ functions $\{f_i \in \mathcal{H}_l\}_{i=1}^N$ such that for any $f \in \mathcal{H}_l$ there exists $i \in \{1,2,...,N\}$ such as $\|f - f_i\|_{L^\infty} \leq \varepsilon$.

For the first condition, as shown in Equation (2.1.7) in \citep{vanweak}, it is sufficient to show 
\begin{align}
     \sup_{Q}\log \mathcal{N}(\varepsilon,\{P^1h_f \mid  h_f \in \mathbb{H}\}, \|\cdot\|_{L^2(Q)} ) \leq c \varepsilon^{-2\delta}~(\varepsilon \to 0), \label{cond:donsker}
\end{align}
for arbitrary $\delta \in (0,1)$.
Here, $c > 0$ is some constant, and $Q$ is taken from all possible finite discrete measure.
To this end, it is sufficient to show that $\log \mathcal{N}(\varepsilon,\{P^1h_f \mid  h_f \in \mathbb{H}\}, \|\cdot\|_{L^2(Q)} ) \leq c' \log \mathcal{N}(\varepsilon, \mathcal{H}_l, \|\cdot\|_{L^\infty} )$ with a constant $c' > 0$.
Fix $P^1 h_f \in \{P^1h_f \mid  h_f \in \mathbb{H}\}$ arbitrary, and set $f_i$ which satisfies $\|f - f_i\|_{L^2(Q)} \leq \varepsilon$.
Then, we have 
\begin{align*}
    &\left\|P^1h_f - P^1 h_{f_i} \right\|_{L^2(Q)}^2 \\
    &= \int \left\{ \int(h_f(u,u') + h_f(u',u))/2 - (h_{f_i}(u,u') + h_{f_i}(u',u))/2 \,dP(u') \right\}^2 dQ(u) \\
    &= \int  \left( \int h_f(u,u')  - h_{f_i}(u,u') \,dP(u')\right)^2 dQ(u) \\
    &\leq C \int \left(\int | k(z,z')| \,dP(u')\right)^2 dQ(u) \|f - f_i\|_{L^\infty}^2 \\
    & \leq  C' \|f - f_i\|_{L^\infty}^2 \\
    & \leq  C' \varepsilon^2,
\end{align*}
with constants $C,C' > 0$.
The first inequality follows \sref{Lemma}{lem:bound} with the bounded property of $f,f'$ and $\iny$.
The second inequality follows the bounded condition of $k$ in  \sref{Theorem}{thm:asympnorm_U}.
Hence, the entropy condition shows the first statement.

For the second condition, we have the similar strategy.
For any $h_f \in \mathbb{H}$, we consider $i \in \{1,2,\ldots,N\}$ such that $\|f-f_i\|_{L^\infty} \leq \varepsilon$. 
Then, we measure the following value
\begin{align*}
    &\|h_f - h_{f_i}\|_{L^1(U_n^2)} = \int |h_f(u,u') - h_{f_i}(u,u')| \,dU_n^2(u,u')  \leq C'' \|f-f_i\|_{L^\infty} \leq C'' \varepsilon,
\end{align*}
with a constant $C'' > 0$.
Hence, we have
\begin{align*}
    &\ep[n^{-1/2} \log \mathcal{N}(n^{-1/2} \varepsilon, \mathbb{H}, \|\cdot\|_{L^1(U_n^2)})] \leq n^{-1/2} \log \mathcal{N}(n^{-1/2} \varepsilon, \mathcal{H}_l, \|\cdot\|_{L^\infty}) \\
    &\qquad  \leq C n^{-1/2} \left( \frac{n^{1/2}}{\varepsilon} \right)^s = C n^{(s-1)/2} \to 0, ~(n \to \infty),
\end{align*}
since $s \in (0,1)$.
\end{proof}

From  \sref{Theorem}{thm:u-process} and  \sref{Lemma}{lem:bracket_H}, we rewrite the central limit theorem utilizing terms of functionals.
Note that $\iota^{-1}(U_n^2), \iota^{-1}(P^2) \in B_S$ holds.
Then, we can obtain
\begin{align}
    \sqrt{n}(\iota^{-1}(U_n^2) - \iota^{-1}(P^2))  \rightsquigarrow 2 \mathbb{G}_{P^{ 1}} \mbox{~in~}\ell^\infty(\mathbb{H}). \label{conv:clt_functional}
\end{align}

\textbf{Learning Map and Functional Delta Method}:
We consider a learning map $S: B_S \to \mathcal{H}_l$.
For a functional $F \in B_S$, we define
\begin{align*}
    S_\lambda(F) := \argmin_{f \in \mathcal{H}_l} \iota(F) h_f + \lambda \|f\|_{\mathcal{H}_l}^2.
\end{align*}
Obviously, we have
\begin{align*}
    \hat{f} = S_{\lambda}(\iota^{-1}(U_n^2)), \mbox{~and~} f^*_{\lambda_0} = S_{\lambda_0}(\iota^{-1}(P^2)).
\end{align*}

We consider a derivative of $S_\lambda$ in the sense of the Gateau differentiation by the following steps.

Firstly, we define a partial derivative of the map $R_{Q^2}(f)$.
To investigate the optimality of the minimizer of
\begin{align*}
    R_{Q^2,\lambda}(f) := \int h_f(u,u') \,dQ^2(u,u') + \lambda \|f\|_{\mathcal{H}_l}^2.
\end{align*}
To this end, we consider the following derivative $ \nabla R_{Q^2,\lambda}[f]:\mathcal{H}_l \to \mathcal{H}_l$ with a direction $f$ as
\begin{align*}
    \nabla R_{Q^2,\lambda}[f](f') := 2 \lambda f + \int \partial_{f,1} h_f(u,u') f'(x) + \partial_{f,2}h_f(u,u') f'(x') \,dQ^2(u,u').
\end{align*}
Here, $\partial_{f,1} h_f$ is a partial derivative of $h_f$ in terms of the input $f(x)$ as
\begin{align*}
    \partial_{f,1} h_f(u,u') = \partial_{t | t= f(x)} (y - t)  k(z, z')(y' - f(x'))
    &= -(y' - f(x')) k(z,z'),
\end{align*}
and $\partial_{f,2} h_f$ follows it respectively.
The following lemma validates the derivative:
\begin{lemma} \label{lem:deriv_risk}
    If the assumptions in \sref{Theorem}{thm:asympnorm} hold, then $\nabla R_{Q^2,\lambda}[f]$ is a Gateau-derivative of $R_{Q^2,\lambda}$ with the direction $f \in \mathcal{H}_l$.
\end{lemma}
\begin{proof}[Proof of  \sref{Lemma}{lem:deriv_risk}]
We consider a sequence of functions $h_n \in \mathcal{H}_l$ for $n \in \mathbb{N}$, such that $h_n(x) \neq 0, \forall x \in \inx$ and $\|h_n\|_{L^\infty} \to 0$ as $n \to \infty$.
Then, for $f \in \mathcal{H}_l$, a simple calculation yields
\begin{align*}
    \MoveEqLeft \left| \frac{R_{Q^2, \lambda}(f + h_n) -R_{Q^2, \lambda}(f ) - \nabla R_{Q^2, \lambda}[f](h_n)  }{\|h_n\|_{L^\infty}} \right| \\
    & \leq \int \|h_n\|_{L^\infty}^{-1} | k(z,z')((y-f(x))h_n(x) \\
    & \qquad + (y'-f(x'))h_n(x') + h_n(x)h_n(x')) - \nabla R_{Q^2, \lambda}[f](h_n) | \,dQ^2(u,u') \\
    & \leq \int \|h_n\|_{L^\infty}^{-1}| k(z,z') h_n(x) h_n(x') | \,dQ^2(u,u') \\
    & \leq \int  \|h_n\|_{L^\infty}^{-1} | k(z,z') \|h_n\|_{L^\infty}^2 | \,dQ^2(u,u')\\
    &\leq \|h_n\|_{L^\infty} \int |k(z,z')| \,dQ^2(u,u')  \to 0, ~(n \to \infty).
\end{align*}
The convergence follows the definition of $h_n$ and the absolute integrability of $k$, which follows the bounded property of $k$ and compactness of $\inz$.
Then, we obtain the statement.
\end{proof}

Here, we consider its RKHS-type formulation of $\nabla R_{Q^2,\lambda}$, which is convenient to describe a minimizer.
Let $\Phi_l : \inx \to \mathcal{H}_l$ be the feature map associated with the RKHS $\mathcal{H}_l$, such that $\langle \Phi_l[x], f \rangle_{\mathcal{H}_l} = f(x)$ for any $x \in \inx$ and $f \in \mathcal{H}_l$.
Let $\nabla \tilde{R}_{Q^2,\lambda}:\mathcal{H}_l \to \mathcal{H}_l$ be an operator such that
\begin{align*}
    \nabla \tilde{R}_{Q^2,\lambda}(f) := 2\lambda f + \int \partial_{f,1} h_f(u,u') \Phi_l[x](\cdot) + \partial_{f,2}h_f(u,u') \Phi_l[x'](\cdot) \,dQ^2(u,u').
\end{align*}
Obviously, $\nabla {R}_{Q^2,\lambda}[f](\cdot) = \langle \nabla \tilde{R}_{Q^2,\lambda}(f), \cdot \rangle_{\mathcal{H}_l}$.
Now, we can describe the first-order condition of the minimizer of the risk.
Namely, we can state that
\begin{align*}
    \hat{f} = \argmin_{f \in \mathcal{H}_l} R_{Q^2,\lambda}(f)  \Leftrightarrow  \nabla \tilde{R}_{Q^2,\lambda}(\hat{f})= 0.
\end{align*}
This equivalence follows Theorem 7.4.1 and Lemma 8.7.1 in \citep{luenberger1997optimization}.

Next, we apply the implicit function theorem to obtain an explicit formula of the derivative of $S$.
To this end, we consider a second-order derivative $\nabla^2 \tilde{R}_{Q^2,\lambda}:\mathcal{H}_l \to \mathcal{H}_l$ as
\begin{align*}
    \nabla^2 \tilde{R}_{Q^2,\lambda}(f) := 2 \lambda f + \int k(z,z')(f(x)\Phi_l[x](\cdot) + f(x')\Phi_l[x'](\cdot)) \,dQ^2(u,u'),
\end{align*}
which follows (b) in Lemma A.2 in \citep{hable2012asymptotic}.
Its basic properties are provided in the following result:
\begin{lemma}\label{lem:second_deriv}
    If \sref{Assumption}{asmp:k_idp} and the assumptions in \sref{Theorem}{thm:asympnorm} hold, then $\nabla^2 \tilde{R}_{Q^2,\lambda}$ is a continuous linear operator and it is invertible.
\end{lemma}
\begin{proof}[Proof of  \sref{Lemma}{lem:second_deriv}]
    By (b) in Lemma A.2 in \citep{hable2012asymptotic}, $\nabla^2 \tilde{R}_{Q^2,\lambda}$ is a continuous linear operator.
    In the following, we define $A:\mathcal{H}_l \to \mathcal{H}_l$ as $A(f) = \int k(z,z') f(x) \Phi_l[x](\cdot) + f(x') \Phi_l[x'](\cdot)) \,dQ^2(u,u')$.
    To show $\nabla^2 \tilde{R}_{Q^2,\lambda}$ is invertible, it is sufficient to show that (i) $\nabla^2\tilde{R}_{Q^2,\lambda}$ is injective, and (ii) $A$ is a compact operator.
    
    For the injectivity, we fix non-zero $f \in \mathcal{H}_l$ and obtain
    \begin{align*}
        &\|\nabla^2 \tilde{R}_{Q^2,\lambda}(f)\|_{\mathcal{H}_l}^2 \\
        &= \langle 2 \lambda f + A(f), 2 \lambda f + A(f)\rangle_{\mathcal{H}_l} \\
        & = 4 \lambda^2 \|f\|_{\mathcal{H}_l}^2 + 4 \lambda \langle f,A(f) \rangle_{\mathcal{H}_l} + \|A(f)\|_{\mathcal{H}_l}^2 \\
        & >  4 \lambda \langle f,A(f) \rangle_{\mathcal{H}_l} \\
        &=4 \lambda \left\langle f, \int k(z,z') f(x) \Phi[x] \,dQ^2(u,u') \right\rangle_{\mathcal{H}_l}  + 4 \lambda \left\langle f, \int k(z,z') f(x') \Phi[x'] \,dQ^2(u,u') \right\rangle_{\mathcal{H}_l} \\
        & = 4 \lambda  \int k(z,z') f(x)^2 \,dQ^2(u,u') + 4 \lambda  \int k(z,z') f(x')^2 \,dQ^2(u,u') \\
        & \geq 0.
    \end{align*}
    The last equality follows the property of $\Phi_l$ and the last inequality follows the ISPD property in \sref{Assumption}{asmp:k_idp}.
    
    For the compactness, we follow Lemma A.5 in \citep{hable2012asymptotic} and obtain that operators $(f \mapsto \int k(z,z') f(x) \Phi_l[x](\cdot) \,dQ^2(u,u'))$ and $(f \mapsto \int k(z,z') f(x') \Phi_l[x'](\cdot)\,dQ^2(u,u'))$ are compact.
\end{proof}

We define the Gateau derivative of $S$.
For a functional $F' \in \ell^\infty(\mathcal{G}_1 \cup \mathcal{G}_2)$, we define the following function
\begin{align*}
    &\nabla S_{Q^2,\lambda}(F') := - \nabla^2 \tilde{R}_{Q^2,\lambda}^{-1} \left( \int \partial_{f,1} h_{f_{Q^2}}(u,u') \Phi_l[x](\cdot) + \partial_{f,2}h_{f_{Q^2}}(u,u') \Phi_l[x'](\cdot)  \,d\iota(F')(u,u') \right),
\end{align*}
where $f_{Q^2} = S_\lambda(\iota^{-1}(Q^2))$ and {\color{black}$Q^2$ is a signed measure on $\inw \times \inw$.}
Then, we provide the following derivative theorem:
\begin{proposition}\label{prop:deriv_S}
    Suppose the assumptions in \sref{Theorem}{thm:asympnorm} hold.
    For $F \in B_S$, $F' \in \ell^\infty(\mathcal{G}_1 \cup \mathcal{G}_2)$, and $s \in \rr$ such that $F+sF' \in B_S$, $\nabla S_{\iota(F),\lambda}(F')$ is a Gateau-derivative of $S_\lambda$, namely, 
    \begin{align*}
        \lim_{s \to 0} \left\| \frac{S_\lambda(F+sF') - S_\lambda(F)}{s} - \nabla S_{\iota(F),\lambda}(F') \right\|_{\mathcal{H}_l} = 0.
    \end{align*}
\end{proposition}
\begin{proof}[Proof of  \sref{Proposition}{prop:deriv_S}]
This proof has the following two steps, (i) define a proxy operator $\Gamma$, then (ii) prove the statement by the implicit function theorem.

\textbf{(i) Define $\Gamma$}:
Note that $\iota(F')$ exists since $F+sF' \in B_S$ implies $F' \in B_0$.
We define the following operator $\Gamma(s,f,\lambda):\mathcal{H}_l \to \mathcal{H}_l$ for $f \in \mathcal{H}_l$:
\begin{eqnarray*}
    \Gamma(s,f,\lambda) &:=& \nabla \tilde{R}_{\iota(F) + s \iota(F'), \lambda } \\
    &=& 2\lambda f + \int \partial_{f,1} h_f(u,u') \Phi_l[x](\cdot) + \partial_{f,2}h_f(u,u') \Phi_l[x'](\cdot) \,d\iota(F)(u,u') \\
    && \quad + s\int \partial_{f,1} h_f(u,u') \Phi_l[x](\cdot) + \partial_{f,2}h_f(u,u') \Phi_l[x'](\cdot) \,d\iota(F')(u,u').
\end{eqnarray*}
For simple derivatives, Lemma A.2 in \citep{hable2012asymptotic} provides
\begin{align*}
    \nabla_s \Gamma(s,f,\lambda) = \int \partial_{f,1} h_f(u,u') \Phi_l[x](\cdot) + \partial_{f,2}h_f(u,u') \Phi_l[x'](\cdot) \,d\iota(F')(u,u'),
\end{align*}
and 
\begin{align*}
    \nabla_f \Gamma(s,f,\lambda) = \nabla^2 \tilde{R}_{\iota(F) + s \iota(F'),\lambda}.
\end{align*}

\textbf{(ii) Apply the implicit function theorem}: 
By its definition and the optimal conditions, we have
\begin{align*}
    \Gamma(s,f,\lambda ) = 0 \Leftrightarrow f = S_\lambda(\iota(F) + s \iota(F')).
\end{align*}
Also, we obtain
\begin{align*}
    \nabla_f\Gamma(0, S_\lambda(F), \lambda) = \nabla^2 \tilde{R}_{\iota(F),\lambda}.
\end{align*}
Then, for each $\lambda > 0$, by the implicit function theorem, there exists a smooth map $\varphi_\lambda:\rr \to \mathcal{H}_l$ such that
\begin{align*}
     \Gamma(s, \varphi_\lambda(s), \lambda) = 0, \; \forall s,
\end{align*}
and it satisfies
\begin{align*}
    \nabla_s \varphi_\lambda(0) = - \left(  \nabla_f\Gamma(0,\varphi_\lambda(0), \lambda) \right)^{-1} \left(  \nabla_s \Gamma(0,\varphi_\lambda(0),\lambda) \right) = \nabla S_{\iota(F),\lambda}(F').
\end{align*}
Also, we have $\varphi_\lambda(s) = S(Q^2 + s \mu^2)$.
Then, we have
\begin{align*}
    &\lim_{s \to 0} \left\| \frac{S_\lambda(F + s F') - S_\lambda(F')}{s} -  \nabla S_{\iota(F),\lambda}(F')\right\|_{\mathcal{H}_l} =  \lim_{s \to 0} \left\| \frac{\varphi_\lambda(s) - \varphi_\lambda(0)}{s} -   \nabla_s \varphi_\lambda(0)\right\|_{\mathcal{H}_l} = 0.
\end{align*}
Then, we obtain the statement.
\end{proof}

Now, we are ready to prove \sref{Theorem}{thm:asympnorm} and  \sref{Theorem}{thm:asympnorm_U}.
\begin{proof}[Proof of  \sref{Theorem}{thm:asympnorm_U}]
As a preparation, we mention that $S_\lambda$ is differentiable in the Hadamard sense, which is Gateau differentiable by  \sref{Proposition}{prop:deriv_S}.
Lemma A.7 and A.8 in \citep{hable2012asymptotic} show that $\nabla S_{\iota (F), \lambda, \iota(G)}$ is Hadamard-differentiable for any $\lambda$, $G$ and $F$.

Then, we apply the functional delta method.
As shown in  \sref{Theorem}{thm:u-process} and  \sref{Lemma}{lem:bracket_H}, we have
\begin{align*}
    \sqrt{n}(\iota^{-1}(U_n^2) - \iota^{-1}(P^2))   \rightsquigarrow 2 \mathbb{G}_{P^1}.
\end{align*}
Hence, we obtain
\begin{align*}
    &\sqrt{n}( (\lambda_0 / \lambda) \iota^{-1}(U_n^2) - \iota^{-1}(P^2)) = \frac{\lambda_0}{\lambda} \sqrt{n}(\iota^{-1}(U_n^2) - \iota^{-1}(P^2)) + \frac{\sqrt{n}(\lambda - \lambda_0)}{\lambda} \rightsquigarrow 2 \mathbb{G}_{P^1},
\end{align*}
since $\lambda - \lambda_0 = o(n^{-1/2})$.
Utilizing the result, we can obtain
\begin{align*}
    &\sqrt{n}(\hat{f} - f^*_{\lambda_0}) = \sqrt{n}(S_{\lambda_0}((\lambda_0/\lambda)\iota^{-1}(U_n^2)) - S_{\lambda_0} (\iota^{-1}(P^2))) + o_P(1) \rightsquigarrow \nabla S_{P^2,\lambda_0}(2\mathbb{G}_{P^1}),
\end{align*}
in $\ell^\infty(\mathbb{H})$.
The convergence follows the functional delta method.
\end{proof}

\begin{proof}[Proof of \sref{Theorem}{thm:asympnorm}]
This proof is completed by substituting the Central Limit Theorem part (\sref{Theorem}{thm:u-process}) of the proof of \sref{Theorem}{thm:asympnorm_U}.
From Section 3 in \citep{akritas1986empirical}, the V- and U- processes have the same limit distribution asymptotically, so the same result holds.
\end{proof}


\subsection{Proof of Proposition \ref{prop:efficiency}}
\begin{proof}[Proof of Proposition \ref{prop:efficiency}]
For short, we write $\Lambda = \Lambda_{\mathrm{min}}$ and $\Lambda' = \Lambda_{\mathrm{max}}$.

We start by investigating $H$ and $\check{H}$.
By using the independent property between $U=(X,Y,Z)$ and $U'=(X',Y',Z')$, we have
\begin{align*}
    H &=   \ep[\nabla^2_{\theta} (Y-f_{\theta}(X))(Y'-f_{\theta}(X'))k(Z,Z')] \\
    &=   \ep[\nabla_\theta^2 f_{\theta^*}(X) \varepsilon' k(Z,Z')] + \ep[\nabla_\theta^2 f_{\theta^*}(X') \varepsilon k(Z,Z')] + 2 \ep[\nabla_\theta f_{\theta^*}(X) k(Z,Z')\nabla_\theta f_{\theta^*}(X')] \\
    &=2 \ep_{Z,Z'}[ \ep_X[\nabla_\theta f_{\theta^*}(X) \mid Z] k(Z,Z')\nabla_\theta \ep_{X'}[ f_{\theta^*}(X') \mid Z']] \\
    &= \iint q(z) q(z') k(z,z') \mathrm{d}P(z,z')\\    &= \iint q(z)p(z) q(z') p(z')k(z,z') \mathrm{d}z \mathrm{d}z',
\end{align*}
where $q: \mathcal{Z} \to \mathbb{R}^d, z \mapsto \ep_X[\nabla_\theta f_{\theta^*}(X) \mid z] $ with the parameter dimension $d$.
The third equality follows that $\varepsilon = Y - f_{\theta^*}(X) $ is an independent noise variable, and the last equality follows the independent property between $Z$ and $Z'$.
By the same way, $\check{H}$ is rewritten as
\begin{align*}
    \check{H} = \iint q(z)p(z) q(z') p(z') \mathrm{d}z \mathrm{d}z'.
\end{align*}
Based on these facts, we obtain the following form
\begin{align*}
    H = \check{H} \circ W,
\end{align*}
where $W$ is a $d \times d$ matrix whose all elements belong to $[\Lambda, \Lambda']$, and $\circ$ denotes the Haramard product of matrices.
In words, the elements of $H$ are obtained by multiplying some coefficients to their corresponding elements of $\check{H}$, and the coefficients belongs to $[\Lambda, \Lambda']$.
By the same calculation, we obtain $\ep_U[\ep^{2}_{U'}[{h}_{\theta^*}(U,U')]] = \ep_U[\ep^{2}_{U'}[\check{h}_{\theta^*}(U,U')]]w^2$ with some $w \in [\Lambda, \Lambda']$.

Using the result, we study $H^{-1}$ with the decomposition
\begin{align*}
    H^{-1} = \check{H}^{-1} + (H^{-1} -  \check{H}^{-1}).
\end{align*}
Let $\mathbb{I}$ be a $d \times d$ matrices whose all elements are $1$.
We have the following difference
\begin{align*}
    H^{-1} - \check{H}^{-1}
    & = H^{-1} (\check{H} - H) \check{H}^{-1}\\
    & = H^{-1} (\check{H} - W \circ H) \check{H}^{-1} \\
    & = H^{-1} ( (\mathbb{I} - W) \circ \check{H} ) \check{H}^{-1} \\
    & = H^{-1} \check{H} \check{H}^{-1} (\mathbb{I} \circ (\mathbb{I} - W)) \\
    &= H^{-1}\circ (\mathbb{I} - W).
\end{align*}

Finally, we study $\Sigma_V$.
It is decomposed as
\begin{align*}
    \Sigma_V &= 4{H}^{-1}\mathrm{diag}(\ep_U[\ep^{2}_{U'}[{h}_{\theta^*}(U,U')]]){H}^{-1} \\
    &=4(\check{H}^{-1} + (H^{-1} -  \check{H}^{-1})) \\
    & \quad \times (\mathrm{diag}(\ep_U[\ep^{2}_{U'}[\check{h}_{\theta^*}(U,U')]]) + (\mathrm{diag}(\ep_U[\ep^{2}_{U'}[{h}_{\theta^*}(U,U')]]) - \mathrm{diag}(\ep_U[\ep^{2}_{U'}[\check{h}_{\theta^*}(U,U')]])))\\
    & \quad \times (\check{H}^{-1} + (H^{-1} -  \check{H}^{-1})) \\
    &=4\check{H}^{-1}  \mathrm{diag}(\ep_U[\ep^{2}_{U'}[\check{h}_{\theta^*}(U,U')]])\check{H}^{-1} \\
    &\quad + 4(H^{-1} -  \check{H}^{-1})\mathrm{diag}(\ep_U[\ep^{2}_{U'}[\check{h}_{\theta^*}(U,U')]]) \check{H}^{-1} \\
    &\quad + 8 \check{H}^{-1}(\mathrm{diag}(\ep_U[\ep^{2}_{U'}[{h}_{\theta^*}(U,U')]]) - \mathrm{diag}(\ep_U[\ep^{2}_{U'}[\check{h}_{\theta^*}(U,U')]]))\check{H}^{-1} \\
    &\quad + 8\check{H}^{-1}(\mathrm{diag}(\ep_U[\ep^{2}_{U'}[{h}_{\theta^*}(U,U')]]) - \mathrm{diag}(\ep_U[\ep^{2}_{U'}[\check{h}_{\theta^*}(U,U')]]))(H^{-1} -  \check{H}^{-1}) \\
    & \quad + 4(H^{-1} -  \check{H}^{-1})^2\mathrm{diag}(\ep_U[\ep^{2}_{U'}[\check{h}_{\theta^*}(U,U')]]) \\
    &\quad + 4(H^{-1} -  \check{H}^{-1})^2(\mathrm{diag}(\ep_U[\ep^{2}_{U'}[{h}_{\theta^*}(U,U')]]) - \mathrm{diag}(\ep_U[\ep^{2}_{U'}[\check{h}_{\theta^*}(U,U')]])) \\
    &= \check{\Sigma}_V   + 8 (w^2 - 1)\mathrm{diag}(\ep_U[\ep^{2}_{U'}[\check{h}_{\theta^*}(U,U')]])\check{H}^{-2} \\
    &\quad + 4(2 - w^2)\mathrm{diag}(\ep_U[\ep^{2}_{U'}[\check{h}_{\theta^*}(U,U')]]) \check{H}^{-1} (H^{-1} -  \check{H}^{-1}) \\
    &\quad +  4(H^{-1} -  \check{H}^{-1})^2 \mathrm{diag}(\ep_U[\ep^{2}_{U'}[{h}_{\theta^*}(U,U')]]) \\
    &= \check{\Sigma}_V   + 8 (w^2 - 1)\mathrm{diag}(\ep_U[\ep^{2}_{U'}[\check{h}_{\theta^*}(U,U')]])\check{H}^{-2} \\
    & \quad + 4(2 - w^2)\mathrm{diag}(\ep_U[\ep^{2}_{U'}[\check{h}_{\theta^*}(U,U')]]) \check{H}^{-1} ( H^{-1} \circ (\mathbb{I} - W)) \\
    & \quad + 4 w^2\mathrm{diag}(\ep_U[\ep^{2}_{U'}[ \check{h}_{\theta^*}(U,U')]])(H^{-1}\circ (\mathbb{I} - W))^2.
\end{align*}
Using the statement, we study the Frobenisu norm $\|\Sigma_V\|_F$. 
As preparation, we set $c := \mathrm{diag}(\ep_U[\ep^{2}_{U'}[\check{h}_{\theta^*}(U,U')]])$, and recall the inequality $\|A \circ B\|_F \leq \|A\|_F \|B\|_F$.
Also, recall the definition $\delta_\Lambda := \max\{|\Lambda'-1|, |1 - \Lambda|\}$.
We obtain
\begin{align*}
    \|\Sigma_V\|_F & \leq \|\check{\Sigma}_V\|_F  + 8 (w^2 -1 )c \|\check{H}^{-2}\|_F + 4 (2-w^2) c \|\check{H}^{-1}\|_F^2 \|\mathbb{I} - W\|_F + 4w^2 c \|\check{H}^{-1}\|_F^2 \|\mathbb{I} - W\|_F^2 \\
    &\leq \|\check{\Sigma}_V\|_F  + 8 \Lambda' \delta_\Lambda c \|\check{H}^{-2}\|_F + 8 c \|\check{H}^{-1}\|_F^2 d \delta_\Lambda + 4(\Lambda')^2 c \|\check{H}^{-1}\|_F^2 d^2 \delta_\Lambda^2 \\
    &= \|\check{\Sigma}_V\|_F + O(\Lambda'\delta_\Lambda + (\Lambda'\delta_\Lambda)^2).
\end{align*}
Landau's Big O notation in the last equality picks up the terms depends only $\Lambda$ and $\Lambda'$.
\end{proof}

\section{Related Studies Beyond IV Regression Setting} \label{sec:related-work-2}

Beyond the IV regression, there are numerous prior studies that are related to ours, especially in policy evaluation, reinforcement learning, and causal inference.
Firstly, the idea of ``kernel loss'' was proposed in \citep{feng2019kernel} to estimate the value function in reinforcement learning. 
Similar ideas have been used to estimate the importance ratio of two state or state-action distributions in \citep{Liu18:Infinite} and \citep{uehara2020minimax}, and to estimate the average policy effect (APE) and policy learning in \citep{kallus2018balanced,Kallus20:Causal}.
Secondly, in the area of causal inference, \citep{wong2018kernel} employed a similar technique to estimate an average treatment effect. 
Despite the methodological similarity to our work, these works did not consider the IV regression setting and the analytical cross validation error. We will first introduce the objective functions of the aforementioned works for a better understanding of the connection and then highlight the differences, challenges, and novelties of our work.

Although estimating different subjects $f$, \citep{Liu18:Infinite} and \citep{feng2019kernel} employ similar population objective functions  in the following form:
\begin{align}\label{eq:kernel-loss}
    \min_{f} \ep_X\ep_{X'}[((\mathcal{A}f)(X)-f(X))k(X,X')((\mathcal{A}f)(X')-f(X'))],
\end{align}
where $\mathcal{A}$ is a task-specific operator acting on the subject $f$ and $X'$ is an independent copy of $X$. More specifically, \citep{Liu18:Infinite} aim to estimate the importance ratio of state distributions of two policies as $f(X)$.
In the task of value function estimation by \citep{feng2019kernel}, $f(X)$ is the value function. In both works, $X$ represents the state variable in the context of reinforcement learning.
Further, \citep{kallus2018balanced} studies the estimation of the APE similar to that of the average treatment effect, which estimate an average effect of a provided policy for a treatment conditioned on a covariate.

Compared with our population and empirical risks \eqref{eq:mmr-kernel} and \eqref{eq:quadratic-form}, 
the kernel loss \eqref{eq:kernel-loss}
has a similar quadratic form.
The difference is that the kernel loss
has variables involved in both the kernel function and the residual, whereas the kernel function in our objective depends on the instrument $Z$ which does not appear in the residual. Besides, our analytical cross validation error was not studied by these works and the similar forms of the objectives allow to adapt our result to their approaches.

The difference in losses also simplifies the estimation in the related work and requires us to perform new analyses. Specifically, the consistency of the estimators in the above related work holds under mild conditions. 
That is, in the estimation of the value function, minimizing the population objective function to zero guarantees a solution for the Bellman equation, which is consistent to the value function according to the unique solution property of the Bellman equation \citep{feng2019kernel}; for the importance ratio and the APE estimation, the consistency holds under mild conditions of data distributions \citep{Liu18:Infinite,kallus2018balanced}.
In contrast, the estimator obtained by minimizing the MMR risk \eqref{eq:mmr-kernel} to zero is hardly consistent to the true $f(X)$ under mild conditions, because there can be $\hat{f}\neq f$ satisfying $\ep[Y-\hat{f}(X)\,|\,Z]=0$ almost surely.
Therefore, we first introduce the completeness condition in \sref{Assumption}{asmp:opt_ident} for the estimation, which is classical in the IV regression literature and is needed to identify $f(X)$ from the CMR.
Second, a new theoretical analysis is necessary to clarify the nature of the estimation on $f(X)$, such as consistency and asymptotic normality.
We develop a novel theory for this point in \sref{Section}{sec:theory}, which has not been studied by these related works. 

\bibliographystyle{plain}
\bibliography{main}

\end{document}